\documentclass[11pt]{article}

\PassOptionsToPackage{numbers, compress}{natbib}
\usepackage{amsmath,amssymb,amsthm,fullpage,mathrsfs,pgf,tikz,caption,subcaption,mathtools,mathabx}
\usepackage[ruled,vlined,linesnumbered]{algorithm2e}
\usepackage{natbib}
\usepackage{amsmath,amssymb,amsthm,mathtools,amsbsy}
\usepackage[utf8]{inputenc} % allow utf-8 input
\usepackage[T1]{fontenc}    % use 8-bit T1 fonts
\usepackage{hyperref}       % hyperlinks
\usepackage{url}            % simple URL typesetting
\usepackage{booktabs}       % professional-quality tables
\usepackage{amsfonts}       % blackboard math symbols
\usepackage{nicefrac}       % compact symbols for 1/2, etc.
\usepackage{microtype}      % microtypography
\usepackage{times}
\usepackage[margin=1in]{geometry}
\usepackage{bbm}
\usepackage{enumitem}
\usepackage{xcolor}
\usepackage{cleveref}
\usepackage{bm}
\hypersetup{colorlinks=true,citecolor=blue,linkcolor=red}
\usepackage{algorithmic}
\usepackage{todonotes}
\usepackage{thmtools}
\usepackage{float}
\usetikzlibrary{shapes.geometric}

\usepackage[utf8]{inputenc}
\usepackage{comment}
\usepackage{xspace}
\usepackage{bbm}
\usepackage{ulem} 
\usepackage{hyperref}
\hypersetup{
    colorlinks=true,
    linkcolor=blue,
    filecolor=magenta,      
    urlcolor=cyan,
    linktocpage=true,
}

\newcommand{\notes}{0} % set to 0 to disable notes
\ifnum\notes=1
\newcommand{\mb}[1]{{\color{blue} MB:#1}}
\newcommand{\mg}[1]{{\color{magenta} MG:#1}}
\newcommand{\sstext}[1]{{\color{red} SS:#1}}
\newcommand{\mh}[1]{{\color{purple} MH:#1}}
\newcommand{\jess}[1]{\textcolor{teal}{[JS: #1]}}
\newcommand{\maxh}[1]{\textcolor{red}{[MH: #1]}}
\newcommand{\rexnote}[1]{\textcolor{cyan}{[RL: #1]}}
\newcommand{\toni}[1]{\textcolor{purple}{[TP: #1]}}
\newcommand{\chris}[1]{\textcolor{purple}{[CY: #1]}}

\else % notes just appear white
\newcommand{\mb}[1]{}
\newcommand{\mg}[1]{}
\newcommand{\sstext}[1]{}
\newcommand{\mh}[1]{}
\newcommand{\jess}[1]{}
\newcommand{\maxh}[1]{}
\newcommand{\rexnote}[1]{}
\newcommand{\toni}[1]{}
\newcommand{\chris}[1]{}
\fi

\ifnum\notes=1
\newcommand{\mbnote}[1]{{\color{blue}\footnote{{\color{blue} {\bf MB:} #1}}}}
\newcommand{\mgnote}[1]{{\color{magenta}\footnote{{\color{magenta} {\bf MG:} #1}}}}
\newcommand{\ssnote}[1]{{\color{red}\footnote{{\color{red} {\bf SS:} #1}}}}
\newcommand{\mhnote}[1]{{\color{purple}\footnote{{\color{purple} {\bf MH:} #1}}}}
\else
\newcommand{\mbnote}[1]{}
\newcommand{\mgnote}[1]{}
\newcommand{\mhnote}[1]{}
\newcommand{\ssnote}[1]{}
\fi

\newtheorem{theorem}{Theorem}[section]
\newtheorem{corollary}[theorem]{Corollary}
\newtheorem{proposition}[theorem]{Proposition}
\newtheorem{lemma}[theorem]{Lemma}
\newtheorem{definition}[theorem]{Definition}

\newtheorem{observation}[theorem]{Observation}
\newtheorem{claim}[theorem]{Claim}

\DeclareMathOperator{\sign}{sign}
\DeclareMathOperator{\poly}{poly}

\newcommand{\supp}{\mathtt{supp}}

\newcommand{\Z}{\mathbb{Z}}
\newcommand{\N}{\mathbb{N}}

\newcommand{\dtv}{d_{\mathrm TV}}
\renewcommand{\Pr}{\mathbf{Pr}}

\newcommand{\Acal}{\mathcal{A}}
\newcommand{\Bcal}{\mathcal{B}}
\newcommand{\Ical}{\mathcal{I}}

\newcommand{\eps}{\varepsilon}
\newcommand{\E}{\mathbb{E}}

\newcommand{\X}{\mathcal{X}}
\newcommand{\VVV}{\mathcal{V}}

\newcommand{\eqdef}{\stackrel{{\mathrm {\scriptstyle def}}}{=}}

\newcommand{\emprisk}[1]{err_{\ensuremath{#1}}}
\newcommand{\distrisk}[1]{err_{\ensuremath{#1}}}

\newcommand{\rFinite}{\mathtt{rFiniteLearner}}

\newcommand{\KeyGen}{\ensuremath{\mathsf{KeyGen}}}

\newcommand{\escheme}{\ensuremath{\mathcal{E}}}
\newcommand{\Dec}{\ensuremath{\mathsf{Dec}}}
\newcommand{\Enc}{\ensuremath{\mathsf{Enc}}}

\newcommand{\sk}{\mathbf{sk}}
\newcommand{\pk}{\mathbf{pk}}

\newcommand{\Ran}{\ensuremath{\mathsf{Rerandomize}}}

\newcommand{\negl}{\mathrm{negl}}

\newcommand{\universe}{\ensuremath{\mathcal{U}}\xspace}
\newcommand{\sampler}{\ensuremath{\mathcal{A}}\xspace}
\newcommand{\Datafixed}{\ensuremath{S}\xspace}

\newcommand{\datafixed}{\ensuremath{s}\xspace}

\newcommand{\indicator}{\ensuremath{\mathbbm{1}}\xspace}

\newcommand{\err}{{\ensuremath{\mathtt{err}}}}
\newcommand{\correlatedsampler}{\mathtt{CorrSamp}}
\newcommand{\correlatedsamplerfull}{\mathtt{CorrSamp}(C, \nu; r)}

\newcommand{\elementfindershort}{\mathtt{EF}}
\newcommand{\elementfinder}{\mathtt{ElemFind}}
\newcommand{\elementfinderfull}{\mathtt{ElemFind}_{C, \nu, \ell, \beta}(H_1, u; r)}

\newcommand{\hashchecker}{\mathtt{HashCheck}}
\newcommand{\hashcheckerfull}{\mathtt{HashCheck}_{C, \nu, \ell, H_1, u}(H_2,v; r)}

\newcommand{\Fcircuit}{F_{C,\ell, H_1, H_2}}

\newcommand{\score}{\operatorname{score}}

\let\oldnl\nl
\newcommand{\nonl}{\renewcommand{\nl}{\let\nl\oldnl}}
\author{
  Mark Bun\thanks{Department of Computer Science, Boston University. Email: \texttt{mbun@bu.edu}. Supported by NSF awards CCF-1947889 and CNS-2046425, a Sloan Research Fellowship, and Cooperative Agreement CB20ADR0160001 with the Census Bureau.} \and Marco Gaboardi \thanks{Department of Computer Science, Boston University. Email: \texttt{gaboardi@bu.edu}. Supported by NSF awards CNS 2040249 and Cooperative Agreement CB20ADR0160001 with the Census Bureau.}  \and
  Max Hopkins\thanks{Department of Computer Science and Engineering, UCSD, CA 92092. Email: \texttt{nmhopkin@eng.ucsd.edu}. Supported by NSF Award DGE-1650112.}\and Russell Impagliazzo\thanks{Department of Computer Science and Engineering, UCSD, CA 92092. Email: \texttt{russell@eng.ucsd.edu} Supported by NSF Award AF: Medium 2212136} 
  \and Rex Lei\thanks{{Department of Computer Science and Engineering, UCSD, CA 92092. Email: \texttt{rlei@ucsd.edu} Supported by NSF Award AF: Medium 2212136}}
  \\ \and Toniann Pitassi\thanks{Columbia University, University of Toronto and IAS. Email: \texttt{tonipitassi@gmail.com}. Supported by NSF Award AF: Medium 2212136.}
  \and Satchit Sivakumar\thanks{Department of Computer Science, Boston University. Email: \texttt{satchit@bu.edu}. Supported by NSF award CNS-2046425 and Cooperative Agreement CB20ADR0160001 with the Census Bureau.}%, as well as Cooperative Agreement CB20ADR0160001 with the Census Bureau. The views expressed in this paper are those of the author and not those of the U.S. Census Bureau or any other sponsor.}
  \and Jessica Sorrell\thanks{University of Pennsylvania. Email: \texttt{jsorrell@seas.upenn.edu}. Supported in part by the Simons Foundation Collaboration on the Theory of Algorithmic Fairness.}
}
\title{Stability is Stable: Connections between Replicability, Privacy, and Adaptive Generalization}

\begin{document}
\maketitle

\begin{abstract}

The notion of replicable algorithms was introduced in~\cite{ImpLPS22} to describe randomized algorithms that are stable under the resampling of their inputs. More precisely, a replicable algorithm gives the same output with high probability when its randomness is fixed and it is run on a new i.i.d.\ sample drawn from the same distribution. Using replicable algorithms for data analysis can facilitate the verification of published results by ensuring that the results of an analysis will be the same with high probability, even when that analysis is performed on a new data set.

In this work, we establish new connections and separations between replicability and standard notions of algorithmic stability. In particular, we give sample-efficient algorithmic reductions between perfect generalization, approximate differential privacy, and replicability for a broad class of statistical problems. Conversely, we show any such equivalence must break down computationally: there exist statistical problems that are easy under differential privacy, but that cannot be solved replicably without breaking public-key cryptography. Furthermore, these results are tight: our reductions are statistically optimal, and we show that any computational separation between DP and replicability must imply the existence of one-way functions.

Our statistical reductions give a new algorithmic framework for translating between notions of stability, which we instantiate to answer several open questions in replicability and privacy. This includes giving sample-efficient replicable algorithms for various PAC learning, distribution estimation, and distribution testing problems, algorithmic amplification of $\delta$ in approximate DP, conversions from item-level to user-level privacy, and the existence of private agnostic-to-realizable learning reductions under structured distributions.
\end{abstract}

\newpage
\tableofcontents
\newpage

\section{Introduction}

\emph{Replicability} is the principle that the findings of an empirical study should remain the same when it is repeated on new data. Despite being a pillar of the scientific method, replicability is extremely difficult to ensure in today's complex data generation and analysis processes. Questionable research practices including misapplication of statistics, selective reporting of only the findings that appear most statistically significant, and the formulation of research hypotheses after the results are already known have been identified as causes of an ongoing ``crisis of replicability'' across the empirical sciences. Toward formulating solutions in the context of machine learning and algorithmic data analysis, 
Impagliazzo, Lei, Pitassi, and Sorrell~\cite{ImpLPS22} recently put forth
a new definition of replicability for statistical learning algorithms.\footnote{\cite{ImpLPS22} stated this definition under the name ``reproducibility.'' See Section \ref{ssec:terminology-replicability} for a discussion of why we refer to it as ``replicability'' instead.}
%a new stability definition to capture the reproducibility of learning algorithms themselves.
%Roughly speaking, an algorithm is replicable if with high probability over the choice of two independent samples $S, S'$ from the same distribution, it produces exactly the same output. 
\begin{definition}
\label{def:replicable}
A randomized algorithm $A : \mathcal{X}^n \to \mathcal{Y}$ is $\rho$-replicable if for every distribution $D$ over $\mathcal{X}$, we have
\[\Pr[A(S_1; r) = A(S_2; r)] \ge 1-\rho,\]where $S_1, S_2 \in \mathcal{X}^n$ are independent sequences of i.i.d. samples from $D$, and $r$ represents the coin tosses of the algorithm $A$.
\end{definition}
That is, an algorithm (capturing an end-to-end data analysis process) is replicable if with high probability over the choice of two independent samples from the same distribution, it produces exactly the same output. If one research team shares both their replicable analysis process ($A$) and the random choices made along the way ($r$), then another research team can independently verify their conclusions by performing the same analysis on a fresh dataset.

Replicability is an extremely strong \emph{stability} constraint to place on an algorithm. Informally, an algorithm is stable if its output is insensitive to small changes to its input. Nevertheless, replicability is achievable for many fundamental data analysis tasks, including statistical query learning, heavy hitter identification, approximate median finding, and large-margin halfspace learning~\cite{ImpLPS22, GhaziKM21}. 

Replicability is not the first definition of algorithmic stability aimed at ensuring the utility and safety of modern data analysis. Others have played central roles in relatively mature areas such as differential privacy and adaptive data analysis. Some of the aforementioned replicable algorithms were, in fact, motivated or inspired by differentially private counterparts. Is there a systematic explanation for this? \emph{What can we learn about the capabilities and limitations of replicable algorithms by relating replicability to other notions of algorithmic stability?}

%In several cases, these advances were motivated or inspired by connecting replicability to another notion of stability: differential privacy.

%Replicability is an extremely strong notion of \emph{algorithmic stability}.

%Despite being an extremely strong form of stability, reproducibility can be achieved for many learning tasks with only a modest degradation in efficiency.\mg{Perhaps we should give some examples here like we did for DP.}
%\rexnote{Subsection \ref{ssec:terminology-replicability} has a short discussion of why we are  calling this definition ``replicability" and not ``reproducibility"}

%Algorithmic stability has emerged as a central concept for ensuring both the utility and safety of modern machine learning and data analysis. Informally, an algorithm is stable if its output is insensitive to small changes to its input. Precise formulations of algorithmic stability underlie the key definitions in a number of areas:

Let us briefly recall the types of algorithmic stability that arise in these other areas:

\paragraph{Differential privacy.} A randomized algorithm is differentially private~\cite{DworkMNS16j} if changing a single input record results in a small change in the distribution of the algorithm's output. When each input record corresponds to one individual's datum, differential privacy guarantees that nothing specific to any individual can be learned from the output of the algorithm. (See Section \ref{def:differentially private}.) Differential privacy comes with a rich algorithmic toolkit and understanding of the feasibility of fundamental statistical tasks in query estimation, classification, regression, distribution estimation, hypothesis testing, and more. 

\paragraph{Generalization in adaptive data analysis.} Generalization is the ability of a learning algorithm to reflect properties of a population, rather than just properties of a specific sample drawn from that population. Techniques for provably ensuring generalization form a hallmark of theoretical machine learning. However, generalization is particularly difficult to guarantee in settings where multiple analyses are performed adaptively on the same sample. Traditional notions of generalization do not hold up to downstream misinterpretation of results. For example, a classifier that encodes detailed information about its training sample in its lower order bits may generalize well, but can be used to construct a different classifier that behaves very differently on the sample than it does on the population. Interactive processes such as exploratory data analysis or feature selection followed by classification/regression can ruin the independence between the training sample and the method used to analyze it, invalidating standard generalization arguments.

Adaptivity in data analysis has been identified as one contributing factor to the replication crisis, and imposing stability conditions on learning algorithms
%leads to general-purpose solutions to these problems.
offers solutions to this part of the problem. A variety of such stability conditions have been studied~\cite{DworkFHPRR15, DworkFHPRR15-1, BassilyNSSSU16, RussoZ16, CummingsLNRW16, BassilyF16, RaginskyRTWX16, BassilyMNSY18, LigettS19, SteinkeZ20}, each offering distinct advantages in terms of the breadth of their applicability and the quantitative parameters achievable. Two specific notions play a central role in this work. The first is \emph{perfect generalization}~\cite{CummingsLNRW16, BassilyF16}, which ensures that whatever can be inferred from the output of a learning algorithm when run on a sample $S$ could have been learned just from the underlying population itself:

\begin{definition}
\label{def:PG}
An algorithm $A : \mathcal{X}^n \to \mathcal{Y}$ is $(\beta, \eps, \delta)$-perfectly generalizing if, for every distribution $D$ over $\mathcal{X}$, there exists a distribution $Sim_D$ 
%\rexnote{$Sim_D$ is really bothering me, instead of something like $\mathtt{Sim}_D$. I can find-replace these or just stop caring}
%\sstext{Yeah I dont have strong feelings either way so whichever you prefer is fine} 
such that, with probability at least $1-\beta$ over $S$ consisting of $n$ i.i.d. samples from $D$, and every set of outcomes $\mathcal{O} \subseteq \mathcal{Y}$,
\begin{equation} \label{eqn:pg-intro}
    e^{-\eps}(\Pr_{Sim_D}[\mathcal{O}] - \delta) \le \Pr[A(S) \in \mathcal{O}] \le e^{\eps} \Pr_{Sim_D}[\mathcal{O}] + \delta.
\end{equation} 
%That is, whatever can be learned from the output of $A(S)$, for random $S$, could have been simulated using access only to the distribution $D$ itself. Thus, perfect generalization gives a strong guarantee of robustness of a learning algorithm to postprocessing.
\end{definition}

The second is \emph{max-information}~\cite{DworkFHPRR15} which constrains the amount of information revealed to an analyst about the training sample:
\begin{definition}
\label{def:bounded-max-info}
An algorithm $A : \mathcal{X}^n \to \mathcal{Y}$ has $(\eps, \delta)$-max-information with respect to product distributions
%(abbreviated $(\rho, \beta)$-BAMIPD)%\footnote{Strictly speaking, \cite{DworkFHPRR15} gave a one-sided definition of max-information that only requires the inequality on the right. The two-sided definition we use here is more convenient for studying the implications of bounded max-information.}
%\mb{Terrible, terrible acronym...} 
if for every set of outcomes $\mathcal{O} \subseteq (\mathcal{Y} \times \mathcal{X}^n)$ we have
\[\Pr[(A(S), S) \in \mathcal{O}] \le e^{\eps} \Pr[(A(S), S') \in \mathcal{O}] + \delta,\]
where $S$ and $S'$ are independent samples of size $n$ drawn i.i.d. from an arbitrary distribution $D$ over $\mathcal{X}$. 
\end{definition}
As with differential privacy, both perfect generalization and max-information 
%(with respect to arbitrary, not just product distributions) 
are robust to post-processing.
%and degrades gracefully under adaptive composition \sstext{max information I guess is fine, though we do define it specifically for product distributions. And PG definitely does not adaptively compose}.

\medskip

Each stability definition described above is tailored to model a distinct desideratum. At first glance, they may all appear technically incomparable. For instance, differential privacy is stricter than the other definitions in that it holds in the worst case over all input datasets without any assumptions on the data-generating procedure. On the other hand, it is weaker in that it only requires insensitivity to changing one input record, rather than to resampling the entire input dataset as in max-information, perfect generalization, or replicability. Meanwhile, differential privacy, max-information, and perfect generalization quantify the sensitivity of the algorithm's output in a weaker way than replicability; the former three notions only require that the distributions on outputs are similar, whereas replicability demands that precisely the same output realization is obtained with high probability.

Nevertheless, the (surprising!) technical connections between these definitions have enabled substantial progress on the fundamental questions in their respective areas. For example, it was exactly the adaptive generalization guarantees of differential privacy that kickstarted the framework of adaptive data analysis from~\cite{DworkFHPRR15-1}; the definition of max-information was subsequently introduced~\cite{DworkFHPRR15} to unify existing analyses based on differential privacy and description length bounds. As another illustration, variants of replicability were introduced in~\cite{bun2020equivalence, ghazi2021sample, GhaziKM21} for purely technical reasons, as it was observed that such algorithms could be immediately used to construct differentially private ones. This connection was essential in proving the characterization of private PAC learnability in terms of the Littlestone dimension from online learning~\cite{alon2019private, bun2020equivalence}. In fact, this characterization shows, that, in principle a private PAC learner using $n$ samples can be converted to a replicable PAC learner using a number of samples that is an exponential tower of height $n$, but it is non-constructive and does not suggest what such a learner looks like in general.

\subsection{Our Main Results}
\subsubsection{Equivalences}
Our main result is a complete characterization of the relationships between these quantities. We prove that all four central stability notions --- replicability, differential privacy, perfect generalization, and bounded max-information w.r.t. product distributions --- are equivalent to one another via constructive conversions that incur at most a near-quadratic overhead in sample complexity.
%with respect to sample complexity. 
%clarify the relationship between these disparate notions of algorithmic stability. For an extremely general class of statistical tasks, we show that the achievability of all of these definitions is equivalent.
%
%In particular, we show that any learning algorithm achieving one of $\{$differential privacy, perfect generalization, bounded max-information,
%w.r.t. product distributions, 
%replicability$\}$ can be converted to an algorithm satisfying any other with at most polynomial overhead in the number of samples required. 
%(Figure~\ref{fig:transformations} summarizes these equivalences.) 
%We also give statistical and computational lower bounds illustrating the tightness of these connections. A more detailed discussion of our results follows.

%\subsection{Equivalences between Stability Notions}

Our equivalences apply to an abstract and broad class of \emph{statistical tasks} that capture learning from i.i.d.\ samples from a population. An instance of such a task is obtained by considering a distribution $D$ from a pre-specified family of distributions. Given i.i.d. samples from $D$, the goal of a learning algorithm is to produce an outcome that is ``good'' for $D$ with high probability. %See Definition~\ref{def:stat-task} for a precise definition. 
This formulation of a statistical task captures problems such as PAC learning, where a sample from $D$ is a pair $(x, f(x)) \in \mathcal{X} \times \{0, 1\}$ where $x$ is drawn from an arbitrary marginal distribution over $\mathcal{X}$, and $f$ is an arbitrary function from a fixed concept class $H$. A ``good'' outcome for such a distribution $D$ is a hypothesis $h: \mathcal{X} \to \{0, 1\}$ that well-approximates $f$ on $D$. %\mb{Worth giving another example? Huge part of the point, after all, is showing we can do stuff beyond PAC learning}
Many other objectives such as regression, distribution parameter estimation, distribution learning, hypothesis testing, and confidence interval construction can be naturally framed as statistical tasks. (See Section~\ref{sec:alg-by-reduction} for other examples.)

Figure~\ref{fig:transformations} illustrates the known relationships between the various stability notions that hold with respect to any statistical task. 
%consider that we prove to hold for general statistical tasks.
%
\begin{figure}[h!] 
	\begin{tikzpicture}[->,auto,>=stealth,node distance=7cm,
		scale = 1,transform shape, rct/.style = {draw,rounded corners=0.1cm},
		every edge/.append style={line width=0.5mm}]
		
		\node[rct] (rep) {$(0.1)$-replicability};
		\node[rct] (dp) [right of=rep, xshift=4cm] {$(\eps, \delta)$-differential privacy};
		%\node[rct] (strong-dp)  [below right of=dp]  {$(0.1/\sqrt{n}, 1/n^2)$-differential privacy};
		\node[rct, align=center] (bmi)  [below of=dp]  {$(\eps, \delta)$-max-information \\ w.r.t. product distributions};
		\node[rct, align=center, xshift=-4cm] (opg)  [left of=bmi]  {$(\delta, \eps, \delta)$-one-way \\ perfect generalization};
		%\node[rct, align=center] (pg) [below right of=rep]  {$(\delta, \eps, \delta)$-perfect generalization};
		
		\path 
		(rep) edge[dashed]   node {$n \mapsto n \cdot \frac{\log(1/\delta)}{\eps}$, Thm. \ref{thm:Rep-to-DP} \cite{GhaziKM21}} (dp)
		(dp) edge[dashed] node[align=center]{$n \mapsto n^2$ \\ Cor. \ref{cor:dp-to-max-info} \cite{RogersRST16}}        (bmi)     %node {$n \mapsto n^2$} (strong-dp)
		%(strong-dp) edge              node {} (bmi)
		(bmi) edge              node{Lem. \ref{lem:max-info-PG}} (opg)
		(opg) edge[bend left=20, thin, densely dotted]              node[align=center]{$n \mapsto n$ \\ Thm.~\ref{PGtoRep}} (rep)
		(rep) edge[bend left=20, dashed]              node[align=left]{$n \mapsto n \cdot \frac{\poly\log(1/\eps, 1/\delta)}{\eps^2}$ \\ Thm.~\ref{thm:reprodtoPG}} (opg);
		%(pg) edge[dashed, swap] node[align=center]{$n \mapsto n$ \\\cite{CummingsLNRW16}} (dp)
		%(pg) edge node {} (opg);
		
	\end{tikzpicture}   
	\caption{The solid arrow from $A$ to $B$ means that every algorithm satisfying $A$ also satisfies $B$. A dashed arrow means that for every statistical task, a solution satisfying $A$ can be computationally efficiently transformed into a solution satisfying $B$ with the stated blowup in sample complexity. The thin dotted arrow means an explicit transformation exists, but is not always computationally efficient, and assumes the outcome space is finite. 
    \\This figure suppresses constant factors everywhere and polynomial factors in $\delta$, assumes $\eps$ is below a sufficiently small constant, and assumes that $\delta$ is a sufficiently small inverse polynomial in $n$. }
   % \mb{Should we have a theorem statement somewhere for PG->DP from CLNRW? I'd like to find a way to clearly give credit for the results that aren't ours.} \toni{Also I think we should do the same for dp to maxinfo, having one theorem statement with credit to [RRST16]. Lastly I would (mildly) vote for stability-for-dummies version of the figure, leave out one-way PG to make the figure a square.}}
 \label{fig:transformations}  
\end{figure}
\rexnote{I made the dotted line thin so it's easier to tell the difference between dashed and dotted.}
%
%Before sketching the ideas underlying these transformations, we 

\medskip
 
From these equivalences we obtain the following  consequences, resolving several open questions.
%the our equivalences both replicable and differentially private algorithm design. 
%Details of these applications appear in Section~\ref{sec:applications}.

\paragraph{Sample-efficient replicable algorithms.} Any differentially private algorithm solving a statistical task (with a finite outcome space) can be converted into a replicable algorithm solving the same task with a near-quadratic blowup in its sample complexity. Thus, the wealth of research on private algorithm design can be brought to bear on designing replicable algorithms. We illustrate this algorithmic paradigm by describing new replicable algorithms for some PAC learning, distribution parameter estimation, and distribution testing problems in \Cref{sec:alg-by-reduction}.

\paragraph{Equivalence between perfect generalization and differential privacy.} For simplicity, the relationships  summarized in Figure~\ref{fig:transformations} are stated in terms of a one-way variant of perfect generalization, where only the inequality on the right of (\ref{eqn:pg-intro}) is required to hold. But the original two-way definition turns out to be statistically equivalent for tasks with a finite outcome space. This is because a one-way perfectly generalizing algorithm can be converted to a replicable algorithm using Theorem~\ref{PGtoRep}, and Theorem~\ref{thm:reprodtoPG} actually yields the stronger conversion back to a two-way perfectly generalizing algorithm (See Theorem~\ref{thm:opg-to-pg}). Thus, an $(\eps, \delta)$-differentially private algorithm (with a finite outcome space) can be converted to a perfectly generalizing one solving the same statistical task with a near quadratic blow-up in sample complexity. This resolves an open question of~\cite{CummingsLNRW16}. Their work also gave a conversion from perfectly generalizing algorithms to differentially private ones with no sample complexity overhead, and while their transformation preserves accuracy for (agnostic) PAC learning, it is not clear how to analyze it for general statistical tasks. Our conversion from perfect generalization to replicability and then to differential privacy holds for all statistical tasks with a finite outcome space.

\paragraph{Converting item-level to user-level privacy.} Consider a ``user-level'' learning scenario in which $n$ individuals each hold $m$ training examples drawn i.i.d. from the same distribution. When is $(\eps, \delta)$-differentially private learning possible if we wish to guarantee privacy with respect to changing \emph{all} of any individual's samples at once? Ghazi, Kumar, and Manurangsi \cite{GhaziKM21} showed that this is possible when $n \ge O(\log(1/\delta) / \eps)$ and the task  admits a replicable learner. For the special case of PAC learning a concept class $H$, they argued that this implies a user-level private learning algorithm whenever $H$ is privately PAC learnable with respect to changing a single \emph{sample}.  They posed the open problem of extending this result beyond PAC learning, e.g., to private regression~\cite{JungKT20, Golowich21}.
Our conversion from any differentially private algorithm to a replicable one implies that such a transformation is possible for \emph{any} statistical task with a finite outcome space (Section~\ref{sec:item-user}). Moreover, one can always take each indvidual's number of samples $m$ to be nearly quadratic in the sample complexity of the original item-level private learner. %\mb{Did they state this as an explicit open question?} \sstext{Yes, they asked whether their techniques could be extended beyond the PAC setting.}

\paragraph{Amplifying differential privacy parameters.} %A classic privacy amplification by subsampling argument~\cite{kasiviswanathan2011can} shows how to efficiently reduce the parameter $\eps$ in $(\eps, \delta)$-differential privacy.
    While almost all $(\eps, \delta)$-differentially private algorithms enjoy a mild $\propto \log(1/\delta)$ dependence in their sample complexity on the parameter $\delta$, it was not known how to achieve this universally, say by amplifying large values of $\delta$ to asymptotically smaller ones. \cite{bun2020equivalence} showed that for private PAC learning, such amplification is possible in principle, but posed the open question of giving an explicit amplification algorithm. By converting an $(\eps, \delta)$-differentially private algorithm with weak parameters to a replicable one, and then back to a differentially private one with strong parameters, we resolve this question for the general class of statistical tasks with a finite outcome space, and with a much milder sample complexity blowup (Section~\ref{sec:delta-amp}). 
    
\paragraph{Agnostic-to-realizable reductions for distribution-family learning.} \cite{hopkins2021realizable} introduced a simple and flexible framework for converting realizable PAC learners to agnostic learners without relying on uniform convergence arguments. The framework applies to diverse settings such as robust learning, fair learning, partial learning, and (as observed in this work) replicable learning, with differential privacy providing a notable exception.\footnote{We note the technique we introduce to adapt \cite{hopkins2021realizable} to the replicable setting has no clear translation to the private setting.} While an agnostic-to-realizable reduction for private PAC learning is known~\cite{beimel2013private, alon2020closure}, it relies on uniform convergence and is only known to hold in the distribution-free PAC model. By converting a realizable private learner to a realizable replicable learner, then to an agnostic replicable learner, and back to an agnostic private learner, we obtain a reduction that works in the absence of uniform convergence (Section~\ref{sec:realizable-agnostic}). In particular, this reduction applies to the \emph{distribution-family} learning model, where one is promised that the marginal distribution on unlabeled examples comes from a pre-specified family of distributions.

\medskip

\subsubsection{Separating Stability: Computational Barriers and the Complexity of Correlated Sampling}  

All of the transformations appearing in Figure~\ref{fig:transformations} preserve computational efficiency, with the lone exception of the transformation from perfectly generalizing algorithms to replicable ones. This transformation makes use of the technique of \emph{correlated sampling} from the distribution of outputs of a perfectly generalizing algorithm $A$ when run on a fixed sample $S$ (elaborated on more in Sections~\ref{sec:equiv-overview} and~\ref{sec:correlated-sampling}). This step can be explicitly implemented via rejection sampling from the output space of $A$, with the rejection threshold determined by the probability mass function of $A(S)$, but in general it is not computationally efficient.

We show that under cryptographic assumptions, this is inherent (Section~\ref{sec:separating-stability-computationally}). Specifically, we show that under standard assumptions in public-key cryptography, there exists a statistical task that admits an efficient differentially private algorithm, %(hence, also an efficient perfectly generalizing one \sstext{This paranthesis isnt true anymore because we go through rep to get to PG}), 
but does not have any efficient replicable algorithm. The task is defined in terms of a public-key encryption scheme with the following rerandomizability property: Given a ciphertext $\Enc(\pk, b)$, there is an efficient algorithm producing a uniformly random encryption of $b$. Fixing such a rerandomizable PKE, the statistical task is as follows. Given a dataset consisting of random encryptions of the form $\Enc(\pk, b)$ where $\pk$ is a fixed public key and $b \in \{0, 1\}$ is a fixed bit, output any encryption of $b$.

One can solve this problem differentially privately, essentially by choosing a random ciphertext from the input dataset and rerandomizing it. On the other hand, there is no efficient replicable algorithm for this task. If there were, then one could use the public key to produce many encryptions of $0$ and $1$ and run the replicable algorithm on the results to produce canonical ciphertexts $c_0$ and $c_1$, respectively. Then, given an unknown ciphertext, one could repeatedly rerandomize it, run the replicable algorithm on the results, and compare the answer to $c_0$ and to $c_1$ to identify the underlying plaintext.

We also show that cryptographic assumptions are necessary even to separate replicability from perfect generalization. Recalling again that the bottleneck in computationally equating the two notions is in implementing correlated sampling, we show in Section~\ref{sec:owfi} that if one-way functions do not exist, then correlated sampling is always tractable. In addition to addressing a natural question about the complexity of correlated sampling, this shows that function inversion enables an efficient transformation from perfectly generalizing algorithms into replicable ones. (See Section \ref{sec:correlated-sampling} for more discussion.)

\subsubsection{Separating Stability: Statistical Barriers}
Our equivalences show that the sample complexities of perfectly generalizing and replicable learning are essentially equivalent. Moreover:
(1) An approximate-DP algorithm can be converted to a perfectly generalizing/replicable algorithm with near-quadratic blowup; and
(2) A perfectly generalizing/replicable algorithm can be converted to an approximate-DP one using roughly the same number of samples.
We prove that both of these conversions are optimal by showing:

\begin{enumerate}
    \item {\bf Quadratic separations between differential privacy and perfect generalization/replicability.} We first consider the problem of estimating the parameters of a product of $d$ Bernoulli distributions. By simply taking the empirical mean of an input dataset, this problem can be solved using $O(\log d)$ without any stability constraints. However, with differential privacy, it is known that $\tilde{\Theta}(\sqrt{d})$ samples are necessary and sufficient. By adapting the ``fingerprinting'' method underlying these privacy lower bounds~\cite{BunUV18, DworkSSUV15, BunSU19} to perfect generalization, we prove that any perfectly generalizing or replicable algorithm for this problem requires $\tilde{\Omega}(d)$ samples (Section~\ref{sec:marginals}).
    
    By reducing from a variant of this one-way marginals problem, we also show a general lower bound for replicable agnostic learning. Namely, we show that every concept class $H$ requires $\tilde{\Omega}(VC(H)^2)$ samples. For concept classes of maximal VC dimension $VC(H) = \log|H|$, this too gives a quadratic separation between replicable learning and both private and unconstrained learning (Section~\ref{sec:agnostic-lb}).
     
    \item {\bf No separation between differential privacy and perfect generalization/replicability.} Complementing our lower bounds, we also show that every finite class $H$ can be replicably PAC learned (in the realizable setting) to error $\alpha$ with sample complexity $\tilde{O}_{H}(1/\alpha)$ (Section~\ref{sec:finiteclasses}). Up to logarithmic factors, this matches the learning rate achievable for both unconstrained and differentially private learning. Our learner works by selecting a random threshold $v$, and selecting a random concept from $H$ whose error with respect to the sample is at most $v$. A more involved random thresholding strategy also yields an agnostic learner with sample complexity $\tilde{O}_{H}(1/\alpha^2)$.
\end{enumerate}

%\subsubsection{Applications}

\subsection{Overview of Proofs of Equivalences} \label{sec:equiv-overview}

\paragraph{Perfect generalization is equivalent to replicability.} %\cite{ImpLPS22} initiated a formal study of replicability for learning algorithms by defining a randomized algorithm $A : \mathcal{X}^n \to \mathcal{Y}$ to be $\rho$-replicable if, for every product distribution $D$ on $\mathcal{X}^n$, we have $\Pr[A(S; r) = A(S'; r)] \ge 1-\rho$, where $S, S'$ are independent samples from $D$ and $r$ represents the random choices of $A$. That is, upon fixing the coins of $A$, 
Recall that an algorithm is replicable if it is likely to produce exactly the same output when run on two independent samples from any given population. Replicability appears to be a dramatic strengthening of perfect generalization, which only requires the distributions of $A(S)$ and $A(S')$ to be statistically close. Nevertheless, we prove that perfectly generalizing algorithms can always be converted to replicable ones whenever the output space $\mathcal{Y}$ is finite (Theorem \ref{PGtoRep}). This can be done via a primitive called \emph{correlated sampling} (See Section \ref{sec:correlated-sampling}).
A correlated sampling algorithm for a class of distributions $\mathcal{P} = \{P\}$ is a procedure $CS(P, r)$ such that 1) $CS(P, r)$ produces a sample distributed according to $P$ when provided a uniformly random input $r$, and 2) Whenever $P, Q \in \mathcal{P}$ satisfy $\dtv(P, Q) \le \eta$, we have $\Pr[CS(P, r) = CS(Q, r)] \ge 1 - O(\eta)$. That is, applying correlated sampling to two similar distributions results in the same output with high probability -- exactly what is needed for replicability. We actually prove a stronger theorem, showing that the larger class of {\it one-way} perfectly generalizing algorithms (where only the right-hand inequality in \ref{eqn:pg-intro} holds) are replicable via correlated sampling. 
%We give the details of this argument in Theorem~\ref{PGtoRep}.

Conversely, we show how to convert replicable algorithms to perfectly generalizing ones (Theorem~\ref{thm:reprodtoPG}). While a $\rho$-replicable algorithm is automatically also a $(\beta = O(\rho), \eps = 0, \delta = O(\rho))$-perfectly generalizing one, these parameters are too weak for applications where one wants to take $\beta, \delta$ to be inverse polynomial in the dataset size $n$ (e.g., to prove the lower bounds in Section~\ref{sec:stat-sep}). To obtain a perfectly generalizing algorithm with stronger parameters, we repeatedly run the replicable algorithm using $k = O(\log (1/\delta))$ different sequences of coin tosses $r_1, \dots, r_k$, and using $\tilde{O}(1/\eps^2)$ independent samples for each sequence of coin tosses. Using the exponential mechanism from differential privacy~\cite{McTalwar}, we select an outcome $y_i$ that appears approximately the most frequently amongst these repetitions in a manner that ensures $(\beta = \delta, \eps, \delta)$-perfect generalization. This strategy allows us to obtain inverse polynomial $\beta, \delta$ parameters with only a logarithmic multiplicative overhead in the number of samples.

%\begin{itemize}
%    \item Describe reproducibility.
%    \item Explain correlated sampling
%    \item Say something about converse direction??
%\end{itemize}

\medskip

\noindent{\bf Bounded max-information implies perfect generalization.}
In Lemma~\ref{lem:max-info-PG}, we show that bounded max-information implies one-way perfect generalization
%the one-way variant of perfect generalization mentioned above (where only the right-hand inequality in (\ref{eqn:pg-intro}) holds), 
with similar parameters. Namely, if an algorithm $A$ has $(\eps, \delta)$-max-information with respect to product distributions, then it is also $(\sqrt{\delta}, 2\eps, \sqrt{\delta})$-one-way perfectly generalizing. The idea is to take the simulator distribution $Sim_D$ to be the distribution of $A(S')$, where the randomness is taken over both the coin tosses of $A$ \emph{and} the randomness of a sample $S' \sim D$. A similar argument is implicit in~\cite[Proof of Lemma 4.5]{BassilyF16}. 
Then by combining Theorems \ref{PGtoRep} and \ref{thm:reprodtoPG}, it follows that bounded max-information also implies perfect generalization for finite outcome spaces (Theorem~\ref{thm:opg-to-pg}).

\paragraph{Replicability implies differential privacy.} 
%The final conversion in Figure~\ref{fig:transformations} is from replicable algorithms to differentially private ones. 
In \Cref{thm:Rep-to-DP} we show that replicability implies differential privacy. %\footnote{We note that the same implication can also be obtained indirectly going through perfect generalization via Theorems \ref{thm:reprodtoPG} and \cite{CummingsLNRW16}, but with a quadratic rather than linear dependence on $\eps$.} 
Given a replicable algorithm, one can run it $k = O(\log(1/\delta) / \eps)$ times using the same sequence of coin tosses, but on independent samples, producing outcomes $y_1, \dots, y_k$. Replicability ensures that most of these outcomes are the same with high probability, and so this common outcome can be selected in a standard differentially private way. This argument appears in the differential privacy literature as a conversion from ``globally stable'' and ``pseudo-globally stable'' learners to private ones~\cite{bun2020equivalence, ghazi2021sample, GhaziKM21}. Our presentation of Theorem~\ref{thm:Rep-to-DP} includes an additional amplification step that avoids union bounding over correctness, making the conversion suitable for a broader range of parameters.
% \sstext{Last sentence not applicable anymore.}\maxh{Are you sure? In GKM they just union bound, so their analysis only really works for problems where $\beta$-dependence is inverse exponential. This amplification refers to our method of drawing additional samples+random strings then applying Markov, which GKM doesn't do.} \sstext{Ah, I see, I misinterpreted the statement to imply correctness amplification when you moved to DP. Maybe slight rephrasing will help.}

\medskip

\noindent{\bf Differential privacy implies bounded max-information.}
The final conversion in Figure~\ref{fig:transformations} is from differentially private algorithms to algorithms with bounded max-information.
This argument is implicit in \cite{RogersRST16} and we show how it follows from their work here
(Corollaries~\ref{cor:max-info-concrete} and~\ref{cor:dp-to-max-info}).

\subsection{Further Discussion of Related Work}

Several elements of our approach were inspired by Ghazi, Kumar, and Manurangsi's study of the relationship between user-level and item-level differentially private learning~\cite{GhaziKM21}. They introduced a notion of ``pseudo-global stability'' that is essentially the same as replicability, and showed that it implies differential privacy. Correlated sampling also played a crucial role in their work by allowing individuals to use shared randomness to reach consensus on a learned hypothesis. In fact, it provided a key step in their conversion from ``list globally stable'' algorithms~\cite{ghazi2021sample} (learning algorithms that output a short list of hypotheses, one of which is almost guaranteed to be canonical for the given distribution) to pseudo-globally stable ones.

Stability in learning has a long history as a tool for ensuring generalization. Early work~\cite{RogersW78, DevroyeW79, bousquet2002stability, shalev2010learnability} showed that the stability of a learning algorithm with respect to a specific loss function could ensure strong generalization guarantees with respect to that loss. A more recent literature has focused on stability notions that are not tied to a specific loss, and which ideally are robust under post-processing and adaptive composition. This includes understanding the generalization guarantees of differential privacy~\cite{DworkFHPRR15, DworkFHPRR15-1, BassilyNSSSU16, RogersRST16, RogersRSSTW20, ZrnicH19, JungLNRSS20} and other constraints on the information-theoretic relationship between the input and output of a learning algorithm~\cite{RussoZ16, BassilyMNSY18, XuR17, RaginskyRTWX16, LigettS19, SteinkeZ20}. A related line of work~\cite{CummingsLNRW16, BassilyF16, NissimSSU18} considers more ``semantic'' notions of stability, defining it in terms of the difficulty of inferring properties specific to the sample rather than of the underlying distribution. Perfect generalization, one of the main definitions we study in this work, was introduced by~\cite{CummingsLNRW16} and is a special case of \emph{typical stability} that was introduced in independent work of Bassily and Freund~\cite{BassilyF16}.

Independent of this work, \cite{KKMV23} study similar relationships between notions of stability. They focus on the PAC-learning setting, where they show a statistical equivalence between differential privacy, replicability, and a notion called ``TV-indistinguishability'' which can be thought of as a special case of perfect generalization with $\eps=0$. 
% (via somewhat different intermediate notions of stability), but restricted to the PAC learning setting. 
To clarify the differences between our work and \cite{KKMV23}, first recall how we obtain replicability from differential privacy: 
\begin{itemize}
\item First, we exploit existing connections between privacy and bounded max-information from \cite{RogersRST16} to obtain an algorithm with bounded max-information from a differentially private one. 
\item We prove that bounded max-information implies perfect generalization.
\item We then show that we can obtain a replicable algorithm from a perfectly generalizing one by applying correlated sampling to its distribution over outputs. The relevant output distribution is induced by fixing an input sample of the perfectly generalizing algorithm and redrawing its internal randomness. 
\end{itemize}

Recall that the correlated sampling procedure may not be efficient, and that we assume the output domain of the differentially private algorithm is finite.

The work of \cite{KKMV23} follows a different approach. First, they start from a differentially private PAC learner, rather than a differentially private algorithm for a general statistical task, and factor through TV-indistinguishability and Littlestone dimension. More specifically:
\begin{itemize}
    \item They first observe a similar equivalence of TV-indistinguishability and replicability for general statistical tasks. 
    \item They then show that a private PAC learner implies the existence of a TV-indistinguishable learner, leveraging results from~\cite{alon2019private} showing that private PAC learning implies finite Littlestone dimension, and results from~\cite{GhaziKM21,ghazi2021sample} showing that finite Littestone dimension implies list global-stability.
    % algorithm for a statistical task implies a replicable algorithm for the same task.
\end{itemize}

Our approach gives us a constructive procedure for converting a private algorithm for a general statistical task into a replicable algorithm, so long as the private algorithm has finite range. Our transformations induce a modest sample complexity increase, resulting in a replicable algorithm with sample complexity $n^2$, given a private learner with sample complexity $n$. By contrast, the results of \cite{KKMV23}, while non-constructive, apply to countably infinite domains (and therefore to some uncountably infinite ranges). However, their results go through Littlestone dimension, which may be an exponential tower in $n$, and so they obtain sample complexity bounds which are an exponential tower in $n$ as well. 

\subsection{Open Problems}

We highlight several directions and open problems for future work. \mb{Everyone should add their favorite question(s) here!}

\begin{enumerate}
    \item Is a transformation from (one-way) perfectly generalizing algorithms to replicable algorithms possible for infinite output spaces in general? While correlated sampling introduces no sample complexity overhead in terms of the output space, it is only known to be possible when the output space is finite or the class of distributions to be sampled from is structured. (E.g., the distributions in the class all have uniformly bounded Radon-Nikodym derivative with respect to some fixed base measure).\footnote{Formally, such a case would fall into a restricted notion of correlated sampling over a subset of distributions, similar to the multiple coupling of \cite{AngSpi19}.} 
    In independent work, \cite{KKMV23} make progress towards this goal by giving a transformation from TV-indistinguishability to replicability when there are only countably many options for the TV-indistinguishable algorithm $\{A(S)\}_{S \in X^n}$. It follows from Lemma~\ref{lem:PGeps0} that $(\beta, \varepsilon, \delta)$-one-way perfect generalization implies $(4\varepsilon + 2\delta + 2\beta)$-TV indistinguishability, and so the result of~\cite{KKMV23} gives the following corollary.
    \begin{corollary}
        Fix $n \in \N$, $\beta,\varepsilon, \delta \in (0,1]$. Let $\X$ be a countable domain and $A: \X^n \rightarrow \mathcal{Y}$ be a $(\beta,\varepsilon, \delta)$-one-way perfectly generalizing algorithm for a statistical task. Then there exists an algorithm $A': \X^n \rightarrow \mathcal{Y}$ that is $\left( \frac{2\rho}{1 + \rho} \right)$-replicable for $\rho = 4\varepsilon + 2\delta + 2\beta$, and for all $S \in \X^n$, $A(S) = A'(S)$.
    \end{corollary}
    Whether a transformation exists for general measure spaces remains open.\footnote{We note that in the PAC-setting one can resolve this issue via factoring through Littlestone Dimension and \cite{ImpLPS22}'s heavy-hitters, but this results in tower sample complexity.} 
    %{\color{red}In independent work, \cite{KKMV23} observe this transformation can be adapted to the setting when there are only countably many options for the perfectly generalizing algorithm $\{A(S)\}_{S \in X^n}$, but the problem remains open for general measure spaces.}\footnote{\textcolor{red}{We note that in the PAC-setting one can resolve this issue via factoring through Littlestone Dimension and \cite{ImpLPS22}'s heavy-hitters, but this results in tower sample complexity.}} 
    In Section~\ref{sec:list-heavy-hitters} we discuss the \emph{list heavy-hitters} problem that may be a candidate for separating perfect generalization from replicability over infinite output spaces.
    % \item Does our lower bound of $\tilde{\Omega}(VC(H)^2)$ on the sample complexity of replicable agnostic learning apply to realizable learning as well? To what extent can we close the gap between $\tilde{\Omega}(VC(H)^2)$ and $\tilde{O}(\log^2|H|)$ for learning finite classes?
    \item What are the minimal cryptographic assumptions under which a computational separation between replicability and differential privacy exists? Our results in Section~\ref{sec:separating-stability-computationally} show that one-way functions are necessary, while public-key assumptions are sufficient.
    \item \cite[Lemma A.7]{ImpLPS22} showed that replicable algorithms compose adaptively. That is, a sequence of $k$ adaptively chosen $\rho$-replicable algorithms yields a transcript that is $O(k\rho)$-replicable. One way to interpret this result is as follows: Given a sequence of $k$ analyses that are each $(0.01)$-replicable using a sample of size $n$, one can amplify their individual replicability parameters to $O(1/k)$ at the expense of increasing their sample complexity to $O(k^2 n)$. This yields a $(0.01)$-replicable algorithm for performing all $k$ analyses at a sample cost of $O(k^2 n)$.
    
    Our conversions between replicability and differential privacy yield a different tradeoff, at least for simulating non-adaptive composition. Given $k$ analyses that are each $(0.01)$-replicable using a sample of size $n$, one can convert them to $\tilde{O}(1/\sqrt{k})$-differentially private algorithms each using a sample of size $\tilde{O}(\sqrt{k}n)$. ``Advanced'' composition of differential privacy~\cite{DworkRV10} yields an $(0.01, \delta)$-differentially private algorithm using $\tilde{O}(\sqrt{k} n)$ samples, which can then be turned back into a $(0.01)$-replicable algorithm using $\tilde{O}(kn^2)$ samples.
    
    What is the optimal sample cost for conducting, or at least statistically simulating, the (adaptive) composition of $k$ replicable algorithms? Is it possible to do so at a cost of $O(kn)$ samples?
    \item In Section~\ref{sec:finiteclasses}, we give a direct replicable algorithm for the task of realizable PAC learning of finite classes with sample cost inverse linear in the accuracy parameter $\alpha$. (As opposed to inverse quadratic, which is what applying the reduction from replicability to approximate DP gives -- see Theorem~\ref{thm:finitehypred} and the following discussion.) Are there other natural problems for which there are (perhaps more dramatic) separations between what's achievable via directly constructing a replicable algorithm for a task, and what's achievable using our reduction to approximate DP? For example, can discrete distributions over $[k]$ be replicably estimated using $O(k)$ samples (as opposed to quadratic in $k$, which is what is obtained through our reduction)? Can the mean of a $d$-variate Gaussian with unknown covariance be estimated directly using $O(d)$ samples (as opposed to quadratic in $d$, which is what is obtained through our reduction)? Even more ambitiously, is it possible to characterize the types of problems for which our reduction from replicability to approximate DP gives tight bounds?

    \item To what extent is replicability preserved under distributional shift? In Appendix~\ref{app:additional-properties-of-replicability}, we give a simple argument showing that a $\rho$-replicable algorithm is $\rho(1-\delta)^{2m}$-replicable across two close distributions. Are there tighter replicability and non-replicability bounds for specific families of distributions, problems, and algorithms under distributional shifts? 
\end{enumerate}

\section{Preliminaries}
%\rexnote{Add definitions (of``statistical task" and ``empirical task"), distributional-based vs input-based correctness or remove this section entirely. Toni adds: do we need to define statistical task more than what we did in Intro?}

We start by formally defining a statistical task.
\begin{definition} \label{def:stat-task}
A \emph{statistical task} with data domain $\mathcal{X}$ and output space $\mathcal{Y}$ is a set of pairs $\mathcal{T} = \{(D, G_D)\}$, where $D$ is a distribution over $\mathcal{X}$ and $G_D \subseteq \mathcal{Y}$ is a ``good'' set of outputs for distribution $D$. A randomized algorithm $\mathcal{A}$ solves statistical task $\mathcal{T}$ using $m$ samples and with failure probability $\beta$ if for every $(D, G_D) \in \mathcal{T},$
\[\Pr_{\Datafixed \sim D^m, \mathcal{A}}[\mathcal{A}(\Datafixed) \in G_D] \ge 1-\beta.\]
\end{definition}
%\toni{add comments about how this  generalizes PAC}
\subsection{Notions of Distributional Closeness}
We recall the definition of total variation distance, that will be crucial in this work.
\begin{definition}[Total Variation Distance] \label{def:DTV} Let $P$ and $Q$ be probability distributions over some domain $S$. Then
\begin{equation*}
    \dtv(P, Q) := \sup_{E\subseteq S} |\Pr_{P}[E] - \Pr_{Q}[E]|.
\end{equation*}
\end{definition}
We also define the notion of $(\eps, \delta)$-indistinguishability, the notion of closeness that is used in differential privacy. 
\begin{definition}[$(\eps, \delta)$-indistinguishability] \label{def:TV} Let $P$ and $Q$ be probability distributions over some domain $\mathcal{Y}$. Then, we say that $P$ is $(\epsilon, \delta)$-indistinguishable from $Q$ (denoted as $P \approx_{\eps, \delta} Q$) if for all $O \subseteq \mathcal{Y}$,
\[
 e^{-\epsilon} [\Pr_P[O] - \delta] \leq \Pr_Q[O] \leq e^{\epsilon} \Pr_P[O] + \delta.
\]
\end{definition}
We will frequently talk about random variables being $(\eps, \delta)$-indistinguishable, which means that their distributions are $(\eps, \delta)$-indistinguishable. 
\subsection{Differential Privacy}

%\mb{Should have definitions of our measures of statistical closeness, like $\approx_{\eps, \delta}$ and total variation. I didn't find them anywhere?}

We say that two datasets $S, S'\in \mathcal{X}^n$ are neighboring if they differ for the data of one individual, i.e., their Hamming distance is one. Differential privacy is formulated as a notion of indistinguishability between the results of  an algorithm when run on neighboring datasets.  %\mg{Notice that we already introduced DP in the intro. I think it is fine to repeat it here but if we need to save space we can probably avoid repeating.}

\begin{definition}[Differential Privacy~\cite{DworkMNS16j}]\label{def:differentially private} A randomized algorithm $\mathcal{A}: \mathcal{X}^n \rightarrow \mathcal{Y}$ is said to be {\em $(\eps, \delta)$-differentially private} if for every pair of neighboring datasets $S, S'\in \mathcal{X}^n$, we have that for all subsets $O \subseteq \mathcal{Y}$,
 \begin{equation*}
    \Pr[\sampler(\Datafixed) \in O] \leq e^\eps \cdot \Pr[\sampler(\Datafixed') \in O] + \delta.
 \end{equation*}
 That is, we have $\mathcal{A}(S) \approx_{\eps, \delta} \mathcal{A}(S')$ for all neighboring $S, S'$.
 \end{definition}

One important property of differential privacy is that it is closed under post-processing by arbitrary functions.

\begin{lemma}[Post-Processing~\cite{DworkMNS16j}]\label{prelim:postprocess} If $\mathcal{A}: \mathcal{X}^n \rightarrow \mathcal{Y}$ is $(\eps, \delta)$-differentially private, and $\mathcal{B} : \mathcal{Y} \rightarrow \mathcal{Z}$ is any randomized function, then the algorithm $\mathcal{B} \circ \mathcal{A}$ is $(\eps, \delta)$-differentially private.
\end{lemma}

%Standard $(\eps, \delta)$-differential privacy protects the privacy of groups of individuals.

%\begin{lemma}[Group Privacy~\cite{DworkMNS16j}]\label{prelim:group_privacy}
%Each $(\eps, \delta)$-differentially private algorithm $\mathcal{A}$ satisfies the following property for all datasets $S$ and $S'$ that differ in $k$ coordinates,
%\begin{equation*}
%    \Pr[\mathcal{A}(S) \in Y] \leq e^{k\eps} \cdot \Pr[\mathcal{A}(S') \in Y] + \delta \cdot \frac{e^{k\eps} -1}{e^\eps-1}.
%\end{equation*}
%\end{lemma}

Differential privacy can be achieved by adding some carefully chosen noise to a function and calibrating the noise to the sensitivity of the function: a measure of how different can the results of the function be when run on adjacent datasets.  
%The sensitivity of a function is an important attribute that controls the amount of noise that needs to be added in many differential privacy mechanisms. 

\begin{definition}[$\ell_1$-Sensitivity] Let $f: \mathcal{X}^n \rightarrow \mathbb{R}^d$ be a function. Its $\ell_1$-sensitivity is
\begin{equation*}
    \Delta_f = \max_{\substack{S, S' \in \mathcal{X}^n \\ S, S' \text{neighbors}}} \|f(S) - f(S')\|_1.
\end{equation*}
\end{definition}

There are many techniques that can be used to design differentially private algorithms. One important technique that we will use in some of our applications is the exponential mechanism.
\begin{lemma}[Exponential Mechanism \cite{McTalwar}]\label{lem:expmech}
Let $L$ be a set of outputs and $g: L \times \mathcal{X}^n \to \mathbb{R}$ be a function that measures the quality of each output on a dataset. Assume that for every $m \in L$, the function $g(m,.)$ has $\ell_1$-sensitivity at most $\Delta$. Then, for all $\eps>0$, there exists an $(\eps, 0)$-DP mechanism that, on input $S\in \mathcal{X}^n$, outputs an element $m\in L$ such that, for all $a>0$, we have
\begin{equation*}
    \Pr\left[\max_{i \in [L]} g(i,S) -  g(m,S) \geq 2\Delta \frac{\ln |L| + a}{\eps}\right] \leq e^{-a}. 
\end{equation*}
\end{lemma} 
%An important property of differential privacy is the `group privacy' property.

Standard $(\eps, \delta)$-differential privacy automatically protects the privacy of groups of individuals.

\begin{lemma}[Group Privacy~\cite{DworkMNS16j}]\label{lem:group_privacy}
Let $k \in \mathbb{N}^+$ and let $\Acal:\mathcal{X}^m \to \mathcal{Y}$ be an $(\eps, \delta)$-DP algorithm. Then for all datasets $\Datafixed, \Datafixed' \in \mathcal{X}^m$ such that $\|\Datafixed - \Datafixed' \|_0 \leq k$,
\begin{equation*}
    \Acal(\Datafixed) \approx_{k \eps, \delta \frac{e^{k\eps}-1}{e^{\eps}-1}} \Acal (\Datafixed').
\end{equation*}
\end{lemma}

Another property of differential privacy that we will use in many of our algorithms is privacy amplification by subsampling. This says that we can have a stronger privacy protection when we run a differentially private algorithm on a subsample of a dataset.

\begin{lemma}[Secrecy of the sample, \cite{kasiviswanathan2011can, BalleBG18}] 
\label{lem:samp-wo-replace}
Let $A : \mathcal{X}^n \to \mathcal{Y}$ be an $(\eps, \delta)$-differentially private algorithm. Consider the algorithm $A' : \mathcal{X}^m \to \mathcal{Y}$ that, given a dataset of size $m$, randomly samples $n$ items without replacement and runs $A$ on the resulting subsample. Then $A'$ is $(\eps', \delta')$-differentially private for
\[\eps' = \frac{n}{m}(e^\eps - 1), \qquad \delta' = \frac{n}{m} \cdot \delta.\]
\end{lemma}

\subsection{Replicability}

Replicability is a strong stability property for randomized algorithms, requiring that the algorithm produce the exact same output with high probability when invoked on two i.i.d. samples from the same distribution, so long as the internal randomness is held fixed. We recall the definition of replicability given in~\cite{ImpLPS22}. \toni{This definition is in the intro so I would vote to remove it here, possibly referring to defn in intro.}

\begin{definition}[\cite{ImpLPS22}]\label{def:replicability}
Let $D$ be a distribution over domain $\mathcal{X}$. Let $\mathcal{A}$ be a randomized algorithm that takes as input samples from $D$. We say that $\mathcal{A}$ is $\rho$-reproducible if
$$\Pr_{S, S', r}[\mathcal{A}(S;r) = \mathcal{A}(S';r)] \geq 1 - \rho, $$
where $S, S'$ are sets of samples drawn i.i.d. from $D$ and $r$ represents the internal randomness of $\mathcal{A}$. 
\end{definition}

We will sometimes use the alternative $2$-parameter definition of replicability defined in \cite{ImpLPS22}. Here we assume that the auxiliary inputs described in their original definition are empty.\footnote{See Section 2.5 for an explanation of our  renaming of this definition to ``replicable.''}

\begin{definition}[\cite{ImpLPS22}] \label{2param-defn}
Let $A(S; r)$ be an algorithm operating on a sample set $S \in \mathcal{X}^n$ and internal coins $r$. We say that coin tosses $r$ are $\eta$-good for $A$ on distribution $D$ if there exists a ``canonical output'' $z_r$ such that $\Pr_{S \sim D^n}[A(S;r)=z_r] \geq 1-\eta$. We say that $A$ is $(\eta,\nu)$-replicable if, for every distribution $D$, with probability at least $1-\nu$, the coin tosses $r$ are $\eta$-good on distribution $D$.
\end{definition}

\cite{ImpLPS22} observed that the two parameter and the original single parameter definition (Definition \ref{def:replicable}) are essentially equivalent:

\begin{claim}[\cite{ImpLPS22}]\label{claim:2parrep}
For every $0 \leq \rho \leq v \leq 1$,
\begin{enumerate}
    \item Every $\rho$-replicable algorithm is also $(\rho/v, v)$-replicable.
    \item Every $(\rho,\nu)$-replicable algorithm is also $\rho + 2\nu$-replicable.
\end{enumerate}
\end{claim}

It was proved in \cite{ImpLPS22} that we can amplify the replicability parameter at an inverse quadratic cost in the desired replicability parameter. We state a version of this theorem with slightly different constants.
\begin{lemma}[Amplification of Replicability, Theorem A.3, \cite{ImpLPS22}]\label{thm:amprep}
Let $0 < \eta, \nu, \beta < \frac{1}{2}$ and $m>0$. Let $\mathcal{A}$ be an $(\eta, \nu)$-replicable algorithm for distribution $D$ with sample complexity $m$ and failure probability $\beta$. If $\rho > 0$, and $\nu + \rho < 0.25$, there exists a $\rho$-replicable algorithm $\mathcal{A}'$ for $D$ with sample complexity $m ' = \tilde{O}(m(\log 1/\beta)^3 / \rho^2 (1/2-\eta)^2)$ and failure probability at most $4(\beta+\rho)$.
\end{lemma}

\subsection{PAC-Learning}\label{sec:PAC}

We start by defining PAC Learning, which is a canonical definition of supervised learning proposed by Valiant \cite{valiant1984theory} and Vapnik and Chervonenkis \cite{vapnik1974theory}. We first consider the realizable setting.

\begin{definition}[Realizable PAC learning, \cite{valiant1984theory,vapnik1974theory}]
A learning problem is defined by a hypothesis class $H$. For any distribution $D$ over the input space $\mathcal{X}$, consider $m$ independent draws $x_1, x_2, \cdots x_m$ from distribution $P$. A labeled sample of size $m$ is the set $\{(x_1, f(x_1)), (x_2, f(x_2)), \cdots, (x_m, f(x_m)) \}$ where $f \in H$. We say an algorithm $A$ is an $(\alpha, \beta)$-accurate PAC learner for the hypothesis class $H$ if for all functions $f \in H$ and for all distributions $D$ over the input space, $A$ on being given a labeled sample of size $m$ drawn from $D$ and labeled by $f$, outputs a hypothesis $h$ such that with probability greater than or equal to $1 - \beta$ over the randomness of the sample and the algorithm,
\begin{equation*}
    \Pr_{x \in D}[h(x) \neq f(x)] \leq \alpha.
\end{equation*}
\end{definition}

We also consider a variant called agnostic PAC learning, where the labels of the input dataset can be noisy.
\begin{definition}[Agnostic PAC learning, \cite{haussler1992decision,vapnik1974theory}]
A learning problem is defined by a hypothesis class $H$. We say an algorithm $\mathcal{A}$ is an $(\alpha, \beta)$-accurate PAC learner for the hypothesis class $H$ if for all distributions $D$ over input, output pairs, $\mathcal{A}$ on being given a sample of size $m$ drawn i.i.d. from $D$ outputs a hypothesis $h$ such that with probability greater than or equal to $1 - \beta$ over the randomness of the sample and the algorithm,
\begin{equation*}
    \err_D(h) \leq \inf_{f \in H} \err_D(f)+\alpha.
\end{equation*}
where $\err_D(h) = \Pr_{(x,y) \in D} [h(x) \neq y]$. In this context, we will sometimes refer to PAC-learning as the \textbf{realizable setting}.
\end{definition}
We will need uniform convergence for several of our results.

\begin{theorem}[Uniform Convergence, e.g., \cite{Blumer}]
\label{thm:unifconv}
Let $H$ be a binary class of functions with domain $\mathcal{X}$. Let its VC dimension be $d$. Then, for any distribution $D$ over $\mathcal{X}$, for all $m>0$,
$$\Pr_{x_1,\dots,x_m \sim D}\left[\sup_{h_z \in H}\left|\frac{1}{m}\sum_{i=1}^m \mathbbm{1}[h_z(x_i) = 1] - \Pr_{x \sim D}[h_z(x)=1] \right| \geq \gamma \right] \leq 4 (2m)^d e^{-\gamma^2 m / 8}.                                      $$
\end{theorem}

\subsection{Correlated Sampling} \label{sec:correlated-sampling}

In the correlated sampling problem, two (or more) players are given probability distributions ${\cal P},{\cal Q}$ over the same finite set, and with access to shared randomness. The players (without communicating) want to sample from their respective distributions, while minimizing the probability that their outputs disagree.

More formally, let $Y = \{0,1\}$, and for a set $\Omega$,
$2^{\Omega}$ denotes  the set of all functions from $\Omega$ to $Y$, and $\Delta_{\Omega}$ denotes the set of all sampleable distributions on $\Omega$.
For two distributions ${\cal P}$, ${\cal Q}$ over $\Omega$, let
$d_{TV}({\cal P},{\cal Q})$ denote the total variational distance between ${\cal P}$ and ${\cal Q}$.
%$$d_{TV}({\cal P},{\cal Q}) = \frac{1}{2} \sum_{w \in \Omega}}

\begin{definition}(Correlated Sampling)
A correlated sampling strategy for a finite set $\Omega$ with error $\eps: [0,1]\to [0,1]$ is an algorithm $CS: \Delta_{\Omega} \times {\cal R'}$ and a distribution ${\cal R'}$ on random strings such that:
\begin{itemize}
    \item (Marginal Correctness) For all ${\cal P} \in \Delta_{\Omega}$ and $w \in \Omega$, $\Pr_{r' \sim {\cal R'}} [CS({\cal P},r')=w] = {\cal P}(w)$.
    \item (Error Guarantee) For all ${\cal P},{\cal Q} \in \Delta_{\Omega}$, $\Pr_{r' \sim {\cal R'}} [CS({\cal P},r') \neq  CS({\cal Q},r')] \leq \eps( d_{TV}({\cal P},{\cal Q}))$
\end{itemize}
\end{definition}
Several independent papers \cite{KleinbergT2002,Holenstein09,Broder97} give correlated sampling strategies over finite sets, with $\epsilon(\delta) = \frac{2 \delta}{1+\delta}$.
% where $\delta$ is 
% $d_{TV}({\cal P},{\cal Q})$. 
These algorithms use consistent sampling strategies, which are also used in several other contexts such as sketching algorithms, approximation algorithms, and parallel repetition theorems.
%using rejection sampling.
%In their works, they discuss rejection sampling (pick an element in $\Omega$, and find its probability; accept if a random coin is less than that probability) as a method to correlated sample. I think where that method fails in comparison to ours (with the owfi assumption) is their algorithm is runs in exponential time (in $n$, when sampling over domain $\{0,1\}^n)$.
%\toni{Toni- finish this}
However, despite the many uses of correlated sampling, some basic questions remain open. Notably, it is not known whether or not correlated sampling is possible for infinite domains, or whether correlated sampling can be made {\it efficient}.
(All known algorithms run in exponential-time in the worst-case.)
For a nice discussion as well as new results on the optimality of these constructions see \cite{BGM+16}.
%paper discussing correlated sampling: %\url{https://theoryofcomputing.org/articles/v016a012/v016a012.pdf}
In Section \ref{sec:PG2rep}, we give another application of correlated sampling, showing how any perfectly generalizing algorithm can be transformed into a replicable one, via correlated sampling. As noted in the introduction, this is the only implication that does not preserve computational efficiency, due to the inefficiency of  the  correlated sampling strategy.
 
In Section \ref{ssec:dp-rep-separation}, we prove that this is inherent: under standard  cryptographic assumptions
any such transformation is intractable, and therefore under the same cryptographic assumption, correlated sampling is intractable. Moreover, we show in Section \ref{sec:owfi} that some assumption is necessary: if one-way functions do not exist (that is, all poly-time computable functions can be efficiently inverted), we show that this implies a polynomial-time algorithm for correlated sampling.

%\toni{todo:Define KL, $d_{TV}$, correlated sampling with references.}

\subsection{Terminology: ``Reproducibility'' and ``Replicability''}
\label{ssec:terminology-replicability}

\cite{ImpLPS22} introduced a mathematical definition referring to a particular stability notion of a randomized learning algorithm which they originally called ``reproducibility'' (reproducible algorithms). In this paper, we use the term ``replicability'' (replicable algorithms) to refer to the same mathematical definition. 

This terminology choice is more in line with the most current Association for Computing Machinery (ACM) guidance regarding artifact review and badging \cite{acmArtifactGuidelines}, version 1.1, updated on August 24, 2020. This update changed the ACM's definitons of the terms ``reproducible'' and ``replicable'' to be more agreeable with the terminology currently used by the National Academies of Sciences, Engineering and Medicine (see Chapter 3: Understanding Reproducibility and Replicability, page 46, in \cite{NAP25303}).

%https://www.acm.org/publications/policies/artifact-review-and-badging-current
%https://nap.nationalacademies.org/catalog/25303/reproducibility-and-replicability-in-science

    According to both the ACM's and National Academies' current definitions, ``reproducibility'' refers to the ability of a second experimental group to obtain similar results using the \emph{same} input data. Meanwhile, ``replicability'' refers to the ability of a second experimental group to obtain similar results using input data and methods that may be different than those used by the original experimental group.
    
    The mathematical definition introduced in \cite{ImpLPS22} is a guarantee that, with high probability, two executions of the same algorithm with the same randomness and different sample sets will produce the same answer. Since this guarantee is over different sample sets, the mathematical definition does not fit the ``same input data'' condition in the above definitions of reproducibility. 
    Instead, the mathematical definition is a specific type of replicability --- if the second experimental group runs the same algorithm with the same random string (but on a new sample), the two groups' results are guaranteed to be identical with high probability.

\section{Equating Stability: Differential Privacy, Perfect Generalization, and Replicability}

\subsection{Replicability implies Approximate-DP}
In \cite{GhaziKM21}, the authors show a sample-efficient reduction from differentially private PAC learning to replicable PAC learning. In this section, we show their technique generalizes to arbitrary statistical problems.\footnote{The argument remains similar to \cite{GhaziKM21}, but requires a few changes to avoid union bounding over failure probability which can be costly in settings beyond PAC learning.}

Recall the definition of a statistical task in Definition~\ref{def:stat-task}. We will show that any statistical task with a ``good'' replicable learner can also be solved privately, without substantial blowup in runtime or sample complexity. 
\begin{theorem}[Replicability $\to$ DP]\label{thm:Rep-to-DP}
Let $\mathcal{T}$ be a statistical problem. 
For all $\beta>0$, if there is a $0.01$-replicable algorithm solving $\mathcal{T}$ using $n_R$ samples and with failure probability $\beta$, then for any $0 < \varepsilon,\delta \leq 1$ there is an $(\varepsilon,\delta)$-DP algorithm for $\mathcal{T}$ using $n_{DP}$ samples with failure probability  $O \left(\beta \log \frac{1}{\beta} \right)$, where
\begin{align*}
n_{DP}(\epsilon,\delta,\beta) \leq n_R \cdot O\left( \frac{\log\delta^{-1}\log\beta^{-1}}{\varepsilon}+\log^2\beta^{-1}\right)
\end{align*}

\end{theorem}

This conversion relies on the following private algorithm for selecting an approximate mode.

\begin{theorem}[DP Selection \cite{korolova2009releasing,bun2016simultaneous, BunDRS18}]
There exists some $c>0$ such that for every $\varepsilon,\delta>0$ and $m \in \mathbb{N}$, there is an ($\varepsilon,\delta$)-DP algorithm that on input $S \in \mathcal{X}^m$, outputs with probability $1$ an element $x \in X$ that occurs in $S$ at most $\frac{c\log \delta^{-1}}{\varepsilon}$ fewer times than the true mode of $S$. Moreover, the algorithm runs in $\text{poly}(m,\log(|\mathcal{X}|))$ time.
\end{theorem}

The idea, as in \cite{GhaziKM21}, is to use replicablity to construct a sample over the output space where some correct solution appears many times. In particular, given a replicable algorithm $\mathcal{A}$ on $n$ samples, consider the following simple procedure adapted from \cite{GhaziKM21}: partition a larger data set, run $\mathcal{A}$ on each part, and privately output a commonly repeated element.

\begin{algorithm}[H]
\KwResult{Privately ouputs solution to $(\mathcal{X},R)$}
\nonl \textbf{Input:} Statistical Problem $(\mathcal{X},R)$. Distribution $D$ over $\mathcal{X}$, Replicable algorithm $\mathcal{A}$ on $n$ samples\\
\nonl \textbf{Parameters:} 
\begin{itemize}
    \item Privacy and Correctness $\beta,\eps,\delta>0$
    \item Seed Number $k_1 = O(\log\beta^{-1})$
    \item Partition Number $k_2 = O\left(\frac{\log\delta^{-1}}{\varepsilon}+\log\beta^{-1}\right)\cdot k_1$
\end{itemize}
\nonl \textbf{Algorithm:}\\

 \begin{enumerate}
\item For every $j \in [k_1]$ and $i \in [\frac{k_2}{k_1}]$ sample $S_{i,j} \sim {\cal X}^n$\label{stp:draws}
\item Sample $k_1$ random strings $\{r_j\}$. 
\item Let $y_{i,j} = \mathcal{A}(S_{i,j}; r_{j})$. \label{stp:outputs}
\item Run ($\varepsilon,\delta$)-DP Selection on $\{y_{i,j}\}$ and denote the output by $y^*$.
\end{enumerate} 
\textbf{return} $y^*$
 \caption{DP-to-Replicability Reduction}
\label{alg:Rep-to-DP}
\end{algorithm}

% We argue that as long as $k_1,k_2$ are taken sufficiently large, it is very likely the common element output by the algorithm is a `canonical' good output of the reproducible algorithm, and therefore also a correct solution. 

Recalling the two-parameter definition of replicability (\ref{2param-defn} and \ref{claim:2parrep}), since our subroutine is $0.01$-replicable and $\beta$-correct, it is 
also $(0.1, 0.1)$-replicable and $\beta$ correct.
Therefore, the proof of \Cref{thm:Rep-to-DP} is an immediate consequence of the following proposition.
\begin{proposition}\label{prop:Rep-to-DP}
For all sufficiently small $\beta,\varepsilon,\delta >0$, if $\mathcal{A}$ is $(0.1, 0.1)$-replicable and has failure probability $\beta$, then \Cref{alg:Rep-to-DP} is $(\varepsilon,\delta)$-private and has failure probability $O(\beta \log 1/\beta)$.
\end{proposition}
\begin{proof}
Privacy is essentially immediate from DP Selection. This follows because the input to selection based on a neighboring input database $T'$ differs in at most one of the $\{y_i\}$ (as we've partitioned the sample disjointly). Thus the reduction automatically inherits $(\varepsilon,\delta)$-privacy from DP Selection.
The main interest in the reduction, then, is maintaining correctness which we argue next. The proof breaks into two parts:
\begin{enumerate}
    \item With probability $1-\beta/2$, some $y^* \in \{y_i\}$ appears at least $t_1 \coloneqq 2c\left(\frac{\log \delta^{-1}}{\varepsilon}+\log\beta^{-1}\right)$ times.
    \item With probability $1-\beta \log 1/\beta$, \textit{any} element appearing at least $t_2 \coloneqq c\left(\frac{\log \delta^{-1}}{\varepsilon}+\log\beta^{-1}\right)$ times is correct.
\end{enumerate}
The result then follows from observing that by a union bound both conditions hold with probability at least $1-O(\beta \log 1/\beta)$, and conditioned on this fact DP-Selection always outputs an element that occurs at least $t_1-c\frac{\log\delta^{-1}}{\varepsilon} \geq t_2$ times (which is then guaranteed to be correct).

It remains to prove the claims. For the first, note that since $\mathcal{A}$ is $(0.1,0.1)$-replicable, there exists some $.1$-good random string $r^* \in \{r_j\}$ with probability at least $1-\beta/4$. By a Chernoff bound, the probability that the canonical element corresponding to $r^*$ appears fewer than $t_1$ times is at most $\beta/4$, which proves the claim (for a large enough choice of $k_2$).

Finally, we argue any common element is correct. Since $\mathcal{A}$ is a $\beta$-correct algorithm, in expectation, the number of incorrect outputs is $\beta k_2$. Hence, by Markov's inequality, the probability that there are more than $b=O(\frac{k_2}{\log\beta^{-1}})$ incorrect outputs is at most $\beta \log 1/\beta$. For small enough choice of constant in the correctness of our replicable algorithm,\footnote{Note this choice can be taken universally with respect to all parameters and the statistical problem itself.} we can make $b<t_2$, so no element appearing at least $t_2$ times can be incorrect as desired.
\end{proof}

\subsection{Approximate-DP Implies One-Way Perfect Generalization}\label{sec:dp2PG}

\subsubsection{Preliminaries about Perfect Generalization}

Perfect generalization (Definition \ref{def:PG}) is a notion of stability that captures the idea (like differential privacy) that an algorithm $\mathcal{A}$ does not depend on its input samples too much. 
%It requires that with high probability over a randomly drawn dataset $S$, the distribution of $\mathcal{A}(S)$ is close to a distribution that does not depend on the specific input sample $S$.
%\begin{definition}[\cite{CummingsLNRW16}]
%An algorithm $A: \mathcal{X}^m \to \mathcal{Y}$ is said to be $(\beta,\eps,\delta)$-perfectly generalizing, if for every distribution $D$ over $\mathcal{X}$, there exists a distribution $SIM_D$ such that with probability at least $1-\beta$ over the draw of an i.i.d. sample $S \sim D^m$, $\mathcal{A}(S) \approx_{\eps, \delta} SIM_D$.
%\end{definition}

We also consider the following ``two-sample'' version of this definition. This is frequently easier to work with for symmetry reasons.
\begin{definition}
An algorithm $\mathcal{A}: \mathcal{X}^m \to \mathcal{Y}$ is said to be $(\beta,\eps,\delta)$-sample perfectly generalizing, if for every distribution $D$ over $\mathcal{X}$, with probability at least $1-\beta$ over the draw of two i.i.d. samples $S_1, S_2 \sim D^m$, $\mathcal{A}(S_1) \approx_{\eps, \delta} \mathcal{A}(S_2)$.
\end{definition}

Cummings et al. \cite{CummingsLNRW16} prove the following lemma relating perfect generalization to sample perfect generalization. 
\begin{lemma}[\cite{CummingsLNRW16}]\label{lem:samplePG}
If algorithm $\mathcal{A}: \mathcal{X}^m \to \mathcal{Y}$ is $(\beta,\eps,\delta)$-perfectly generalizing, then it is also $(\beta,2\eps,3\delta)$-sample-perfectly generalizing.
\end{lemma}

We can also prove a partial converse to this result; this will be frequently useful since it allows us to prove sample perfect generalization and invoke this result to get perfect generalization.
%\toni{Do we ever actually use this partial converse? We should either remove this or move to an Appendix.} \mb{The conversion from rep to PG is stated as a conversion from rep to sample PG, which can in turn be converted to PG}

\begin{lemma}
\label{lem:samplePGtoregPG}
Fix $\beta, \eps, \delta \in (0,1]$. If algorithm $\mathcal{A}: \mathcal{X}^m \to \mathcal{Y}$ is $(\beta,\eps,\delta)$-sample perfectly generalizing, then it is also $(\sqrt{\beta},\eps,\delta + \sqrt{\beta})$-perfectly generalizing.
\end{lemma}
\begin{proof}
 Since $\mathcal{A}$ is $(\beta,\eps,\delta)$-sample PG, we have that for any distribution $D$ over $\mathcal{X}$, with probability at least $1-\beta$ over the draw of two i.i.d. datasets $S_1, S_2 \sim D^m$, we have that $\Acal(\Datafixed_1) \approx_{\eps, \delta} \Acal(\Datafixed_2)$. This guarantees by the reverse Markov inequality that for $d=1-\sqrt{\beta}$, 
 \begin{align*}
     & \mathbb{E}_{\Datafixed_1}[\Pr_{\Datafixed_2}(\Acal(\Datafixed_1) \approx_{\eps, \delta} \Acal(\Datafixed_2))] \geq 1-\beta \implies \\
     & \Pr_{\Datafixed_1}[\Pr_{\Datafixed_2}(\Acal(\Datafixed_1) \approx_{\eps, \delta} \Acal(\Datafixed_2)) \geq d] \geq \frac{1-\beta-d}{1-d} \implies
     \\
     & \Pr_{\Datafixed_1}[\Pr_{\Datafixed_2}(\Acal(\Datafixed_1) \approx_{\eps, \delta} \Acal(\Datafixed_2)) \geq 1-\sqrt{\beta}] \geq 1-\sqrt{\beta}.
 \end{align*}
Now, set $Sim_D$ to be the distribution of $\Acal(S_2)$ where the randomness of the distribution is taken over both the randomness of the algorithm and the dataset.

For any fixed dataset $S_1$, let $G_{S_1}$ be the set of datasets $S_2$ such that $\Acal(\Datafixed_1) \approx_{\eps, \delta} \Acal(\Datafixed_2 )$.

Then, we get that with probability at least $1-\sqrt{\beta}$ over the draw of $S_1$, for any $O \subseteq Y$,
\begin{align*}
\Pr_{Sim_D}[O] & = \Pr_{S_2, \Acal}[\Acal(S_2) \in O] \\ & = \Pr_{S_2, \Acal}[\Acal(S_2) \in O \mid S_2 \in G_{S_1}] \Pr[S_2 \in G_{S_1}] + \Pr_{S_2, \Acal}[\Acal(S_2) \in O \mid S_2 \not\in G_{S_1}] \Pr[S_2 \not\in G_{S_1}] \\ & \leq
\Pr_{S_2, \Acal}[\Acal(S_2) \in O \mid S_2 \in G_{S_1}] + \sqrt{\beta}
\\ & \leq
e^{\eps}\Pr_{\Acal}[\Acal(S_1) \in O ] + \delta + \sqrt{\beta}.
\end{align*}
Similarly, we can also argue that 
\begin{align*}
\Pr_{Sim_D}[O] & = \Pr_{S_2, \Acal}[\Acal(S_2) \in O] \\ & = \Pr_{S_2, \Acal}[\Acal(S_2) \in O \mid S_2 \in G_{S_1}] \Pr[S_2 \in G_{S_1}] + \Pr_{S_2, \Acal}[\Acal(S_2) \in O \mid S_2 \not\in G_{S_1}] \Pr[S_2 \not\in G_{S_1}] \\ & =\geq
\Pr_{S_2, \Acal}[\Acal(S_2)] \in O \mid S_2 \in G_{S_1})(1-\sqrt{\beta})
\\ & \geq
e^{-\eps}\left(\Pr_{\Acal}[\Acal(S_1) \in O ] - \delta\right)(1-\sqrt{\beta}) \\ & \geq
e^{-\eps}\left(\Pr_{\Acal}(\Acal(S_1) \in O] - \delta - \sqrt{\beta}\right).
\end{align*}
Hence, with probability at least $1-\sqrt{\beta}$ over the draw of $S_1$,
$$Sim_D \approx_{\eps, \delta} \Acal(S_1).$$ 
\end{proof}

We also define a notion of ``one-sided'' perfect generalization %, that replaces KL distance (in the original definition) with variational distance.
which only requires the probability of events under $A(S)$ not to increase too much relative to their probability under the simulator distribution $Sim_D$. This new definition will be crucial to show the equivalence between replicability and perfect generalization, as well as for some of our applications.

\begin{definition}[One-way perfect generalization]
An algorithm $\mathcal{A}: \mathcal{X}^m \to \mathcal{Y}$ is said to be $(\beta,\eps,\delta)$-one-way perfectly generalizing if for every distribution $D$ over $\mathcal{X}$, there exists a distribution $Sim_D$ such that with probability at least $1-\beta$ over the draw of an i.i.d. sample $S \sim D^m$, for every output set $O \subseteq \mathcal{Y}$ we have that
$$\Pr[\mathcal{A}(S) \in O] \leq e^{\eps} \Pr_{Sim_D}[O] + \delta.$$
\end{definition}

Next, we prove a simple lemma relating the parameters achievable with perfect generalization.

\begin{lemma}\label{lem:PGeps0}
Fix $m \in \mathbb{N}$, $\beta, \eps, \delta \in (0,1]$. Let $\mathcal{A}: \mathcal{X}^m \to \mathcal{Y}$ be a $(\beta, \eps, \delta)$-one-way perfectly generalizing algorithm. Then, $\mathcal{A}$ is also $(\beta, 0, 2\eps + \delta)$-perfectly generalizing.
\end{lemma}
\begin{proof}
By the definition of one-way perfect generalization, we have that for all distributions $D$ over $\mathcal{X}$, there exists a distribution $Sim_D$, such that with probability $1- \beta$ over the draw of $S$, for all $O \subseteq \mathcal{Y}$,
$$ \Pr_{\mathcal{A}}[\mathcal{A}(S) \in O] \leq e^{\eps}\Pr_{Sim_D} [O] + \delta.$$
Using the fact that for $\eps \leq 1$, $e^{\eps} \leq 1 + 2\eps$, we get that
$$ \Pr_{\mathcal{A}}[\mathcal{A}(S) \in O] \leq \Pr_{Sim_D} [O](1 + 2\eps) + \delta,$$ 
which gives us that
$$ \Pr_{\mathcal{A}}[\mathcal{A}(S) \in O] \leq \Pr_{Sim_D} [O] + 2\eps + \delta.$$ 
Now, since this works for any set $O$, consider applying it to $O^c$. Then, we get that
$$ \Pr_{\mathcal{A}}[\mathcal{A}(S) \in O^c] \leq \Pr_{Sim_D} [O^c] + 2\eps + \delta,$$ 
which implies (by writing $\Pr_{\mathcal{A}}[\mathcal{A}(S) \in O^c] = 1 - \Pr_{\mathcal{A}}[\mathcal{A}(S) \in O]$ and likewise for $Sim_D$ and doing some algebraic manipulation) that
$$ \Pr_{Sim_D} [O] \leq \Pr_{\mathcal{A}}[\mathcal{A}(S) \in O] + 2\eps + \delta.$$ 
Hence, the lemma is proved.
\end{proof}

\cite{CummingsLNRW16} also proved a strong relationship between $(\eps, 0)$-differential privacy and perfect generalization. Specifically, they proved the following.

\begin{theorem}[\cite{CummingsLNRW16}]
If algorithm $\mathcal{A}: \mathcal{X}^m \to \mathcal{Y}$ is $(\eps, 0)$-differentially private, then for all $\beta > 0$, it is $(\beta, \eps \sqrt{2m \ln(2|\mathcal{Y}|/\beta)},0)$-perfectly generalizing.
\end{theorem}

They left establishing similar relationships between $(\eps, \delta)$-differential privacy and perfect generalization as an open question. 
%While we aren't able to directly prove such a result, 
We resolve this question for finite outcome spaces. Our argument is indirect and involves showing that any approximate differentially private algorithm is \textit{one-way} perfectly generalizing. We will later use this to prove that any approximate differentially private algorithm can be \textit{compiled} into another algorithm that is perfectly generalizing (under the original  definition). We do this through the tool of max information, that we discuss next.

\subsubsection{Preliminaries about Max Information}
The notion of max information was formulated in work on the connection between differential privacy and adaptive data analysis. It quantitatively captures the degree of correlation between two random variables, by comparing the joint distribution of the random variables to the product measure. Intuitively, if the joint distribution and the product measure are ``close'', then the random variables are not too correlated with each other. %Here, we give a two-sided version of the original definition (the original definition was one-sided).
\begin{definition}[Based on \cite{DworkFHPRR15}]
The $\beta$-approximate max information between two correlated random variables $X$ and $Z$, denoted $I^{\beta}_{\infty}(X,Z)$, is defined as the minimum (infimum) value $k$ such that for all output sets $O$, we have that
\begin{align}
   \Pr_{(a,b) \sim (X,Z)} [(a,b) \in O] \leq 2^k \Pr_{(a,b) \sim X \otimes Z} [(a,b) \in O] + \beta
\end{align}
where $X \otimes Z$ represents the product measure of the $2$ random variables.
\end{definition}

In this paper, we will be concerned about the degree of correlation between a randomly sampled dataset, and the output of an algorithm run on that dataset. Intuitively, replicability requires that an algorithm's output does not depend too much on the specific input sample it gets, so the max information between these two random variables will be a useful quantity to analyze. 
\subsubsection{Approximate Differential Privacy to Bounded Max Information}
Connections between max information and differential privacy have been previously studied. Rogers, Roth, Smith and Thakkar \cite{RogersRST16} give a bound on the max information between an approximate DP algorithms' outputs and its inputs (Theorem 3.1 in their paper). In fact, they prove the following general statement that can be seen by examining their proof of Theorem 3.1.
\begin{lemma}[\cite{RogersRST16}]\label{lem:max-info-bound}
Fix $m \in \mathbb{N}$, $\eps \in (0,1/2]$ and $\delta \in [0,\eps/15)$. Let $\mathcal{A}: \mathcal{X}^m \to \mathcal{Y}$ be an $(\eps, \delta)$-DP algorithm. Then, for any distribution $D$ over $\mathcal{X}$, if $S \sim D^m$, for all $t>0$, $I^{\beta}_{\infty}(S;\mathcal{A}(S)) \leq m \nu + 6t \eps \sqrt{m}$, where $\beta = e^{-t^2 / 2} + cm\sqrt{\frac{\delta}{\eps}}$, and $\nu = C(\eps^2 +  \sqrt{\frac{\delta}{\eps}})$ for some sufficiently large constants $c, C$.
\end{lemma}
We instantiate this lemma with parameters that are suitable for our application.
\begin{corollary}\label{cor:max-info-concrete}
Fix $m \in \mathbb{N}$, sufficiently small $\rho \in (0,1)$. Let $\eps = \frac{\rho}{\sqrt{8 m \log(1/\rho)}}$, $\delta \leq \frac{\eps \rho^6}{m^2}$. Let $\mathcal{A}: \mathcal{X}^m \to \mathcal{Y}$ be an $(\eps, \delta)$-DP algorithm. Then, for any distribution $D$ over $\mathcal{X}$, if $S \sim D^m$, $I^{\rho^3}_{\infty}(S;\mathcal{A}(S)) \leq O(\rho)$.
\end{corollary}
\begin{proof}
Substituting the value of $\eps$, $\delta$, and setting $t = \sqrt{8\log(1/\rho)}$ in the expression for $\beta$ in Lemma~\ref{lem:max-info-bound}, we get $$\beta = e^{-4\log(1/\rho)} + cm\sqrt{\frac{\rho^6}{m^2}} = O(\rho^3).$$ Substituting in the expression for $\nu$ in Lemma~\ref{lem:max-info-bound} gives $$\nu = C\left(\eps^2 + \sqrt{\frac{\delta}{\eps}}\right) \leq C\left(\frac{\rho^2}{m} + \frac{\rho^3}{m}\right) = O\left(\frac{\rho^2}{m}\right).$$ Substituting the values of $\nu, t$, and $\eps$ in the upper bound for max information gives $$I^{\rho^3}_{\infty}(S;\mathcal{A}(S)) \leq m \nu + 6t \eps \sqrt{m} = \rho^2 + 6\rho = O(\rho).$$
\end{proof}

In general, one can convert an $(\eps, \delta)$-differentially private algorithm with $\eps = O(1)$ and $\delta = 1/\poly(n)$ into a bounded max-information algorithm by first amplifying the privacy parameters:
\maxh{Are the constants here important? They make it a bit harder to read so if we can it might be better to just give asymptotics in the theorem statement}
\begin{corollary} \label{cor:dp-to-max-info}
There are constants $c, C > 0$ such that the following holds. Let $\gamma>0$. Suppose $\mathcal{A} : \mathcal{X}^n \to \mathcal{Y}$ is an $(\eps, \delta)$-differentially private algorithm solving a statistical task up to some failure probability $\beta$ such that $\eps \in (0, 1)$ and $\delta \le  \min\{\gamma^2/c^2, \rho^2 / C^2\} \rho^4 / 64\eps^3(2C\rho + \sqrt{72\log(2/\gamma)})^2n^4$. Then there is an algorithm $\mathcal{A}': \mathcal{X}^m \to \mathcal{Y}$ solving the same statistical task with the same failure probability $\beta$, such that
\[m = 4\eps^2n^2  \cdot \left(\frac{2C}{\rho} + \frac{\sqrt{72 \log(2/\gamma)})}{\rho^2}\right),\]
and for every distribution $D$ over $\mathcal{X}$, if $S \sim D^m$ we have $I_\infty^{\gamma}(S; \mathcal{A}(S)) \le \rho$. Moreover, the conversion from $\mathcal{A}$ to $\mathcal{A}'$ preserves computational efficiency.
\end{corollary}

\begin{proof}
 The algorithm $\mathcal{A}'$ simply samples $n$ items without replacement from $S$ and runs $\mathcal{A}$ on the result. This perfectly preserves correctness with respect to any statistical task. By Lemma~\ref{lem:samp-wo-replace}, we have that $\mathcal{A}'$ is $(\eps', \delta')$-differentially private for
    \[\eps' = \frac{2n}{m} \le \frac{1}{\sqrt{m}}\min\left\{\sqrt{\frac{\rho}{2C}}, \frac{\rho}{\sqrt{72 \log(2/\gamma)}}  \right\}, \qquad \delta' = \frac{n}{m} \cdot \delta.\]
Note that these parameters ensure that
\[\frac{\delta'}{\eps'} = \frac{\delta}{e^{\eps} - 1} \le \frac{\delta}{\eps} \le \min \left\{\frac{\gamma^2}{4c^2m^2}, \frac{\rho^2}{4C^2m^2} \right\}.\]
Now set $t = \sqrt{2\log(2/\gamma)}$ in the statement of Lemma~\ref{lem:max-info-bound}, which ensures that
\[e^{-t^2 / 2} + cm\sqrt{\frac{\delta'}{\eps'}} \le \frac{\gamma}{2} + cm \sqrt{\frac{\gamma^2}{4c^2m^2}}\le \gamma.\]
Then the lemma implies that
\begin{align*}
    I_\infty^{\gamma}(S; \mathcal{A}(S)) &\le Cm\left((\eps')^2 + \sqrt{\frac{\delta'}{\eps'}}\right) + 6t\eps'\sqrt{m} \\
    &\le Cm\left(\frac{\rho}{2Cm} + \frac{\rho}{2Cm}\right) + 6\sqrt{2\log(2/\gamma)} \cdot \frac{\rho}{\sqrt{72\log(2/\gamma)}}\\
    &\le \rho.
\end{align*}
%Thus, Corollary~\ref{cor:max-info-concrete} implies the result.
\end{proof}

\subsubsection{Bounded Max Information to One-Way Perfect Generalization}
Next, we prove a key lemma relating bounded max-information to one-way perfect generalization. The approach we follow is similar to that used to derive relationships between pointwise $(\eps, \delta)$-indistinguishability and $(\eps, \delta)$-indistinguishability in Lemma 3.3 of \cite{KasiviswanathanS14}.
\begin{lemma}\label{lem:max-info-PG}
Fix $m \in \mathbb{N}$, $k>0$, $\beta \in (0,1)$ and $\hat{\beta} = \sqrt{\frac{\beta}{1-2^{-k}}}$. Let $\mathcal{A}: \mathcal{X}^m \to \mathcal{Y}$ be an algorithm. Then, for every distribution $D$ over $\mathcal{X}$ and $S \sim D^n$, if $I^{\beta}_{\infty}(S;\mathcal{A}(S)) \leq k$, then $\mathcal{A}$ is $(\hat{\beta},2k,\hat{\beta})$-one-way perfectly generalizing.
\end{lemma}
\begin{proof}
The canonical distribution $Sim_D$ we will consider is the distribution of $\mathcal{A}(S')$, where the randomness is over both $S' \sim D^m$, and the internal randomness of $\mathcal{A}$. 

We start by defining a set of `bad' outputs for each fixed dataset $S$, i.e. outputs on which the probability mass of $\mathcal{A}(S)$ is substantially larger than that of the canonical distribution $Sim_D$. Formally, for each dataset $S$, let 
\[
B_{S} = \{y \in \mathcal{Y}: \Pr_\mathcal{A}[\mathcal{A}(S) = y] \geq 2^{2k} \Pr_{S' \sim D^m, \mathcal{A}}[\mathcal{A}(S') = y] \}.
\]

Next, we define a set of ordered pairs consisting of datasets and their corresponding `bad' outputs. Formally, let $$\Theta = \{(S,y): S \in \supp(D^m), y \in B_{S}\}.$$

Our goal will be to prove that with high probability over a draw of a dataset $S$, $\mathcal{A}(S)$ lands in the bad set $B_{S}$ with small probability. This can then be used to establish one-way perfect generalization.

With this in mind, consider the expression $\mathbb{E}_{S}[\Pr_{\mathcal{A}}[\mathcal{A}(S) \in B_{S}]]=\mathbb{E}_{S}[\Pr_{\mathcal{A}}[(S,\mathcal{A}(S)) \in \Theta]]$. By the law of total probability, this is equal to $\Pr_{S \sim D^m, \mathcal{A}}[(S,\mathcal{A}(S)) \in \Theta]$. Using the definition of max-information, we get that
\begin{align}\label{eq:max-inf-relation}
    \Pr_{S \sim D^m, \mathcal{A}}[(S,\mathcal{A}(S)) \in \Theta] \leq 2^{k}\Pr_{S,S' \sim D^m, \mathcal{A}}[(S,\mathcal{A}(S')) \in \Theta] + \beta.
\end{align} 
Now, analyzing the term $\Pr_{S,S' \sim D^m, \mathcal{A}}[(S,\mathcal{A}(S')) \in \Theta]$, we get 
\begin{align*}
    \Pr_{S,S' \sim D^m, \mathcal{A}}[(S,\mathcal{A}(S')) \in \Theta] & = 
    \sum_{T \in \supp(D^m)} \Pr[S=T] \Pr_{S' \sim D^m, \mathcal{A}}[\mathcal{A}(S') \in B_T] \\
    & \leq \sum_{T \in \supp(D^m)} \Pr[S=T] 2^{-2k} \Pr_{\mathcal{A}}[\mathcal{A}(S) \in B_T] \\
   & = 2^{-2k} \sum_{T \in \supp(D^m)} \Pr[S=T] \Pr_{\mathcal{A}, S \sim D^m}[A(S) \in B_{S} \mid S=T] \\
   & = 2^{-2k} \Pr_{S \sim D^m, \mathcal{A}}[(S,\mathcal{A}(S)) \in \Theta],
\end{align*}
where the first inequality is by the definition of $B_T$ and the last equality is by the law of total probability. Substituting the above in (\ref{eq:max-inf-relation}), we get that 
\begin{align*}
    \Pr_{S \sim D^m, \mathcal{A}}[(S,\mathcal{A}(S)) \in \Theta] \leq 2^{-k}\Pr_{S \sim D^m, \mathcal{A}}[(S,\mathcal{A}(S) \in \Theta] + \beta.
\end{align*} 
Rearranging, this gives us that
\begin{align}\label{eq:expec-bound}
    \mathbb{E}_{S}[\Pr_{\mathcal{A}}[\mathcal{A}(S) \in B_{S}]] = \Pr_{S \sim D^m, \mathcal{A}}[(S,\mathcal{A}(S)) \in \Theta] \leq \frac{\beta}{1-2^{-k}} = \hat{\beta}^2.
\end{align}
Finally, by Markov's inequality and the above equation, we can write the following.
\begin{align}\label{eq:high-prob-bound}
    \Pr_{S}[\Pr_{\mathcal{A}}[\mathcal{A}(S) \in B_{S}] > \hat{\beta}] \leq
    \frac{\mathbb{E}_{S}[\Pr_{\mathcal{A}}[\mathcal{A}(S) \in B_{S}]]}{\hat{\beta}} \leq \hat{\beta}.
\end{align}
This implies that with probability $1-\hat{\beta}$ over $S \sim D^m$, $\Pr_{\mathcal{A}}[\mathcal{A}(S) \not \in B_{S}] > 1-\hat{\beta}$. Finally, we can write with probability $1-\hat{\beta}$ over $S \sim D^m$, for every $O \subseteq \mathcal{Y}$,
\begin{align*}
\Pr_{\mathcal{A}}[\mathcal{A}(S) \in O]  & = \Pr_{\mathcal{A}}[\mathcal{A}(S) \in O \mid \mathcal{A}(S) \in B_{S}] \Pr[\mathcal{A}(S) \in B_{S}] \\
& +  \Pr_{\mathcal{A}}[\mathcal{A}(S) \in O , \mathcal{A}(S) \not\in B_{S}]  \\
& \leq 1 \cdot \hat{\beta} + \Pr_{\mathcal{A}}[\mathcal{A}(S) \in O , \mathcal{A}(S) \not\in B_{S}] \\
%& \leq \hat{\beta} + \sum_{y \in O \cap \overline{B_{S}}} \Pr(\mathcal{A}(S) = y) \\
%& \leq \hat{\beta} + 2^{2k}  \sum_{y \in O \cap \overline{B_{S}}} \Pr_{S' \sim D^m, \mathcal{A}}(\mathcal{A}(S') = y) \\
& \leq \hat{\beta} + 2^{2k} \Pr_{S' \sim D^m, \mathcal{A}}[\mathcal{A}(S') \in O , \mathcal{A}(S') \not\in B_{S'}] \\
& \leq \hat{\beta} + 2^{2k} \Pr_{S' \sim D^m, \mathcal{A}}[\mathcal{A}(S') \in O].
\end{align*}
%A similar argument (swapping the canonical distribution $Sim_P$ and the distribution of $A(X)$ for a fixed dataset in the above argument) can be used to prove that with probability $1-\hat{\beta}$ over $X \sim P^n$, for every $O \subseteq \mathcal{Y}$,
%\begin{align*}
%\Pr_{X' \sim P^n, A}(A(X') \in O) \leq \hat{\beta} + 2^{2k} \Pr_{A}(A(X) \in O)
%\end{align*}
This completes the proof.
\end{proof}
We note that the above proof sets the failure probability due to data sampling and that due to bad coins of the algorithm to be the same. Other tradeoffs between these can be obtained by using Markov's inequality with different parameters in Equation~\ref{eq:high-prob-bound}. We chose them to be equal to each other for simplicity of presentation and because that is the setting of interest in our applications.

Observe that combining Lemma~\ref{lem:max-info-bound} and Lemma~\ref{lem:max-info-PG} above gives the following connection between differential privacy and one-way perfect generalization.
\begin{corollary}
Fix $m \in \mathbb{N}$, $\eps \in (0,1/2]$ and $\delta \in [0,\eps/15)$. Let $\mathcal{A}: \mathcal{X}^m \to \mathcal{Y}$ be an $(\eps, \delta)$-DP algorithm. Then, for sufficiently large constants $c,C$, for all $t>0$, $\mathcal{A}$ is $(\delta', \eps', \delta')$-one-way perfectly generalizing, where $\eps' = Cm( \eps^2 + \sqrt{\frac{\delta}{\eps}}) + 6t \eps \sqrt{m}$, and $\delta' = \frac{\beta}{1-2^{-O(\eps')}}$, where $\beta = e^{-t^2 / 2} + cm\sqrt{\frac{\delta}{\eps}}$. 
\end{corollary}
%\sstext{Explain more about what the parameters above mean intuitively?}.
As an example of the kind of result this can give, we show what we'd get if we instantiated it with our parameters of interest (as in Corollary~\ref{cor:max-info-concrete}).
\begin{corollary}\label{cor:DPtoPG}
Fix $m \in \mathbb{N}$, sufficiently small $\rho \in (0,1]$. Let $\mathcal{A}: \mathcal{X}^m \to \mathcal{Y}$ be an $(\eps, \delta)$-DP algorithm, where $\eps = \frac{\rho}{\sqrt{8 m \log(1/\rho)}}$, $\delta \leq \frac{\eps \rho^6}{m^2}$. Then, $\mathcal{A}$ is $(O(\rho), O(\rho), O(\rho))$-one-way perfectly generalizing.
\end{corollary}
\begin{proof}
From Corollary~\ref{cor:max-info-concrete}, we get that $I^{\rho^3}_{\infty}(S;\mathcal{A}(S)) \leq O(\rho)$. Substituting $k = \rho$ and $\beta = \rho^3$ in Lemma~\ref{lem:max-info-PG}, we get that $\hat{\beta} = \sqrt{\frac{c\rho^3}{1-2^{-O(\rho)}}} \leq O(\sqrt{\rho^2}) = O(\rho)$ (where the first inequality is since $\frac{1}{1-2^{-O(\rho)}} = \frac{2^{O(\rho)}}{2^{O(\rho)}-1} \leq \frac{2}{2^{O(\rho)}-1} \leq \frac{C}{\rho}$ for some constant $C$, since $2^{c\rho} = e^{c \rho \ln 2 }$ and $e^x \geq 1+x$ for all real $x$). This gives us from Lemma~\ref{lem:max-info-PG} that $\mathcal{A}$ is $(O(\rho), O(\rho), O(\rho))$-one-way perfectly generalizing.
\end{proof}
\subsection{Perfect Generalization Implies Replicability}\label{sec:PG2rep}

%We will prove a stronger result. 
In this section will show that the class of one-way perfectly generalizing algorithms, which includes the special case of (two-way) perfectly generalizing algorithms, can be transformed to replicable algorithms.

Let $CS(Q,\mathcal{Y},r')$ represent a correlated sampling procedure over domain $\mathcal{Y}$ sampling from a distribution $Q$ over $\mathcal{Y}$ with public randomness $r'$. %\sstext{(put in preliminaries about correlated sampling)} 
(See Section~\ref{sec:correlated-sampling} for background on correlated sampling).  We now describe our transformation.

\begin{algorithm}[H]
    \caption{Transformation from one-way perfectly generalizing algorithm to replicable algorithm}
    \label{alg:reprodtrans}
    \hspace*{\algorithmicindent} \textbf{Input:} dataset $S = (\datafixed_1, \ldots, \datafixed_{n})$, description of one-way perfectly generalizing algorithm $\mathcal{A}: \mathcal{X}^m \to \mathcal{Y}$ \\
    \hspace*{\algorithmicindent} \textbf{Output:} $i\in \universe$
    \begin{algorithmic}[1] % The number tells where the line numbering should start
            \STATE Let $Q_{\Datafixed}$ represent the distribution of $\mathcal{A}(\Datafixed)$.
            \STATE Output $CS(Q_{\Datafixed},\mathcal{Y},r')$ where $r'$ is the random string drawn in the correlated sampling algorithm. %(corresponding to public randomness).
    \end{algorithmic}
\end{algorithm}

The key idea is that correlated sampling converts total variation distance into collision probability, which is the notion that is used in the definition of replicability.

We now use this to prove the main theorem of this section.
\begin{theorem}\label{PGtoRep}
Fix $m \in \mathbb{N}$ and $\beta, \eps, \delta \in (0,1)$. Let $A: \mathcal{X}^m \to \mathcal{Y}$ be a $(\beta, \eps, \delta)$-one-way perfectly generalizing algorithm with finite output space. Then, for any distribution $D$ over $\mathcal{X}$, if $\Datafixed \sim D^m$, Algorithm~\ref{alg:reprodtrans} when run on dataset $\Datafixed$ and with access to $\cal{A}$ is $4(\beta + 2\eps + \delta)$-replicable.
\end{theorem}
\begin{proof}
By Lemma~\ref{lem:PGeps0}, we have that $\mathcal{A}$ is also $(\beta, 0, 2\eps + \delta)$-perfectly generalizing. For any distribution $D$ over $\mathcal{X}$, let $Sim_D$ be the canonical distribution witnessing the perfect generalization property. Then, by the definition of $(0,2\eps+\delta)$-indistinguishability, we have that with probability at least $1-\beta$ over a draw of a random dataset $\Datafixed \sim D^m$, 
$$d_{TV}(\mathcal{A}(\Datafixed), Sim_D) \leq 2\eps + \delta.$$
From the guarantee of correlated sampling, we have that
$$\Pr_{r' \sim R}[CS(Q_{\Datafixed},\mathcal{Y},r') \neq CS(Sim_D, \mathcal{Y},r')] \leq 2d_{TV}(Q_{\Datafixed},Sim_D). $$
Using the bound on TV distance from perfect generalization, we get that with probability at least $1-2\beta$ over the draw of two datasets $\Datafixed_1,\Datafixed_2 \sim D^m$, we have that
$$\Pr_{r' \sim R}[CS(Q_{\Datafixed_1},\mathcal{Y},r') \neq CS(Sim_D, \mathcal{Y},r')] \leq 2(2\eps + \delta)$$ and $$\Pr_{r' \sim R}[CS(Q_{\Datafixed_2},\mathcal{Y},r') \neq CS(Sim_D, \mathcal{Y},r')] \leq 2(2\eps + \delta).$$ 
Consider the event $CS(Q_{\Datafixed_1},\mathcal{Y},r') \neq CS(Q_{\Datafixed_2},\mathcal{Y},r')$. It is clear that this implies that either $CS(Q_{\Datafixed_1},\mathcal{Y},r') \neq CS(Sim_D, \mathcal{Y},r')$ or $CS(Q_{\Datafixed_2},\mathcal{Y},r') \neq CS(Sim_D, \mathcal{Y},r')$. Hence, we can write that with probability $1-2\beta$ over draws of $\Datafixed_1$ and $\Datafixed_2$,
$$\Pr_{r' \sim R}[CS(Q_{\Datafixed_1},\mathcal{Y},r') \neq CS(Q_{\Datafixed_2}, \mathcal{Y},r')] \leq 4(2\eps + \delta).$$ 
Taking the expectation with respect to the draws of $\Datafixed_1$ and $\Datafixed_2$ gives us
$$\Pr_{r' \sim R, \Datafixed_1, \Datafixed_2 \sim D^m }[CS(Q_{\Datafixed_1},\mathcal{Y},r') \neq CS(Q_{\Datafixed_2}, \mathcal{Y},r')] \leq 4(2\eps + \delta + \beta)$$
which proves the result.
\end{proof}
Combining the above result and Corollary~\ref{cor:DPtoPG}, we get a transformation from approximate differentially private algorithms to replicable algorithms.
\begin{corollary}\label{cor:DPtoRep}
Fix $m \in \mathbb{N}$, sufficiently small $\rho \in (0,1)$. Let $\eps = \frac{\rho}{\sqrt{8 m \log(1/\rho)}}$, $\delta \leq \frac{\eps \rho^6}{m^2}$. Let $\mathcal{A}: \mathcal{X}^m \to \mathcal{Y}$ be an $(\eps, \delta)$-DP algorithm with finite output space. Fix a distribution $D$ over $\mathcal{X}$, and let $\Datafixed \sim D^m$. Then, Algorithm~\ref{alg:reprodtrans} run with inputs $\Datafixed$ and algorithm $\mathcal{A}$ is $O(\rho)$-replicable. Additionally, on any fixed dataset, the output distribution of Algorithm~\ref{alg:reprodtrans} is the same as that of 
$\mathcal{A}$.
\end{corollary}
\begin{proof}
From Corollary~\ref{cor:DPtoPG}, we have that Algorithm $\mathcal{A}$ is $(O(\rho), O(\rho), O(\rho))$-one-way perfectly generalizing. Then, applying Theorem~\ref{PGtoRep} proves that the transformation in Algorithm~\ref{alg:reprodtrans} gives a $O(\rho)$-replicable algorithm. Correlated sampling does not change the marginal distribution of the algorithm applied to a dataset and hence the second part of the corollary is proved.
\end{proof}

\subsection{Replicability Implies Perfect Generalization}\label{sec:rep2pg}
In this section, we show how to convert a replicable algorithm to a perfectly generalizing algorithm at a poly-logarithmic cost in $1/\delta$ (where $\delta$ is the additive perfect generalization parameter). %Note that this could be done by converting a reproducible algorithm into a differentially private one, and then converting the differentially private algorithm into a perfectly generalizing one. However, the connection between differential privacy and perfect generalization discussed in previous sections increases $\eps$ by a factor of $\sqrt{m}$ (where $m$ is the size of the dataset), which will result in a quadratic blow-up in sample complexity. The direct conversion that we show in this section will show that such a quadratic blow-up can be avoided. This stronger conversion will be crucial for us to prove statistical lower bounds against reproducible algorithms.

It's straightforward to show that $(\delta,\delta)$-replicability can be used to obtain $(O(\delta),0,O(\delta))$-perfect generalization by translating from collision probability to total variation distance. However, since we typically want $\delta$ to be very small (often inverse polynomial in the number of samples $m)$, obtaining such small parameters starting from, say, $0.1$-replicability comes at a significant cost. This is because amplifying $0.1$-replicability to $(\delta, \delta)$-replicability incurs a multiplicative sample complexity overhead of $O(1/\delta^2)$, which is tight by known lower bounds for replicability~\cite[Theorem 7.1]{ImpLPS22}, and prohibitively large for many applications. For example, our lower bounds showing tasks where replicability has quadratically higher sample cost than differential privacy (see Section~\ref{sec:stat-sep}) follow from proving such lower bounds on perfectly generalizing algorithms with $\delta$ polynomially small in the dataset size, and then applying our conversion from replicability to perfect generalization. If such a conversion required $1/\delta^2$ samples, then this would not give us any non-trivial lower bounds on the sample cost associated with replicably solving these problems.

%\mb{We should elaborate on why, here. Since it's not clear that large delta dramatically affects the semantics of PG itself. Pointing to the lower bounds in Section 5 is one example. Anything else?}
However, this idea still leaves hope, because it achieves $\eps = 0$. Hence, by settling for larger $\eps$, we hope to avoid this problem.

Our approach is inspired by a natural attempt to amplify weak replicability parameters into strong parameters. Suppose we wish to turn a $(0.01, 0.01)$-replicable algorithm $A$ into a $(0.01, \delta)$-replicable one. We know that with probability at least $0.99$ over the choice of the randomness $r$ for $A$, there is a canonical output $z$ such that $A(S; r) = z$ with high probability over the sample $S$. Consider running $A$ using $k = O(\log(1/\delta))$ independent sequences of coin tosses, $r_1, \dots, r_k$, then with probability $1-\delta$, at least one of these sequences will have such a canonical output. Moreover, when such a canonical output exists, we can identify it by running $A(\cdot; r_j)$ on many independent \emph{samples} $S$ and choosing the plurality outcome if it appears enough times. Unfortunately, there is an obstacle here to directly designing a replicable algorithm. The problem is that there may be many good sequences of coin tosses, each with their own canonical outputs, and it is unclear how to replicably identify a single one.

%The idea is to run a weakly replicable algorithm on many independent choices of the random coins, and for each random coin sequence, attempt to identify a frequent output for that coin sequence. This can in turn be done by running the algorithm on many independent data samples with that coin sequence used as the internal randomness. Now, consider the plurality output for each coin sequence $j$ (call it the ``canonical output'' for $j$) and suppose it occurs $k_j$ times. Our goal  is to identify a canonical output that occurs many times, i.e., the corresponding $k_j$ is large. The idea is that such an output is very likely to be the frequent output for some good coin sequence, and since we've repeated the weakly replicable algorithm for many coin sequences, at least one of our runs is likely to be good. Unfortunately there is an obstacle here in directly designing a strongly replicable algorithm, since it's unclear how to replicably identify a single canonical output when there could be many possibilities to choose from.

By relaxing our goal to achieving perfect generalization instead of replicability, we can instead use the exponential mechanism to \emph{sample} from the set of plurality outcomes. We define the score of the plurality output $c_j$ for coin $r_j$ to be the number of datasets $S$ on which $A(S; r_j) = c_j$, and sample such a $c_j$ with probability proportional to exponential in its score. We are able to show that the resulting algorithm is $(\delta, \eps, \delta)$-perfectly generalizing with $\eps > 0$, but there are several technical nuances that make our analysis not quite straightforward from the standard guarantees of the exponential mechanism. For instance, we need to deal with the fact that the sets of plurality outputs could differ when our algorithm is run on two i.i.d. datasets drawn from the same distribution. %We are able to prove that this idea gives a perfectly generalizing algorithm - there are a number of technical nuances that crop up when attempting to prove this that we will leave to the formal argument (e.g., the sets of plurality outputs could differ when this algorithm is run on two i.i.d. datasets drawn from the same distribution). 
Another interesting feature of this proof is that unlike standard uses of the exponential mechanism to obtain differential privacy or perfect generalization, we need to invoke the \textit{accuracy} of the exponential mechanism in our proof of perfect generalization.

\begin{algorithm}[H]
    \caption{Transformation from replicable algorithm $\mathcal{A}$ to perfectly generalizing algorithm $\mathcal{A}'$}
    \label{alg:perfgentrans}
    \hspace*{\algorithmicindent} \textbf{Input:} Sample access to distribution $D$, description of $(0.01, 0.01)$-replicable algorithm $\mathcal{A}: \mathcal{X}^* \to \mathcal{Y}$, sample complexity parameter $m$, perfect generalization parameters $\eps, \delta, \beta$\\
    \hspace*{\algorithmicindent} \textbf{Output:} $y \in \mathcal{Y}$
    \begin{algorithmic}[1] % The number tells where the line numbering should start
            \STATE Let $k = O(\log (1/\delta))$, and $t = O\left(\frac{\log^4 (1/\beta) \log (1/\eps)}{\eps^2}\right)$.
            \STATE Draw uniformly random coins $r_1,r_2,\dots,r_k$ for algorithm $\mathcal{A}$. \label{step:randcoinpg}
            \STATE Draw $k$ sets $\mathbf{S_i}$, each of $t$ samples $S_{i,j} \sim D^m$.
            \FOR{all $j \in [k]$} \label{step:loopPG}
            \FOR{all $i \in [t]$}
            \STATE Run $\mathcal{A}$ with coins $r_j$ and sample $S_{i,j}$ to get output $z_{i,j}$.
            \ENDFOR
            \STATE Let $c_j = \arg \max_{z \in \mathcal{Y}} \sum_{i=1}^t \indicator[z_{i,j}=z]$, and let $\score\left((j, c_j),(\mathbf{S}_1, \mathbf{S}_2, \dots, \mathbf{S}_k) \right)= \sum_{i=1}^t \indicator[z_{i,j}=c_j]$.
            \ENDFOR
            \STATE Let $C = \{(1, c_1),\dots,(k, c_k)\}$. Run the exponential mechanism on the set $C$ with the score function $\score(.,.)$, sensitivity parameter $4\sqrt{t \log(8kt/\beta)}$, and privacy parameter $\eps$ to get value $(j^*, c_{j^*})$. \label{step:expmechPG} 
            \RETURN output $c_{j^*}$ of the previous step. \label{step:PGop}
    \end{algorithmic}
\end{algorithm}
We prove that the above algorithm is sample perfectly generalizing. Note that this can be converted to a perfectly generalizing algorithm with asymptotically the same parameters (for both perfect generalization and accuracy) by setting the $\delta$ parameter to be $\delta^2$ instead and invoking Lemma~\ref{lem:samplePGtoregPG}.
\begin{theorem} \label{thm:reprodtoPG}
Fix sufficiently small $\delta, \gamma > 0$ and $0<\eps \leq 1$. Every $(0.01,0.01)$-replicable algorithm $\mathcal{A}$ with $m$ samples that succeeds on a statistical task with probability at least $1-\gamma^2$ can be converted to a $(2\delta,\eps,2\delta)$-sample perfectly generalizing algorithm $\mathcal{A}'$ taking $O\left(\frac{m \log (1/\eps)}{\eps^2} \poly\log(1/\delta)\right)$ samples, that succeeds on the statistical task with probability at least $1-O(\delta)-\gamma \log (1/\delta)$.
\end{theorem}
\begin{proof}
Fix a distribution $D$ over the input set $\mathcal{X}$. Our proof is constructive; the corresponding algorithm is given in Algorithm~\ref{alg:perfgentrans} ($\mathcal{A}'$), and we feed it with the following inputs: a description of algorithm $\mathcal{A}$, sample complexity parameters $m$, and perfect generalization parameters $(\delta, \eps,\delta)$. (We also give it sample access to distribution $D$). We will start by proving sample perfect generalization of Algorithm~\ref{alg:perfgentrans}.
\begin{claim}
Algorithm $\mathcal{A}'$ (represented in Algorithm~\ref{alg:perfgentrans}) with the input parameters specified in the previous paragraph is $(2\delta,\eps,2\delta)$-sample perfectly generalizing.
\end{claim}
\begin{proof}
Consider two samples $S$ and $S'$ drawn independently from $D^{mkt}$. We consider Algorithm~\ref{alg:perfgentrans} run on both of these samples and argue that their output distributions are close in the sense required by sample perfect generalization. 
\paragraph{Step 1: At least one coin sequence is good w.h.p.} We say that a choice of the random coin tosses of Algorithm $\mathcal{A}$ is ``good'' if it has a $0.99$-canonical output and call it ``bad'' otherwise. Then by the two parameter definition of replicability, a random coin sequence is ``bad'' with probability at most $0.01$. Hence, the probability that all $k$ coins sequences drawn in Step~\ref{step:randcoinpg} of Algorithm~\ref{alg:perfgentrans} are bad is at most $(0.01)^k \leq \delta^2$ for $k = O(\log (1/\delta))$. %Hence, the probability that all $k$ coins are bad is at most $\delta^2$.
Let $E_{coin}$ represent the event that there is at least one good coin. We will now condition on $E_{coin}$ occurring; fix any set of coins $r_1,\dots,r_k$ that has non-zero probability of occurring under this conditioning. We will first consider $\mathcal{A}$ run on the two independent datasets $S$ and $S'$ with the same random coins fixed above.

\paragraph{Step 2: Empirical output frequencies are close on two independent datasets.} We define stage $j$ of Algorithm~\ref{alg:perfgentrans} as the process involved in generating $c_j$ (i.e., one iteration of the outer loop in Step~\ref{step:loopPG}). We now use uniform convergence to argue that with high probability over the samples, the empirical frequencies of the outputs of all stages, i.e. all values of $Score\left((j, c_j),(\mathbf{S}_1, \mathbf{S}_2, \dots, \mathbf{S}_k) \right)$, are close to their expectation. 

Start by defining $H$ to be the function class consisting of point functions, i.e., functions of the form $h_{\ell}(x) = 1$ if $x={\ell}$, and $h_{\ell}(x) = 0$ otherwise, for every $\ell \in \mathcal{Y}$. %(where $\ell$ belongs to any discrete set. In our setting $l$ will be the set $\mathcal{Y}$). 
It is easy to prove that the VC dimension of $H$ is equal to $1$. Let $Q_j$ be the distribution of the output of the replicable algorithm $\mathcal{A}$ when run with coin $r_j$ on a random sample.

Then, using uniform convergence (Theorem~\ref{thm:unifconv}), we get for every fixed $j$ and for every $\gamma > 0$, that
$$\Pr_{z_{1,j},\dots,z_{t,j} \sim Q_j}\left[\sup_{h_{\ell} \in H}\left|\frac{1}{t}\sum_{i=1}^t \mathbbm{1}[h_{\ell}(z_{i,j}) = 1] - \Pr_{z \sim Q_j}[h_{\ell}(z)=1] \right| \geq \gamma \right] \leq 8t e^{-\gamma^2 t / 8}.                                      $$

Observe that $\frac{1}{t}\sum_{i=1}^t \mathbbm{1}[h_{\ell}(z_{i,j}) = 1] = \frac{1}{t}\sum_{i=1}^t \mathbbm{1}[z_{i,j} = \ell]$. Similarly, $\Pr_{z \sim Q}[h_{\ell}(z)=1] = \Pr_{z \sim Q_j}[z=\ell]$. Hence, we get that 
$$\Pr_{z_{1,j},\dots,z_{t,j} \sim Q_j}\left[\sup_{\ell \in \mathcal{Y}}\left|\frac{1}{t}\sum_{i=1}^t \mathbbm{1}[z_{i,j} = \ell] -\Pr_{z \sim Q_j}[z=\ell] \right| \geq \gamma \right] \leq 8t e^{-\gamma^2 t / 8}.                                      $$

Setting $\gamma = 2\sqrt{\frac{\log(8kt/\delta)}{t}}$, we get that
$$\Pr_{z_{1,j},\dots,z_{t,j} \sim Q_j}\left[\sup_{\ell \in \mathcal{Y}}\left|\sum_{i=1}^t \mathbbm{1}[z_{i,j} = \ell] -t\Pr_{z \sim Q_j}[z=\ell] \right| \geq 2\sqrt{t \log(8kt/\delta)} \right] \leq \frac{\delta^2}{2k}. $$

Using a union bound over all $k$ stages of the algorithm, this guarantees us that the empirical frequencies (and in particular, the values $\score\left((j, c_j),(\mathbf{S}_1, \mathbf{S}_2, \dots, \mathbf{S}_k) \right)$ are all within $2\sqrt{t \log(8kt/\delta)}$ of their expectations with probability at least $1-\delta^2/2$. 

Note that since we conditioned on a fixed random coin sequence, all the randomness in $z_{i,j}$ comes from the data sample. Hence, if we consider another sample $S' = (\mathbf{S}'_1, \mathbf{S}'_2, \dots, \mathbf{S}'_k)$ drawn i.i.d .from $D^{mkt}$, we have that with probability at least $1-\delta^2$ over the draws of $S$ and $S'$ that for all $j \in [k]$ and all $\ell \in \mathcal{Y}$ that $\score\left((j, \ell),(\mathbf{S}_1, \mathbf{S}_2, \dots, \mathbf{S}_k) \right)$ and $\score\left((j, \ell),(\mathbf{S}'_1, \mathbf{S}'_2, \dots, \mathbf{S}'_k) \right)$ %(note that we consider the same output $Can_j$ above, i.e. the canonical output of the $j^{th}$ stage when Algorithm~\ref{alg:perfgentrans} is run on $S$)
are  both within $2\sqrt{t \log(8kt/\delta)}$ of their expectations, and are hence within $4\sqrt{t \log(8kt/\delta)}$ of each other. We call this event $E_{sample}$, and fix any sample pairs $(S, S')$ that occur with non-zero probability conditioned on this event. This allows us to argue that with probability at least $1-\delta^2$, for all $j \in [k]$,
\begin{align}
\label{eq:PGcanopclose}
|\score\left((j, c_j),(\mathbf{S}_1, \mathbf{S}_2, \dots, \mathbf{S}_k) \right) - \score\left((j, c'_j),(\mathbf{S}'_1, \mathbf{S}'_2, \dots, \mathbf{S}'_k) \right)| \leq 4\sqrt{t \log(8kt/\delta)}.
\end{align}
This follows directly from the above argument if $c_j = c'_j$, but if they are not equal, it also holds since otherwise either $c_j$ or $c'_j$ would not be a plurality output in stage $j$ of the corresponding runs (since there would be an output that occurs more times in stage $j$). This is because uniform convergence guarantees us that if $c_j$ occurs $a$ times in stage $j$ when the algorithm is run on dataset $S$, then $c_j$ occurs atleast $a-4\sqrt{t \log(8kt/\delta)}$ times in stage $j$ when the algorithm is run on dataset $S'$. Hence, if $c'_j$ occurs less than  $a - 4\sqrt{t \log(8kt/\delta)}$ in stage $j$, then we'd get that $c_j$ would be the plurality output of stage $j$ in the run on dataset $S'$ and not $c'_j$, which is a contradiction. This shows that $\score\left((j, c_j),(\mathbf{S}_1, \mathbf{S}_2, \dots, \mathbf{S}_k) \right) - \score\left((j, c'_j),(\mathbf{S}'_1, \mathbf{S}'_2, \dots, \mathbf{S}'_k) \right) \leq 4\sqrt{t \log(8kt/\delta)}$; the other direction can be proved similarly.
% \mb{The second case needs more explanation, since $c_j$ is the observed plurality output, not the canonical output that depends only on the distribution (and may not even exist if coin $j$ is bad).} \sstext{Even if the coin is bad uniform convergence means the plurality output $c$ (occurring $a$ times) of such a coin occurs at least $a-4\sqrt{n}$ in the other case right? And so since $c'$ is the plurality output it must occur atleast this many times, and same argument holds the other way.} \mb{Yes, the statement is correct, but I think this bit of the analysis needs to be explicit}

\paragraph{Step 3: Arguing that there is at least one canonical output $c_j$ with high score.} Conditioned on $E_{coin}$, we know that the run of Algorithm~\ref{alg:perfgentrans} on $S$ has at least one coin sequence with a $0.99$-canonical output $z$. Suppose $r_j$ is such a coin sequence. From the settings of $k$ and $t$, we get that $2\sqrt{t \log(8kt/\delta)} \leq 0.09t$. Hence, conditioned further on $E_{sample}$, we know that this canonical output $z$ is equal to the plurality output $c_j$ in stage $j$,  and that $\score\left((j, c_j),(\mathbf{S}_1, \mathbf{S}_2, \dots, \mathbf{S}_k) \right)$ is at least $0.9t$. Hence, there exists an candidate $(j, c_j)$ with score at least $0.9t$. 

\paragraph{Step 4: Arguing that probable outputs $(j^*, c_{j^*})$ are in both output sets $C$ and $C'$.} By the accuracy guarantee of the exponential mechanism (Lemma~\ref{lem:expmech}), we have that
$$\Pr\left[\max_{j \in [k]} \score\left((j, c_j),(\mathbf{S}_1, \mathbf{S}_2, \dots, \mathbf{S}_k) \right)-  \score\left((j^*, c_{j^*}),(\mathbf{S}_1, \mathbf{S}_2, \dots, \mathbf{S}_k) \right) \geq 2\Delta \frac{\ln k + k}{\eps}\right] \leq e^{-k} = \delta,$$
where $\Delta = 4\sqrt{t \log(8kt/\delta)}$. Hence, for the settings of $t$ and $k$, we get that $2\Delta \frac{\ln k + k}{\eps} = O(\frac{k\sqrt{t \log(8kt/\delta)}}{\eps}) = o(t)$. Hence, we have that
$$\Pr\left[  \score\left((j^*, c_{j^*}),(\mathbf{S}_1, \mathbf{S}_2, \dots, \mathbf{S}_k) \right) \geq 0.9t - 2\Delta \frac{\ln k + k}{\eps}\right] \leq e^{-k} = \delta,$$
which implies that for sufficiently small $\delta$,
$$\Pr\left[\score\left((j^*, c_{j^*}),(\mathbf{S}_1, \mathbf{S}_2, \dots, \mathbf{S}_k) \right) \geq 0.8t \right] \leq e^{-k} = \delta.$$
Let's consider any such $c_{j^*}$. By the conditioning on $E_{sample}$, we have that $c_{j^*}$ occurs more than $0.8t -  4\sqrt{t \log (kt/\delta)} > 0.5t$ times in the output set of stage $j^*$ when Algorithm~\ref{alg:perfgentrans} is run on the sample $S'$. Hence, $c'_{j^*}$ is also equal to $c_{j^*}$. 
\paragraph{Step 5: Proving that $\mathcal{A}'(S,r) \approx _{\eps, \delta} \mathcal{A}'(S',r)$ w.h.p.} 
We exploit the fact that two random variables $C$ and $D$ are $(\eps, \delta)$-indistinguishable if w.p. $\geq 1-\delta$ over a draw $o$ from the distribution of $C$, we have $e^{-\eps}\Pr[D=o] \leq \Pr[C=o] \leq e^{\eps} \Pr[D=o]$, and vice versa for a draw from the distribution of $D$~ \cite[Lemma 3.3, Part 1]{KasiviswanathanS14}.
 
We proved in Step $4$ that fixing any coins and sample pairs that have non-zero probability of occurring conditioned on $E_{coin}$ and $E_{sample}$, with probability at least $1-\delta$ from a draw of $A'(S,r)$ (where the randomness is only that of the exponential mechanism), the output $(j^*, c_{j^*})$ occurs in both the sets $C$ and $C'$. For all such outputs, our idea is to use the differential privacy analysis of the exponential mechanism. 

A technical obstacle we need to surmount is that the output sets $C$ and $C'$ might be different, and so the normalizing factors used in the exponential mechanism will vary accordingly. We deal with this by invoking Inequality~\ref{eq:PGcanopclose}, which points out that even though the output sets are different, the scores $\score\left((j, c_j),(\mathbf{S}_1, \mathbf{S}_2, \dots, \mathbf{S}_k) \right)$ and $\score\left((j, c'_{j}),(\mathbf{S}'_1, \mathbf{S}'_2, \dots, \mathbf{S}'_k) \right)$ can differ by at most the sensitivity specified in Step~\ref{step:expmechPG} where the exponential mechanism is invoked.

%The other obstacle is that $Can$ and $Can'$ are multisets, and so can have repeats, which could affect the analysis (for e.g. some outputs could repeat in $Can$ but not repeat in $Can'$, so their probabilities of being output would add up in the former case but not in the latter). We argue that we can assume that they are sets of size $k$ without loss of generality. The idea is the following: consider a thought experiment where instead of defining $Can = \{Can_1,\dots,Can_j\}$, we would instead define $Can$ as $\{(Can_1,1), (Can_2, 2), \dots, (Can_k, k)\}$. Then, the exponential mechanism run on this set $Can$ of size $k$ would result in some output $(Can_j, j)$. Now, we can strip off the label $j$ and simply output $Can_j$. Observe that the distribution over outputs produced in this thought experiment is identical to the distribution of outputs of Algorithm~\ref{alg:perfgentrans}. Hence, if we prove that the algorithm prior to stripping the labels is sample perfectly generalizing, then the stripping of the labels is simply post-processing. Since sample perfect generalization is robust to post-processing, this would guarantee that Algorithm~\ref{alg:perfgentrans} is sample perfectly generalizing with the same parameters.

Hence, exactly mimicking the differential privacy analysis of the exponential mechanism (see e.g., \cite{DworkR14}, Theorem 3.10) conditioned on $E_{coin}$ and $E_{sample}$, with probability at least $1-\delta$ from a draw  $(j^*, c'_{j^*})$ of $\mathcal{A}'(S,r)$, we get that 
$$ e^{-\eps} \Pr[ \mathcal{A}'(S,r) = (j^*, c_{j^*})]  \leq \Pr[ \mathcal{A}'(S',r) = (j^*, c'_{j^*})] \leq e^{\eps} \Pr( \mathcal{A}'(S,r) = (j^*, c_{j^*})]$$
and, moreover, $c'_{j^*} = c_{j^*}$. By symmetry (since $S$ and $S'$ are both independent samples from the distribution with the same properties), conditioned on $E_{coin}$ and $E_{sample}$, we  get that with probability at least $1-\delta$ from a draw $(j^*, c_{j^*})$ of $\mathcal{A}'(S',r)$,
$$ e^{-\eps} \Pr[ \mathcal{A}'(S,r) = (j^*, c_{j^*})] \leq \Pr[ \mathcal{A}'(S',r) = (j^*, c_{j^*})] \leq e^{\eps} \Pr[ \mathcal{A}'(S,r) = (j^*, c_{j^*})].$$
Hence, we have proved that conditioned on any fixed coins and sample pairs with non-zero probability of occurring conditioned on $E_{coin}$ and $E_{sample}$, we have $\mathcal{A}'(S,r) \approx _{\eps, \delta} \mathcal{A}'(S',r)$. Now, using the law of total probability, we get that
\begin{align}
\label{eq:highprobPG}
\Pr_{r_1,\dots,r_k}\left[ \Pr_{S, S' \sim D^{mkt}}\Big[ \mathcal{A}'(S;r_1,\dots,r_k) \approx_{\eps, \delta} \mathcal{A}'(S';r_1,\dots,r_k) \Big] \geq 1 - \delta^2 \right] \geq 1-\delta^2.
\end{align}
%\mb{Don't the previous calculations only guarantee that the inner probability is $1-\delta$?}
\paragraph{Step 6: Switch quantifiers to get sample perfect generalization:}
Now, we switch the quantifiers in equation~\ref{eq:highprobPG}. 
\begin{align*}
        & \Pr_{S, S' \sim D^{mkt}, r_1,\dots,r_k} \Big[ \mathcal{A}'(S;r_1,\dots,r_k)  \approx_{\eps, \delta }\mathcal{A}'(S';r_1,\dots,r_k)  \Big] \geq 1-2\delta^2  \\
      \implies &\mathbb{E}_{S, S' \sim D^{mkt}} \left[ \Pr_{r_1,\dots,r_k} \Big[ \mathcal{A}'(S;r_1,\dots,r_k)  \approx_{\eps, \delta } \mathcal{A}'(S';r_1,\dots,r_k) \Big]\right] \ge 1-2\delta^2 \\
      \implies & \Pr_{S, S' \sim D^{mkt}} \left[ \Pr_{r_1,\dots,r_k} \Big[ \mathcal{A}'(S;r_1,\dots,r_k)  \approx_{\eps, \delta }\mathcal{A}'(S';r_1,\dots,r_k) \Big] \geq d \right] \ge \frac{1-2\delta^2-d}{1-d} .
\end{align*}
Here, the last inequality holds by the reverse Markov inequality. Setting $d = 1-\delta$, we get that
\begin{align*}
     \Pr_{S, S' \sim D^{mkt}} \left[ \Pr_{r_1,\dots,r_k} \Big[ \mathcal{A}'(S;r_1,\dots,r_k)  \approx_{\eps, \delta }\mathcal{A}'(S';r_1,\dots,r_k) \Big] \ge 1-\delta \right] \ge 1- 2\delta.
\end{align*}
Now, using the fact that if $X \approx_{\eps, \delta} Y$ and $M \approx_{\eps, \delta} N$, then $\alpha X + (1-\alpha) M \approx_{\eps, \delta} \alpha Y + (1-\alpha)N$ for every $\alpha \in [0, 1]$ (i.e., $(\eps, \delta)$-indistinguishability is preserved under convex combinations), we get that 
$$ \Pr_{S, S' \sim D^{mkt}} \left[ \mathcal{A}'(S)  \approx_{\eps, 2\delta } \mathcal{A}'(S') \right] \ge 1- 2\delta.$$ 
This proves that $\mathcal{A}'$ with the specified inputs is $(2\delta, \eps, 2\delta)$-sample perfectly generalizing, as required. Next, we deal with accuracy.
\end{proof}
\begin{claim}
If Algorithm $\mathcal{A}$ succeeds at a statistical task with probability at least $1-\gamma^2$, Algorithm $\mathcal{A}'$ succeeds at the same statistical task with probability at least $1-O(\delta)-\gamma \log (1/\delta)$.
\end{claim}
\begin{proof}
Recall the definition of success for a  statistical task. The statistical task is defined by a set of distribution, set pairs. For every distribution $D$, there is an associated good set of outputs $O_D$. An algorithm succeeds at this task with probability at least $1-\gamma$ if it outputs a member of this good set with at least that probability (taken over random samples from $D$ and any internal coins of the algorithm).

If $\mathcal{A}$ succeeds at the task with probability at least $1-\gamma^2$, by using reverse Markov's inequality as in Step 6 of the previous proof, we have that
\begin{align}
\label{eq:accuracy}
       \Pr_{r} \left[ \Pr_{S \sim D^{m}}  \Big[ \mathcal{A}(S;r) \in O_D  \Big] \geq 1-\gamma \right] \ge 1-\gamma
\end{align}

%\mb%{Equation 8 isn't right. That's %alrady a statement about $A'$, but you want a statement about the original algorithm $A$ on a single run $r_j$ when defining accuracy}
We say a coin sequence $r_j$ is ``accurate'' if the inner inequality under the probability is satisfied. %\mb{Accuracy is a property of one coin run, not the whole collection, right?} 
Recall that we call a coin sequence $r_j$ ``good'' if it has a $0.99$-canonical output. By the analysis in Step $1$ of the previous proof, we have that with probability at least $1-\delta^2$, there is a good coin sequence among the runs $r_1,\dots, r_k$. Now, the probability that all $k$ coin sequences are ``accurate'' is equal to $(1-\gamma)^{\log (1/\delta)} \geq 1- \gamma \log (1/\delta)$. (This follows from Bernoulli's inequality $(1+a)^k \geq 1+ak$ for all $a \geq -1$ and non-negative integers $k$.) Hence, by a union bound, the probability that the set of runs $r_1, \dots, r_k$ both contains a good coin sequence and that all the coin sequences in the set are accurate is at least $1-\gamma\log (1/\delta)-\delta^2$. Call this event $E_{coin-acc}$ and condition on it. Additionally, condition on $E_{sample}$ as defined in Step $2$ of the previous proof. Then, by the analysis in Step 4 of the previous proof, we have that  $$\Pr\left[\score\left((j^*, c_{j^*}),(\mathbf{S}_1, \mathbf{S}_2, \dots, \mathbf{S}_k) \right) \geq 0.8t \right] \leq e^{-k} = \delta,$$
which implies that the exponential mechanism outputs a plurality output that occurs at least $0.8t$ times in its stage with probability at least $1-\delta$. Since we have conditioned on $E_{sample}$, we have that empirical frequencies are close to their expected values, and hence the exponential mechanism outputs the canonical output of a coin sequence that is at least $0.25$-good with probability at least $1-\delta$. Using the law of total probability to remove the conditioning on $E_{sample}$, we get that the exponential mechanism outputs the canonical output of a coin sequence that is at least $0.25$-good with probability at least $1-\delta-\delta^2$ (since event $E_{sample}$ happens with probability at least $1-\delta^2$). Note that since all the drawn coin sequences are accurate, we get that the canonical output for every such sequence is in the good set $O_D$ (otherwise, the inequality inside the outer probability in equation~\ref{eq:accuracy} would not be satisfied). Hence, conditioned on $E_{coin-acc}$, we have that $\mathcal{A}'$ outputs an element of the good set with probability at least $1-\delta-\delta^2$. Using the law of total probability, we then get that $\mathcal{A}'$ outputs an element of the good set with probability at least $1-\delta-2\delta^2 - \gamma \log( 1/\delta) = 1-O(\delta) - \gamma \log (1/\delta)$.
\end{proof}
Hence, combining the two claims on perfect generalization and accuracy, we complete the proof of the theorem.
\end{proof}

%===============================
%===============================
%===============================

%\newpage
\section{Separating Stability: Computational Barriers}
\label{sec:separating-stability-computationally}

\newcommand{\prandenc}{\Pi_{\textit{RandEnc}}}
\newcommand{\cspace}{\mathcal{C}}
\newcommand{\dprandenc}{\mathtt{DPRandEnc}}
\newcommand{\cnew}{c_{\mathtt{new}}}
\newcommand{\cout}{c_{\mathtt{out}}}
\newcommand{\maj}{\texttt{Maj}}

In this section, we show that standard cryptographic assumptions imply there cannot exist computationally efficient transformations from differentially private algorithms to replicable ones. 
Moreover, such cryptographic assumptions are necessary: if one-way functions do not exist, there exists an efficient algorithm for correlated sampling (and therefore for converting DP to replicability as well via \Cref{cor:DPtoRep}).

In Section~\ref{ssec:dp-rep-separation}, we define $\prandenc$, a statistical promise problem. 
\chris{might be nice to say whether in $\prandenc$ it is valid to output one of the samples you saw. Makes it more clear that the problem is easy in general but hard replicably.}
Given a public key $\pk$ and a dataset of ciphertexts encrypting the same bit $b$ under $\pk$, a solution to $\prandenc$ is any encryption of $b$ under $\pk$. 
In Section \ref{ssec:dp-randenc-algo}, we give a simple algorithm $\dprandenc$ solving $\prandenc$. $\dprandenc$ is $(\eps, \delta)$-differentially private and runs in polynomial time. 
In Section \ref{ssec:no-rep-randenc-algo}, we show that the existence of an efficient replicable algorithm for $\prandenc$ would violate the security guarantee of the encryption scheme. Thus, assuming randomizable encryption schemes exist, there is no efficient transformation from DP algorithms to replicable algorithms for $\prandenc$. 
$\prandenc$ can be instantiated with any PKE satisfying the requirements of \Cref{def:rand-enc}, but to demonstrate that these requirements are not unreasonable, in Section \ref{ssec:gm} we show that they are satisfied by the Goldwasser-Micali public-key encryption scheme \cite{STOC:GolMic82}. 
Therefore the hardness of quadratic residuosity is sufficient to show hardness for the transformation from differential privacy to replicability.

In \Cref{sec:owfi}, we give an algorithm for correlated sampling that is efficient so long as no one-way functions exist. This algorithm can in turn be used in \Cref{alg:reprodtrans}, to implement the correlated sampling step of the transformation from a one-way perfectly generalizing algorithm to a replicable one, giving an efficient transformation. 

\subsection{Cryptographic Hardness of Replicability}
\label{ssec:dp-rep-separation}

We define a promise problem $\prandenc$ (Definition~\ref{def:ctxt-id-problem}) for a public-key encryption scheme. $\prandenc$ is parameterized by a public-key $\pk$ for a public-key encryption scheme $\escheme$ with message space $\{0,1\}$ and ciphertext space $\cspace$. An instance of $\prandenc$ consists of a sample of $m$ elements $c_i$, drawn i.i.d.\ from an unknown distribution $D$ over $\cspace$.
Promised that either 
\begin{enumerate}
    \item $D$ is supported entirely on encryptions of $0$ under $\pk$ or
    \item $D$ is supported entirely on encryptions of $1$ under $\pk$,
\end{enumerate}
Problem $\prandenc$ asks the algorithm to output an encryption of $0$ under $\pk$ in the first case and an encryption of $1$ under $\pk$ in the second case. 

We show that if the public-key encryption scheme supports a strong form of rerandomization, Problem $\prandenc$ can be efficiently solved with a differentially private algorithm. At the same time, $\prandenc$ cannot be efficiently solved using a replicable algorithm, assuming the security of the underlying encryption scheme. 
Thus, in this setting, there cannot be an efficient black-box reduction from DP algorithms to replicable algorithms.

For our construction, we use a standard definition for public-key encryption.
\begin{definition}[Public-Key Encryption Scheme]
\label{def:encryption-scheme}
 Let $\lambda \in \N$ be a security parameter and $\escheme = (\KeyGen, \Enc, \Dec)$ be a tuple of algorithms running in time $\poly(\lambda)$, with $\KeyGen: 1^{*} \rightarrow \mathcal{K}_p \times \mathcal{K}_s$, $\Enc: \mathcal{K}_p \times \{0,1\} \rightarrow \mathcal{C}$, and $\Dec: \mathcal{K}_s \times \mathcal{C} \rightarrow \{0,1\} \cup \bot$. We say $\escheme$ is a \emph{public-key encryption scheme} if it has the following properties.
\begin{itemize}
\item Correctness: Let $(\sk, \pk) \gets \KeyGen(\lambda)$ and $c \gets \Enc(\pk, b)$ for $b \in \{0,1\}$. Then $\Dec(\sk, c) = b$.   

\item Security: There exists a negligible function $\varepsilon(\lambda)$, such that for all adversaries $\mathcal{A}$ running in time $\poly(\lambda)$, letting $(\sk, \pk) \gets \KeyGen(\lambda)$ we have
$$\left|\Pr[\mathcal{A}(\pk, c) = 1 \mid c \gets \Enc(\pk, 1)] - \Pr[\mathcal{A}(\pk, c) = 1 \mid c \gets \Enc(\pk, 0)] \right| < \varepsilon(\lambda) .$$

\end{itemize}
\end{definition}

We also require that a public-key encryption scheme allows for efficient, publicly computable ciphertext verification and rerandomization procedures.  
\begin{definition}[Randomizeable Encryption Scheme]\label{def:rand-enc}
Let $\escheme = (\KeyGen, \Enc, \Dec)$ be a public-key encryption scheme. We call $\escheme$ a \emph{randomizeable encryption scheme} if it supports the following additional procedures. 
\begin{itemize}
	\item \textbf{(Perfect) Verification of Ciphertexts:}
		There exists a deterministic polytime algorithm $\VVV$ such that, for an honestly generated key pair $(\sk, \pk) \gets \KeyGen(\lambda)$ and value $c$,
		\begin{itemize}
		    \item If $\Dec(\sk, c) \in \{0,1\}$, then $\VVV(\pk, c) = 1$
		    \item If $\Dec(\sk, c) = \bot$, then $\VVV(\pk, c) = 0$
		\end{itemize}

	\item \textbf{(Perfect) Randomization of Ciphertexts:}
		There exists a randomized polytime algorithm $\Ran$ such that, 
		for all honestly generated key pairs $(\sk, \pk) \gets \KeyGen(\lambda)$, and all ciphertexts $c_1, c_2$ such that $\Dec(\sk, c_1) = \Dec(\sk, c_2)$, 
		\begin{itemize}
		    \item $d_{TV}(\Ran(\pk, c_1), \Ran(\pk, c_2)) = 0$
		    \item $\Dec(\sk, \Ran(\pk, c)) = \Dec(\sk, c)$
		\end{itemize}
\end{itemize}
\end{definition}

Consider the following search problem $\prandenc$. Given a public key $\pk$ for an encryption scheme $\escheme$, and an i.i.d. sample of $m$ elements from a distribution $D$ supported on encryptions under $\pk$ of a fixed bit $b \in \{0,1\}$, output an encryption of $b$ under $\pk$. 

\begin{definition}[Ciphertext Identification Problem]
	\label{def:ctxt-id-problem}
	An instance of $\prandenc$ is defined as follows. 
	Let $\escheme = (\KeyGen, \Enc, \Dec)$ be a randomizeable encryption scheme (Definition~\ref{def:rand-enc}). 
	Let $\lambda, m \in \N$, and let $D$ be a distribution over the ciphertext space $\cspace$ of $\escheme$. 
	Given public key $\pk$, honestly generated as $(\pk, \sk) \gets \KeyGen(\lambda)$, and a sample $S \sim D^m$ drawn i.i.d. from $D$,
	output an element $c^* \in \cspace \cup \bot$ such that 
	\begin{enumerate}
	    \item If $\Dec(\sk, c) = 1$ for all $c \in \supp(D)$, $\Dec(\sk, c^*) = 1$
	    \item If $\Dec(\sk, c) = 0$ for all $c \in \supp(D)$, $\Dec(\sk, c^*) = 0$
	\end{enumerate}

\end{definition}

\subsubsection{DP Algorithm for $\prandenc$}
\label{ssec:dp-randenc-algo}

In this subsection, we present a differentially private algorithm $\dprandenc$ for $\prandenc$. 
Our algorithm removes from the dataset $S$ all $c_i$ for which verification fails, i.e., $\VVV(\pk, c_i) = 0$.  It then pads the remaining elements with $k$ encryptions of 0 under $\pk$ and $k$ encryptions of 1 under $\pk$. An element $c_i$ from the new dataset is then chosen uniformly at random, and the algorithm outputs $\Ran(\pk, c_i)$. 

Padding the dataset with additional ciphertexts, balanced between encryptions of 0 and 1, guarantees privacy by ensuring that exchanging any element of $S$ for another will not significantly change the probability that the ciphertext chosen for rerandomization encrypts a particular bit. 
If the distribution $D$ is supported on one of the two promised distributions, $\dprandenc$ will be correct unless it chooses to rerandomize an inserted ciphertext which encrypts the incorrect bit. So long as the input sample is of size $m$ much larger than $k$, this will happen only with small probability. 

\begin{algorithm}[H]
\KwResult{Outputs a ciphertext $c \in \cspace$}
\nonl \textbf{Input:} Sample $S$ of $m$ elements drawn i.i.d. from $D$ \\
\nonl \textbf{Parameters:}
\begin{itemize}
    \item Privacy $\eps$, failure probability $\beta$, padding length $k = \frac{1}{\eps}$
    \item Sample Complexity $m=m(\eps, \beta) \in O\left(\frac{1}{\varepsilon\beta} \right)$
\end{itemize}
\nonl \textbf{Algorithm:}\\
\begin{enumerate}
        \item For $i \in [m]$, remove $c_i$ from $S$ if $\VVV(\pk, c_i) = 0$
        \item Add $k$ ciphertexts $\Enc(\pk, 0)$ to the dataset
        \item Add $k$ ciphertexts $\Enc(\pk, 1)$ to the dataset
        \item Choose $c$ uniformly at random from the new dataset
\end{enumerate}
\textbf{return} $\Ran(\pk, c)$
\caption{$\dprandenc$}
\label{alg:dprandenc}
\end{algorithm}

\begin{lemma}\label{lem:dprandenc-dp-corr}
Let $\eps, \beta \in (0,1/2)$. Then for $m \in \Omega(1/(\eps\beta))$ and $k = 1/\eps$,  
$\prandenc$ (\Cref{alg:dprandenc}) runs in time $\poly(\lambda, 1/\eps, 1/\beta)$, is $\eps$-DP, and correct except with probability at most $\beta$.
\end{lemma}
\begin{proof}
We begin by showing $\dprandenc$ is $\eps$-DP. 
Note that the last step of $\dprandenc$ calls $\Ran$ on a ciphertext $c$ that is guaranteed to be a valid encryption of a bit $b\in \{0,1\}$, since all inputs failing verification are removed from $S$ before $c$ is drawn, and only valid ciphertexts under $\pk$ are added to the input dataset. 

We will bound how much the probability that $c$ encrypts a fixed bit $b$ can differ across neighboring data sets. 
Let $c$ be a random variable denoting the ciphertext chosen for rerandomization. For all $b \in \{0,1\}$ and neighboring datasets $S, S'$,  let $S_b$, $S_{\neg b}$, and $S_{\bot}$ denote the subsets of $S$ such that $\Dec(\sk, c) = b$, $\Dec(\sk, c) = \neg b$, and $\Dec(\sk, c) = \bot$ respectively, and let $S'_b$, $S'_{\neg b}$, and $S'_{\bot}$ be defined analogously. Then
\begin{align*}
\frac{\Pr[\Dec(\sk, c) = b \mid S]}{\Pr[\Dec(\sk, c) = b \mid S']} 
&= \frac{|S_b| + k}{m - |S_{\bot}| + 2k }\cdot \frac{m - | S'_{\bot}| + 2k }{|S'_b| + k} \\
&= \frac{|S_b| + k}{|S_b| + |S_{\neg b}|+ 2k }\cdot \frac{|S'_b| + |S'_{\neg b}| + 2k }{|S'_b| + k} \\
&\leq \frac{|S_b| + k}{|S_b| + |S_{\neg b}|+ 2k }\cdot \frac{|S_b| + |S_{\neg b}| + 2k }{|S_b| - 1 + k} \\
& = \frac{|S_b| + k}{|S_b| - 1 + k} \\
& \leq \frac{k+1}{k},
\end{align*}
where the first inequality follows from $S, S'$ neighboring, and the fact that for $a > b$, $\frac{a}{b-1} > \frac{a+1}{b}$, so assuming $|S_b| \geq 1$ and $|S'_b| = |S_b| - 1$ maximizes the rightmost fraction.
Because the output distribution of $\Ran(\pk, c)$ is the same for all ciphertexts encrypting the same bit under $\pk$, it follows that for all subsets $T \subseteq \mathcal{C}$,

\begin{align*}
\Pr[\dprandenc(S) \in T] 
&= \Pr[\Ran(\pk, c) \in T \mid \Dec(\sk, c) = 1]\cdot \Pr[\Dec(\sk, c) = 1 \mid S] \\
& \quad \quad \quad + \Pr[\Ran(\pk, c) \in T \mid \Dec(\sk, c) = 0]\cdot \Pr[\Dec(\sk, c) = 0 \mid S]. 
\end{align*}
Using $p_b$ to denote $\Pr[\Ran(\pk, c) \in T \mid \Dec(\sk, c) = b]$, for $b \in \{0,1\}$, we then have that
\begin{align*}
    \Pr[\dprandenc(S) \in T]
    &= p_1\cdot \Pr[\Dec(\sk, c) = 1 \mid S] + p_0\cdot \Pr[\Dec(\sk, c) = 0 \mid S] \\
    &\leq \frac{k+1}{k}\left(p_1 \cdot \Pr[\Dec(\sk, c) = 1 \mid S'] + p_0\cdot   \Pr[\Dec(\sk, c) = 0 \mid S']\right)\\
    &= \frac{k+1}{k} \cdot \Pr[\dprandenc(S') \in T] \\
    &\leq e^{\eps}\cdot \Pr[\dprandenc(S') \in T] 
\end{align*}
where the final inequality follows from taking $k = 1/\eps$ and observing $1+ x \leq e^x$.

It remains to argue correctness of $\dprandenc$ when the sample $S$ is drawn from one of the promised distributions. In this case, the input sample $S$ consists of $m$ valid encryptions of the same bit $b$ under $\pk$. Because $\Ran(\pk, c)$ is plaintext-preserving, the probability that $\prandenc$ is incorrect given $S$, i.e., outputs a ciphertext encrypting $\neg b$, is exactly the probability that one of the inserted ciphertexts encrypting $\neg b$ is chosen for rerandomization. This happens with probability $\frac{k}{m+2k}$, so taking $m > k/\beta = 1/(\eps\beta)$ ensures $\prandenc$ is correct except with probability $\beta$. 
\end{proof}

\subsubsection{Cryptographic Adversary from Replicable Algorithm for $\prandenc$}\label{ssec:no-rep-randenc-algo}

In this subsection, we show that if there exists a replicable polytime algorithm, $\mathcal{B}$, for $\prandenc$, instantiated with a randomizable encryption scheme $\escheme$, then there exists an adversary breaking the security guarantee of $\escheme$. 
To break security, the adversary must be able to distinguish whether a ciphertext $c$ encrypts a 1 or a 0 with probability noticeably better than a coin flip. 

The high level idea is as follows. The adversary can first use the ciphertext rerandomization procedure to generate a dataset of ciphertexts encrypting the same bit as $c$. It can then generate a dataset of ciphertexts encrypting 0 by encrypting 0 under the public key and rerandomizing the resulting ciphertext. 
The adversary will then invoke $\mathcal{B}$ on both datasets, fixing the same randomness for both invocations. If the outputs of both invocations are equal, the adversary will guess that $c$ encrypts a 0, and guess $c$ encrypts 1 otherwise.

Because rerandomization is perfect and $\mathcal{B}$ is a replicable algorithm for $\prandenc$, if $c$ encrypts a 0, $\mathcal{B}$ will with high probability produce the same output ciphertext for both invocations. 
If $c$ encrypts a 1, $\mathcal{B}$ will can only output the same ciphertext for both invocations if one of the two invocations is incorrect, and so with good probability, the two outputs will differ. 
This implies the adversary will have good distinguishing probability, breaking the security of the underlying cryptosystem.

\begin{algorithm}[H]
\KwResult{Outputs a bit $b'$}
\nonl \textbf{Input:} public key $\pk$, ciphertext $c$\\
\nonl \textbf{Algorithm:}\\
\begin{enumerate}
    \item Draw a random string $r$
    \item $c_0 \gets \Enc(\pk, 0)$
    \item Generate a set $S_0$ of $m$ ciphertexts by running $\Ran(\pk, c_0)$ $m$ times 
    \item $c_0 \gets \mathcal{B}(\pk, S_0; r)$
    \item Generate a sample $S$ of $m$ ciphertexts by running $\Ran(\pk, c)$ $m$ times
    \item $c \gets \mathcal{B}(\pk, S; r)$
    \item if $c_0 = c$, then $b' = 0$, otherwise $b' = 1$
\end{enumerate}
\textbf{return} $b'$
\caption{Adversary $\mathcal{A}$}
\label{alg:adversary}
\end{algorithm}

\begin{lemma}~\label{lem:adversary}
Let $\escheme$ be a randomizeable encryption scheme, let $\prandenc^{\escheme}$ denote the instantiation of $\prandenc$ with $\escheme$. Let $\mathcal{B}$ be a $\rho$-replicable algorithm for $\prandenc^{\escheme}$ with failure probability $\beta$, running in time $\poly(\lambda, \rho, \beta)$, and with sample complexity $m \in \poly(\lambda, \rho, \beta)$. Then there exists an adversary $\mathcal{A}$ running in time $\poly(\lambda, \rho, \beta)$ such that 
$$\Pr[\mathcal{A}(\pk, c) = 1 \mid c \gets \Enc(\pk, 1)] - \Pr[\mathcal{A}(\pk, c) = 1 \mid c \gets \Enc(\pk, 0)] \geq 1 - 2\beta - \rho.$$
\end{lemma}
\begin{proof}
The adversary $\mathcal{A}(\pk, c)$ outputs 1 whenever $c_0 \neq c$. The distribution from which $S_0$ is drawn is supported entirely on encryptions of $0$ and, conditioned on $c \gets \Enc(\pk, 1)$, the distribution from which $S$ is drawn is supported entirely on encryptions of $1$. 
Then $c_0 \neq c$ except when one of the two calls to $\mathcal{B}$ is incorrect, which happens with probability at most $2\beta$. Conditioned on $c \gets \Enc(\pk, 0)$, $S_0$ and $S$ comprise i.i.d. samples from the same distribution over encryptions of 0. 
In this case, $c_0 \neq c$ if either call to $\mathcal{B}$ fails to be replicable, which happens with probability at most $\rho$. 
Therefore 
$$\Pr[\mathcal{A}(\pk, c) = 1 \mid c \gets \Enc(\pk, 1)] - \Pr[\mathcal{A}(\pk, c) = 1 \mid c \gets \Enc(\pk, 0)] \geq 1 - 2\beta - \rho .$$ 
\end{proof}
In particular, taking $\beta,\rho$ to be constant in Lemma~\ref{lem:adversary} gives an adversary breaking the security of $\escheme$, yielding the following theorem as a corollary. 

\begin{theorem}\label{thm:adversary-breaks-enc}
Let $\escheme$ be a randomizeable encryption scheme, and let $\prandenc^{\escheme}$ denote the instantiation of $\prandenc$ with $\escheme$. Then there does not exist a $\rho$-replicable algorithm for $\prandenc^{\escheme}$ with failure probability $\beta$, running in time $\poly(1/\lambda)$, for $\rho < 1/4$ and $\beta < 1/8$.
\end{theorem}
\begin{proof}
If there exists a $\rho$-replicable algorithm $\mathcal{B}$ for $\prandenc^{\escheme}$ with failure probability $\beta < 1/8$ and replicability parameter $\rho < 1/4$ running in time $\poly(\lambda)$, then by Lemma~\ref{lem:adversary}, there exists an adversary $\mathcal{A}$ running in time $\poly(\lambda)$ such that 
$$\Pr[\mathcal{A}(\pk, c) = 1 \mid c \gets \Enc(\pk, 1)] - \Pr[\mathcal{A}(\pk, c) = 1 \mid c \gets \Enc(\pk, 0)] \geq 1/2 > \negl(\lambda),$$
and therefore $\mathcal{A}$ breaks the security of $\escheme$. 
\end{proof}
\subsubsection{Instantiating $\prandenc$ with the Goldwasser-Micali Cryptosystem}
\label{ssec:gm}

Here we recall the high-level structure of the Goldwasser-Micali public-key cryptosystem, introduced in~\cite{STOC:GolMic82}. The security of the cryptosystem relies on the hardness of deciding quadratic residuosity for integers modulo a semiprime $N$. Informally, encryptions of $0$ are quadratic residues modulo $N$, while encryptions of $1$ are non-residues. Because multiplying an integer $c$ by a quadratic residue modulo $N$ preserves quadratic residuosity of $c$, Goldwasser-Micali ciphertexts can be efficiently rerandomized with only a public key. The rerandomization procedure will pick a quadratic residue $r^2$ uniformly at random, and output its product with the given ciphertext modulo $N$.

\begin{definition}[Goldwasser-Micali Cryptosystem (\cite{STOC:GolMic82})]
\label{def:gol-mic}
The Goldwasser-Micali cryptosystem is defined over a plaintext message space $\mathcal{M}= \{0,1\}$ and ciphertext space $\cspace = \Z^{*}_N$, for $N$ a semiprime. The cryptosystem comprises the following routines.
\begin{itemize}
    \item $\KeyGen(\lambda)$: Sample $p,q$ distinct primes of bit-length $O(\lambda)$ and let $N = pq$. Choose $x$ to be a quadratic non-residue modulo $N$ with Jacobi symbol $\left(\frac{x}{p}\right) = \left(\frac{x}{q}\right) = -1$. Let $\sk = (p, q)$, $\pk = (N, x)$, and output $(\sk, \pk)$.
    \item $\Enc(\pk, b)$: To encrypt a bit $b \in \{0,1\}$, sample $u \gets_{\mathcal{U}} \Z^{*}_N$ and output $u^2x^b \mod N$.
    \item $\Dec(\sk, c)$: To decrypt a ciphertext $c$, output $\bot$ if $\gcd(c,N) \neq 1$, $1$ if $c$ is not a quadratic residue modulo $N$ and 0 otherwise.  
\end{itemize}

\end{definition}

We now show that the Goldwasser-Micali cryptosystem satisfies the strong rerandomization property described above. We define the verification procedure $\VVV(\pk, c)$ to output 1 if $\gcd(c, N) = 1$ and 0 otherwise. We define $\Ran(\pk, c)$ to be the procedure that samples $r$ uniformly at random from $\Z_N^{*}$ and outputs $(\pk, r^2c \mod N)$.

\begin{lemma}
The Goldwasser-Micali cryptosystem is a rerandomizeable encryption scheme (\cref{def:rand-enc}), for the rerandomization procedure described above.
\end{lemma}

\begin{proof}
Because $\Dec(\sk, c) = \bot$ if and only if $\gcd(c, N) \neq 1$, and $\VVV(\pk, c) = 0$ if and only if $\gcd(c, N) \neq 1$, $\VVV$ satisfies the requirement of \cref{def:rand-enc}. The rerandomization procedure multiplies a ciphertext $c$ by a random quadratic residue modulo $N$, and therefore preserves quadratic residuosity of $c$. This in turn preserves the plaintext message encrypted by $c$, and so $\Ran$ satisfies $\Dec(\sk, \Ran(\pk, c)) = \Dec(\sk, c)$.

To show that for all $c, c'$ such that $\Dec(\sk, c) = \Dec(\sk, c')$, $d_{TV}(\Ran(\pk, c), \Ran(\pk, c')) = 0$, let $b\in \{0,1\}$, $c = u^2x^b \mod N$ and $c'= v^2x^b \mod N$ be honest encryptions of $b$ under $\pk = (N, x)$. It follows that 
\begin{align*}
\Pr[\Ran(\pk, c) = a] 
&= \Pr[\Ran(\pk, u^2x^b \mod N) = a] \\
&= \Pr[r^2u^2x^b = a \mod N] \\
&= \Pr[r^2 = u^{-2}x^{-b}a \mod N] \\
&= \Pr[(u^{-1}vr)^2 = u^{-2}x^{-b}a \mod N] \\
&= \Pr[r^2v^2x^b = a \mod N ] \\
&= \Pr[\Ran(\pk, c') = a],
\end{align*}
where the fourth equality follows from $r$ being chosen uniformly at random from $\Z^{*}_N$.
\end{proof}

\subsection{Correlated Sampling via One-Way Function Inverters}
\label{sec:owfi}

As we saw in \Cref{sec:PG2rep}, correlated sampling gives a generic way for converting a perfectly generalizing algorithm into a replicable one. 
%Let $D_1, D_2$ be two distributions over $\X$ such that $\dtv(D_1,D_2)$ is small. A correlated sampling algorithm $\Bcal(D; r)$ for $\{D_1, D_2\}$ is a randomized algorithm that uses randomness $r$ to produce examples $x \in \X$ distributed according to distribution $D$. Furthermore, $\Bcal$ should produce the same example for two close distributions $D_1$ and $D_2$, when $\Bcal$ is run with the same random string $r$; i.e., $\Pr_{r} [\Bcal(D_1; r) \ne \Bcal(D_2;r)]$ should be proportional to $\dtv(D_1, D_2)$. 
%More generally, a correlated sampling algorithm for distributions $\{D_i\}_i$ satisfies the second condition for all $D_j, D_k \in \{D_i\}_i$. 
%\rexnote{The next paragraph is wordy and not direct enough. The point I'm trying to make is that correlated sampling can be used as a subroutine to make reproducible algorithms}
%Say $\Acal(S; r)$ is a learning algorithm that takes as input a sample $S \sim D^m$ and uses internal randomness $r$ to produce an output. Let $D_1$ be the distribution of outputs of $\Acal(S_1; r)$ when sample $S_1$ is fixed, and analogously for $D_2$. 
%\rexnote{Maybe more clearly, here describe why many algorithms can ensure that $D_1$ and $D_2$ are close.}
%If $\Acal$ is designed such that $D_1$ and $D_2$ are close, then we can apply a correlated sampling algorithm to create a replicable algorithm. 
In this section, we show that the existence of efficient one-way function inverters implies the ability to efficiently perform correlated sampling on arbitrary distributions over $\{0,1\}^n$. Specifically, we show that 
%(i) if there are no one-way functions, we can do average case implicit correlated sampling, and (ii) 
if there are no non-uniform 
one-way functions, then there is polynomial time implicit correlated sampling.
%\rexnote{Maybe add definitions for ``average-case implicit correlated sampling", `non-uniform owfs', and 'polynomial time implicit correlated sampling'}

%\subsubsection{Correlated Sampling Background}

%\rexnote{Somewhere (either in the paper intro or here), discuss prior attempts at achieving correlated sampling; also explain how our additional assumption leads to a better qualitative/quantitative correlated sampler}
%\toni{I put this discussion in 2.5. Also removing pointer to discuss russell's pessiland.}

\begin{theorem}
    \label{thm:correlated-sampler}
    Assuming uniform one-way function inverters exist (Definition~\ref{def:inverter-unif}), 
    Algorithm $\correlatedsampler$ is an $(m,n, \nu)$-implicit correlated sampling algorithm that runs in time polynomial in $m$, $n$, and $1/\nu$. 
\end{theorem}

\begin{proof}
    The distributional accuracy property is shown in the proof of Lemma~\ref{lem:correlated-sampler-distributional-accuracy}. 
    The correlated sampling property is shown in the proof of Lemma~\ref{lem:correlated-sampler-correlations}.
    The runtime of $\correlatedsampler$ is shown in the proof of Lemma~\ref{lem:correlated-sampler-runtime}. 
\end{proof}

\subsubsection{Relevant Definitions}
\label{ssec:owfi-prelims}

We model samplable distributions by considering the distribution induced by giving random inputs to circuits. Furthermore, we allow for a distributional error parameter $\nu$, giving some slack in the correctness of a correlated sampler. 

\begin{definition}[Implicit Correlated Sampling Algorithm]
\label{def:implicit-correlated-sampling-problem}
Let $m, n \in \Z^+$, and let $C: \{0,1\}^m \rightarrow \{0,1\}^n$ denote a circuit. 
Let distributional error parameter $\nu >0$. 
$\Bcal(C, \nu; r)$ is an  \emph{$(m,n,\nu)$-implicit correlated sampling algorithm} if the following conditions hold:
\begin{enumerate}
    \item \textbf{Inputs/Outputs:} $\Bcal$ takes as input a circuit $C:\{0,1\}^m \rightarrow \{0,1\}^n$, a distributional error parameter $\nu$, and a random string $r$. $\Bcal$ outputs a string in $\{0,1\}^n$.  

%    \rexnote{when the argument works mathematically, and this section's intro is written, change the notation from $\nu$ to $\beta$ and $\beta$ to something else like $\tau$, or $\nu$ (to match glossary notation)}
    
    \item \textbf{$\nu$-distributional accuracy:} For all circuits $C: \{0,1\}^m \rightarrow \{0,1\}^n$, the distributions $D_C$ and $D_{\Bcal(C, \nu)}$ satisfy $\dtv(D_C, D_{\Bcal(C, \nu)}) \le O(\nu)$.
    
    Here, $D_C$ denotes the distribution over $\{0,1\}^n$ induced by querying $C(r)$ on  uniformly random inputs $r$, i.e., probability density function $p_{D_C}(x) = \Pr_{r \sim U^m}[C(r) = x]$. 
    Similarly,
    $D_{\Bcal(C, \nu)}$ denotes the distribution over $\{0,1\}^n$ induced by querying $\Bcal(C, \nu; r)$ with uniformly random strings $r$.
%    \rexnote{The previous notational clarification sentence could be written more precisely/clearly}
    
    \item \textbf{Correlated sampling:} For all pairs of circuits $C_1, C_2: \{0,1\}^m \rightarrow \{0,1\}^n$, 
    $\Pr_{r}[\Bcal(C_1, \nu; r) \ne \Bcal(C_2, \nu; r)] \in O(\dtv(D_{C_1}, D_{C_2}) + \nu)$.
\end{enumerate}
\end{definition}

%In the implicit correlated sampling problem, the input is a circuit $D$ that uses $m$ bits of randomness to produce an output in $\{0,1\}^n$, i.e., $D: \{0,1\}^m \rightarrow \{0,1\}^n$.  Given an error parameter $\nu$, we want a sampling algorithm $\sc{S}(\nu, D, rand)$ whose distribution is within $\nu$ statistical distance of $D$, and so that  for any two such sampling circuits, $D_1,D_2$,  $Prob_{rand} [\sc{S}(\nu, D_1, rand) \neq \sc{S} (\nu , D_2 , rand)] \le O( SD(D_1,D_2)+ \nu)$.

We assume we can invert any one-way function on almost all inputs.  %For now,
Specifically, we assume that there is no non-uniform one-way function family, so
that there is a uniform way of inverting any circuit computing a function via a polynomial-time inverter. 

\begin{definition}[Uniform One-Way Function Inverters]
\label{def:inverter-unif}
Let $\nu' > 0$. $\Ical_{\nu'}(C, y)$ is a \emph{uniform one-way function inverter with error $\nu'$} if, for any circuit $C:\{0,1\}^m \rightarrow \{0,1\}^n$, 
$\Pr_{r' \sim \{0,1\}^m} [C(\Ical_{\nu'}(C, C(r')))=C(r') ] \ge 1 - \nu'$, and $\Ical$ runs in randomized polynomial time in $m$, $n$, and $1/\nu'$.
\end{definition}

In this argument, we will choose $\nu'$ to be inverse polynomially small in $m, n,$ and $1/\nu$.
In addition, we assume that $C(r)$ can be efficiently computed.\footnote{Note that if there is no such efficient circuit $C$ that produces a sample from a distribution $D_C$ (on a uniformly random input), then $D_C$ is hard to sample from, and designing an efficient correlated sampling algorithm whose marginal distribution is $D_C$ (when given circuit $C$) is hopeless.} Thus, we can check if and when the inverter succeeds. For notational convenience, we say that the inverter $\Ical$ returns ``$\perp$'' if it does not succeed.

Our correlated sampler randomly samples from pairwise-independent hash families in its subroutines. 

\begin{definition}[Pairwise-Independent Hash Family]
\label{def:pairwise-independent-hash-family}
A family of Boolean functions $\mathcal{H} = \{H | H: \{0,1\}^m \rightarrow \{0,1\}^n \}$ is \emph{pairwise-independent} if, for all $r_1 \ne r_2 \in \{0,1\}^m$ and $x_1, x_2 \in \{0,1\}^n$, 
$\Pr_{H \in \mathcal{H}
} [H(r_1) = x_1 \land H(r_2) = x_2 ] = 2^{-2n}$.
\end{definition}

\subsubsection{Algorithm Overview}\label{sssec:corrsamp-overview}

A correlated sampling algorithm $\Bcal$ accomplishes two goals. First, $\Bcal(C, \nu; r)$ needs to accurately sample from the distribution $D_C$. Second, $\Bcal$ must convert a random string $r$ into the same output when run on distributionally close circuits $C_1$ and $C_2$, with high probability.  In other words, $\Bcal$ must choose a consistent way to map random strings $r$ to elements in the support of $D_C$.

%\iffalse
For intuition, consider a restricted case of correlated sampling problems in which the distributions $D_C$ induced by random inputs to circuits $C: \{0,1\}^m \rightarrow \{0,1\}^n$ are promised to be uniformly supported on $2^\ell$ elements $x \in \{0,1\}^n$ for some fixed $\ell$. Let $k$ be a small slack parameter, and consider the following sampler:
\begin{enumerate}
    \item Draw a random hash function $H: \{0,1\}^n \to \{0,1\}^{\ell+k}$
    \item Draw a random string $u \in \{0,1\}^{\ell + k}$
    \item Run the inverter on $u$: $r=\Ical_{\nu'} (C \circ H, u)$
    \item If $r=\bot$ (i.e. $u$ is not in the support of $C \circ H$), repeat. Else return $y^*=C(r)$
\end{enumerate}

\begin{figure}[h]
\begin{center}
\begin{tikzpicture}
    \node[ellipse,draw,align=center] (inputs) {Randomness \\ $\{0,1\}^m$};
    \node[above=0mm of inputs,align=center] {\textbf{Step 3.} \\ $r = \Ical_{\nu'}(C \circ H, u)$ };
 %   \node[circle,fill,inner sep=.5mm,below=-5mm of inputs] (x){};
 %   \node[right=0mm of x] (x-txt) {$x' = \Ical_{\nu'}(C \circ H, u)$};

    \node[ellipse,draw,align=center,right=35mm of inputs] (range) {$\supp(D_C)$ \\ $\{0,1\}^n$};
    \node[above=0mm of range,align=center] {\textbf{Step 4.} \\ $y^{*} = C(r)$};
 %   \node[circle,fill,inner sep=.5mm,below=-6mm of range] (y){};
%    \node[right=0mm of y] (y-txt) {$y^{*}$};

    \node[ellipse,draw,align=center,right=35mm of range] (h-range) {Hash range \\ $\{0,1\}^{\ell + k}$};
    \node[above=0mm of h-range,align=center] {\textbf{Step 2.} \\ $u \sim \mathcal{U}(\{0,1\}^{\ell + k})$};

    \node[above left=3.5mm and 13mm of h-range,align=center] {\textbf{Step 1.} \\ Draw $H$};
%    \node[circle,fill,inner sep=.5mm,below=-6mm of h-range] (u){};
 %   \node[right=0mm of u] (u-txt) {$u$};
   % \node[right=7mm of u]{2. Pick $u$};

    \draw[-stealth, line width=.3mm] (inputs.east) -- node[above]{$C$} (range.west) ;
    
    \draw[-stealth, line width=.3mm] (range.east) -- node[above]{$H$} (h-range.west) ;
    
    \draw[-stealth, line width=.3mm] (h-range.south) to [out=-150, in=-30] node[below]{$\Ical_{\nu'}(C \circ H, \cdot)$}(inputs.south) ;
    
 %   \draw[-stealth, line width=.3mm] (h-range.south) to [out=-150, in=-30] node[below]{$\Ical_{\nu'}(C \circ H, \cdot)$}(h-range.south) ;
    %
\end{tikzpicture}
\caption{High-level structure of correlated sampler for uniform $D_C$}
\label{fig:corr-samp-uniform}
\end{center}
\end{figure}

The high level idea is that $k$ can be chosen large enough such that $\supp(D_C)$ has few collisions with good probability (for random $H$), but small enough s.t. $2^k$ (and therefore the runtime) remains polynomial in the relevant parameters. Assuming no collisions occur, it is easy to see this process is a correlated sampler since each element in the support of $D_C$ is sampled uniformly at random, and moreover applied to distinct circuits $C_1$ and $C_2$, $y^*$ only differs if the sampler hits a hash value $u$ that contains an element in the symmetric difference $\supp(C_1) \Delta \supp(C_2)$. Since there are no collisions, this occurs exactly with probability $d_{TV}(D_{C_1},D_{C_2})$ as desired.
% Then, a random pairwise-independent hash function $H_1$ from $n$ to $\ell + k$ bits (for some small $k$) is likely to have few collisions for $x \in \supp (D_C)$. By randomly generating a string $u \in \{0,1\}^{\ell + k}$, running the inverter $\Ical_\nu (C \circ H_1, u)$, and applying $C$ to the output, we are therefore likely to return a unique $x' \in \supp(D_C)$ with probability $2^{-\ell -k}$. Notice that the distribution 

% Inverting a random hash function not only allows us to choose an $x$ proportional to probability density $p_D(x)$, but it also allows us to choose $x$ consistently for close distributions $D_{C_1}$ and $D_{C_2}$.
% In other words, by using the shared randomness $r$ to pick  $H_1$ and $u$, we can define a canonical $x$ to return with high probability (for every circuit in this restricted class). 

Moving to the general case, our algorithm $\correlatedsampler$ applies this idea as follows. $\correlatedsampler$ divides the distribution $D_C$ into ``levels'' $\ell$, such that each level contains elements in the support with probability density near $2^{-\ell}$ (specifically, those in the range $(2^{-\ell-1}, 2^{-\ell}]$). We now pick a level uniformly at random, and hope to apply the above process. Intuitively, the main challenge is that given the output $y^*$, we need to ensure $y^*$ actually belongs at level $\ell$. This is done through the introduction of a \textit{second hash function} $H_2:\{0,1\}^m \to \{0,1\}^{m-\ell+k}$. In particular, fixing $u$ and $y^*$ as in the simplified variant, we wish to estimate $|C^{-1}(y^*)|=|(H \circ C)^{-1} (u)|$ (assuming no collisions). To do this, we call the inverter on the concatenated function $F_{C, l, H, H_2}(r) \eqdef H(C(r)) \mathbin\Vert H_2(r)$\footnote{
We use `$\mathbin\Vert$' to denote concatenation of strings.} on many pairs of the form $(u || v)$, where $v \in \{0,1\}^{m-\ell+k}$ is chosen uniformly at random. Since we have fixed $u$, the inverter can only succeed on this call when $v = H_2(r)$ for some $r \in (C\circ H)^{-1}(u)$. Since $v$ is chosen uniformly at random, the success probability of the inverter is then directly proportional to the density of $y^*$, allowing us to determine whether or not $y^*$ is in level $\ell$ with high probability.\footnote{Of course this only holds assuming few collisions. To handle the general case, we actually draw a new $H_2$ with every choice $v$ to ensure this holds across all rounds.} We can then return $y^*$ if it is in the chosen level, and repeat the process from the beginning if not. Because we have chosen $u$ uniformly at random from $\{0,1\}^{\ell + k}$, for any $x$ it holds that $x =H^{-1}(u)$ with probability $2^{-\ell-k}$ (assuming no collisions). Then this approach allows us to sample uniformly from level $\ell$, where every $x$ in level $\ell$ is output with probability proportional to $2^{-\ell}$. Note that since the true density may be a constant factor away from $2^{-\ell}$, this is not yet quite enough to achieve our true target distributional accuracy---we will address this detail in the next section.

\begin{figure}[h]
\begin{center}
\begin{tikzpicture}

\node[align=center,draw=black!75,rounded corners,inner sep=4pt,minimum width=7cm](l0){$\ell = 0 \mid x: p_D(x) \in (1/2,1]$};
\node[align=center,draw=black!75,rounded corners,inner sep=4pt,minimum width=6cm, below=0.5mm of l0](l1){$\ell = 1 \mid x: p_D(x) \in (1/4,1/2]$};
\node[align=center,inner sep=4pt,minimum width=5cm, below=0.0mm of l1](dots){$\vdots$};
\node[align=center,draw=black!75,rounded corners,inner sep=4pt,below=0.0mm of dots](lm){\tiny$\ell = m-1\mid  x: p_D(x) \in [2^{-m},2^{-m+1}]$};

    \node[above=0mm of l0,align=center] {$\supp(D_C)$ \\ $\{0,1\}^n$};

    \node[above left=3mm and -15mm of l0,align=center] {\textbf{Step 1.} \\ Pick $\ell$};

    \node[ellipse,align=center,below right=0mm and 5mm of l1] (step2) {\textbf{Step 2.} \\ Pick $H_1$};

    \node[ellipse,draw,align=center,below right=25mm and 35mm of lm] (h-range) {$H_1$ range \\ $\{0,1\}^{\ell + k}$};

    \node[above=2mm of h-range,align=center] {\textbf{Step 3.} \\ $u \sim \mathcal{U}(\{0,1\}^{\ell + k})$};

    \node[ellipse,draw,align=center,below=23mm of lm] (h2-range) {$H_2$ range \\ $\{0,1\}^{m-\ell + k}$};

    \node[above=0mm of h2-range,align=center] {\textbf{Step 5.} \\ Sample many \\ $v^{(i)} \sim \mathcal{U}(\{0,1\}^{m - \ell + k})$};

    \node[ellipse,draw,align=center, below left=25mm and 35mm of lm] (randomness) { Randomness \\ $\{0,1\}^m$};
    
    \node[below=8mm of randomness,align=center] {\textbf{Step 6.} \\ Estimate $p_D(y^*)$ \\ by inverting $F$ on\\ all $(u,v^{(i)})$ pairs};

   \node[below left=-2mm and 15mm of l1,align=center] {\textbf{Step 7.} \\ If $p_D(y^*)$ in chosen $\ell$ \\ return $y^*$, \\ o/w go to Step 1.};

    \node[above right=-2mm and 15mm of randomness,align=center] {\textbf{Step 4.} \\ Pick $H_2$};

    \draw[-stealth, line width=.3mm] (dots.east) -- node[above, inner sep=5mm]{$H_1$} (h-range.north west) ;

    \draw[-stealth, line width=.3mm] (randomness.east) -- node[below]{$H_2$} (h2-range.west) ;

    \draw[-stealth, line width=.3mm] (randomness.north east) -- node[above, inner sep=5mm]{$C$} (dots.west) ;
    
    \draw[-stealth, line width=.3mm] (h-range.south) to [out=-130, in=-50] node[above,inner sep=8mm]{$\Ical_{\nu'}(F, \cdot, \cdot)$}(randomness.south) ;

    \draw[-stealth, line width=.3mm] (h2-range.south) to [out=-120, in=-50](randomness.south) ;
    
\end{tikzpicture}
\caption{High-level structure of correlated sampler for general $D_C$}
\label{fig:corr-samp-general}
\end{center}
\end{figure}

\subsubsection{Algorithm Description and Pseudocode 
}
\label{ssec:owfi-construction}

We now give pseudocode for the $\correlatedsampler$ algorithm and its subroutines, in addition to a more detailed description. The main algorithm $\correlatedsampler$ takes as input a circuit $C$ and an error parameter $\nu$. As described in Section~\ref{sssec:corrsamp-overview}, at each iteration of the main loop, $\correlatedsampler$ picks a level $\ell$ uniformly at random, and then tries to sample an element that has probability roughly $2^{-\ell}$ under $D_C$. In addition to drawing a hash function $H_1$ with range $\{0,1\}^{\ell +k}$, and an element $u$ from that range, it will also sample a random threshold parameter $\beta \in (1,2]$, and invoke the subroutine $\elementfinder_{C, \nu, \ell, \beta}(H_1, u)$ with these parameters and inputs. We will properly motivate this new parameter $\beta$ shortly, but looking ahead, it will help us avoid the distributional accuracy issues present in Section~\ref{sssec:corrsamp-overview}. 

In the pseudocode for $\correlatedsampler$ (Algorithm~\ref{alg:correlatedsampler}), $k$ is chosen to be large enough that we will be able to avoid problematic collisions for all hash functions with high probability, but small enough to ensure polynomial runtime. The value $T_1$ is chosen to be large enough that $\correlatedsampler$ returns $x \neq \bot$ with high probability, but also small enough such that we can guarantee certain simplifying assumptions will hold with high probability across all rounds of $\correlatedsampler$. 
\begin{figure}[H]
	\begin{algorithm}[H]
		\caption{$\correlatedsamplerfull$\\
			Input: Circuit $C: \{0,1\}^m\rightarrow \{0,1\}^n$, distributional error parameter $\nu$, and random string $r$\\
			Output: An element $x \in \{0,1\}^n$
		}
		\label{alg:correlatedsampler}
		\begin{algorithmic}
		    \STATE $k \gets \Theta( \log (m)+ \log (1/\nu))$
            \STATE $T_1 \gets \Theta(mk2^{k} \log (1/\nu))$
            \FOR{$t=1$ to $t = T_1$}

		        \STATE $\beta \gets_r [1,2]$

                \STATE $\ell \gets_r \{0,1,\dots, m \}$
    		       
    		    \STATE $H_1 \gets_r $ pairwise independent hash function from $n$ bits to $\ell+k$ bits 
    		        
    		    \STATE $u \gets_r \{0,1\}^{\ell+k}$
                
%                \STATE $(x,q) \gets \elementfinder_{C, \nu, \ell, \beta} (H_1, u; r)$
               \STATE $x \gets \elementfinder_{C, \nu, \ell, \beta} (H_1, u; r)$
                
%                \IF{$(x,q) \ne \perp$}
                \IF{$x\neq \bot$}             
%                    \STATE $\coin \gets_r [0,1]$
%                    \IF{$\coin \le q 2^k$}
                        \RETURN $x$
                    %\ENDIF
                \ENDIF
            \ENDFOR
            \RETURN $\perp$
		\end{algorithmic}
	\end{algorithm}
\end{figure}

The subroutine $\elementfinder_{C, \nu, \ell, \beta}(H_1, u;r)$ follows the approach described in Section~\ref{sssec:corrsamp-overview}, estimating the probability of $y^* = H_1^{-1}(u)$ under $D_C$ by drawing many pairs $(H_2, v)$ of hash functions with elements from their range, and invoking the subroutine $\hashchecker_{C, \nu, \ell, H_1, u}(H_2, v)$ to invert $H(C(r))||H_2(r)$ on each $(u,v)$. This procedure approximates the density of random strings $r$ mapped to $y^*$ by $C$. If $y^*$ is in or nearly in level $\ell$, then $\elementfinder$ will obtain an estimate $\widehat{q}(y^*)$ that is close to $p_D(y^*)2^{\ell - k}$. It will then return $y^*$ only if $\frac{\beta}{2} 2^{-k} < \widehat{q}(y^*) \leq \beta 2^{-k}$, i.e.\ roughly when $\beta2^{-\ell-1}  < p_D(y^*) \leq  \beta2^{-\ell}$. 
% Since $\beta$ is chosen uniformly from the range $(1,2]$, 
% % and so if $\widehat{q}(y^*)$ is in the specified range, it follows that
% this means $y^*$ is indeed near level $\ell$. 

We now reach the core reason for choosing our threshold $\beta$ randomly. As in Section~\ref{sssec:corrsamp-overview}, if we did fix $\beta$, then any $y^*$ at level $\ell$ will be sampled whenever $y^* \in H_1^{-1}(u)$, which happens with probability proportional to $2^{-\ell}$. Since this occurs for any $p_D(y^*) \in [\beta2^{-\ell-1},\beta2^{-\ell}]$, this is not a good enough estimate. The key is to observe that choosing $\beta$ randomly allows us to avoid this kind of uniform sampling over any particular level. Instead we sample from ``fuzzy'' levels, where the choice of $\beta$ shifts the boundaries while maintaining that the fuzzy levels partition $[0,1]$.
% To see how this allows us to achieve distributional accuracy, note that
In slightly more detail, observe that for any $y^*$ there is some $j$ such that 
% we can write its probability under $D_C$ as 
$p_D(y^*) = \alpha 2^{-j-1} + (1-\alpha)2^{-j}$ for $\alpha \in [0,1]$. This means we want $y^*$ to belong to `level' $(j+1)$ with probability $\alpha$, and to level $j$ with probability $(1-\alpha)$. We will show in Lemma~\ref{lem:correlated-sampler-distributional-accuracy} that choosing $\beta$ uniformly at random exactly achieves this.

% Recalling that $y^* = H_1^{-1}(u)$ with probability $2^{-\ell - k}$ over choice of $H_1$ and $u$, to sample $y^*$ with probability proportional to $p_D(y^*)$ means we want only an $\alpha$ probability of returning $y^*$ when we are sampling from level $\ell = j+1$, (conditioning on $y^* = H_1^{-1}(u)$). Similarly, we want only a $1 - \alpha$ probability of returning $y^*$ when we are sampling from level $\ell = j$. We will show in Lemma~\ref{lem:correlated-sampler-distributional-accuracy} that choosing $\beta$ uniformly at random achieves exactly this. There will be an $\alpha$ fraction of choices for $\beta$ that are consistent with returning $y^*$ when we sample from level $\ell = j + 1$ and a $(1 - \alpha)$ fraction of choices for $\beta$ that are consistent with returning $y^*$ when we sample from level $\ell = j$. So randomizing $\beta$ allows us to sample $y^*$ with exactly the right probability from the two levels to return $y^*$ with probability proportional to $p_D(y^*)$, instead of with probability proportional to $2^{-j}$.

In the pseudocode for $\elementfinder$ (Algorithm~\ref{alg:elementfinder}), $k$ is chosen to balance the same constraints we have described for $\correlatedsampler$. The value for $T_2$, the number of hash function and element pairs $(H_2, v)$ used to estimate the probability $p_D(y^*)$, is large enough to ensure a good empirical estimate for $y^*$, but small enough to ensure $\elementfinder$ still runs in time polynomial in $m,n,$ and $1/\nu$.

\begin{figure}[H]
	\begin{algorithm}[H]
		\caption{$\elementfinderfull$\\
		    (Explicit) Input: Hash function $H_1: \{0,1\}^{n} \rightarrow \{0,1\}^{\ell+k}$, string $u \in \{0,1\}^{\ell+k}$, random string $r$\\
			(Implicit) Input: Circuit $C: \{0,1\}^m \rightarrow \{0,1\}^n$, distributional error parameter $\nu$, integer $\ell \in \{0, 1, \dots, m\}$, interval rescaling parameter $\beta$\\
			Output: String $x \in \{0,1\}^n$ and probability $q_x \in ((\beta/2)2^{-k}, \beta 2^{-k}]$.
		}
		\label{alg:elementfinder}
		\begin{algorithmic}
		    \STATE $k \gets \Theta( \log (m)+ \log (1/\nu))$
		    \STATE $T_2 \gets \Theta(\nu^{-2}\log(1/\nu)T_1^{-1}\log(T_1^{-1}))$

		    \FOR{$i=1$ to $i = T_2$}
		    
		        \STATE $H^i_2 \gets_r $ pairwise independent hash function from $m$ bits to $m-\ell+k$ bits 
		        
		        \STATE $v^i \gets_r \{0,1\}^{m-\ell+k}$

                \STATE Run $\hashchecker_{C, \nu, \ell, H_1, u} ( H_2^i, v^i; r)$    
		    \ENDFOR
		    
		    \STATE Let $\widehat{q}_x$ denote the fraction of times $x$ was returned by $\hashchecker$
		    
		    \IF{$\exists$ unique $x$ s.t. $x \ne \perp$ and $\widehat{q}_x \in ((\beta/2)2^{-k}, \beta 2^{-k}]$}
		    
		        \RETURN $x$
		    \ELSE
		        \RETURN $\perp$
		    \ENDIF
		\end{algorithmic}
	\end{algorithm}
\end{figure}

\begin{figure}[H]
	\begin{algorithm}[H]
	\caption{$\hashcheckerfull$\\
		    (Explicit) Input: Hash function (circuit) $H_2: \{0,1\}^{m} \rightarrow \{0,1\}^{m-\ell+k}$, string $v \in \{0,1\}^{m-\ell+k}$, and random string $r$\\
			(Implicit) Input: Circuit $C: \{0,1\}^m \rightarrow \{0,1\}^n$, distributional error parameter $\nu$, integer $\ell \in \{0, 1, \dots, m\}$, hash function $H_1: \{0,1\}^{n} \rightarrow \{0,1\}^{\ell+k}$ and string $u \in \{0,1\}^{\ell+k}$\\
			Output: String $x \in \{0,1\}^n$
		}
		\label{alg:hashchecker}
		\begin{algorithmic}
%		    \STATE $k \gets \Theta( \log (m)+ \log (1/\nu))$

		    \STATE Define circuit $\Fcircuit (r') = H_1(C(r')) \mathbin\Vert H_2(r')$ 

		    %map strings of length $m$ to strings of length $m+2k$.
		    
%		    \STATE Let $y \gets u \mathbin\Vert v$ 
            \STATE $\nu' \gets $ inverse polynomial quantity in $m, n, 1/\nu$. 
	        \STATE $r' \gets \Ical_{\nu'} (\Fcircuit, (u \mathbin\Vert v); r)$ 
	        
	        \IF{$r' = \perp$}
		        \RETURN $\perp$
%		        \rexnote{To do: (somewhere in this section) explain why and when the inverter can return $\perp$}
		    \ELSE
		        \RETURN $C(r')$
		    \ENDIF
		\end{algorithmic}
	\end{algorithm}
\end{figure}

\subsubsection{Analysis -- Structure and Simplifying Assumptions}\label{sssec:corrsamp-ideal-cond}

In this section, we analyze the $\correlatedsampler$ algorithm. Before proceeding with the analysis, we first introduce several simplifying assumptions below we will use throughout. In Section~\ref{sssec:corrsamp-dist-acc}, we analyze the distributional accuracy of $\correlatedsampler$, showing that its distribution over outputs is close to the target distribution. In Section~\ref{sssec:corrsamp-corr-analysis}, we analyze the success probability of $\correlatedsampler$ as a correlated sampler, showing that for two circuits $C_1$ and $C_2$, $\correlatedsampler(C_1, \nu; r) = \correlatedsampler(C_2, \nu; r)$ except with probability proportional to $\dtv(D_{C_1}, D_{C_2})$. In Section~\ref{sssec:corrsamp-runtime}, we show that $\correlatedsampler$ runs in time polynomial in the $m$, $n$, and $1/\nu$. Finally, in Section~\ref{sssec:corrsamp-ideal-proof} we show that our simplifying assumptions hold for the entire execution of $\correlatedsampler$, except with high probability. Thus, all statements about the behavior of $\correlatedsampler$ under the ideal conditions will still hold without these assumptions, except with probability $O(\nu)$.

\begin{definition}[Ideal Conditions]\label{def:ideal-cond}
We collectively refer to the following as the \emph{ideal conditions}.
    \begin{itemize}
        \item The inverter $\mathcal{I}_{\nu'}$ never fails: for all $x$ and $r$ on which $\correlatedsampler$ invokes $\mathcal{I}_{\nu'}(x;r)$, $\mathcal{I}_{\nu'}(x;r) \neq \bot$. 
    %     \item The (pre)-pre-image of $u$ is small: \maxh{Do we actually need this here? I think it's just a useful intermediate technical lemma}
    % \[
    % |C^{-1}H_1^{-1}(u)| \leq O\left(2^{m-\ell} \cdot \frac{mk\log(1/\nu)}{\nu}\right)
    % \]
    \item For every $\ell$, there is at most one $y^* \in H_1^{-1}(u)$ s.t.
    \[
    p_D(y^*) \in [2^{-\ell-4},2^{-\ell+4}]
    \]
    \item $\elementfinder$ always returns $y^*$ or `$\bot$'. Furthermore, for $y^{*}$ returned by $\elementfinder$, we have the stronger condition that $p_D(y^*) \in [2^{-\ell - 2}, 2^{-\ell + 2}].$ 
    % :
    % % \[
    % % \begin{cases}
    % %     (y^*,\hat{q}(y^*)) & \text{if $y^*$ exists}\\
    % %     \bot & \text{otherwise}
    % % \end{cases}
    % % \]
    \item The empirical estimate of $\widehat{q}(y^*)$ is good: 
    % For all $H_1, u$, the values $\widehat{q}(x)$ computed by $\elementfinder_{C, \nu, \ell, \beta}(H_1, u)$ satisfy
    \[
    \hat{q}(y^*) \in \left(1\pm O\left(\nu\right)\right)p_D(y^*)2^{\ell-k}
    \]
    % In particular, it follows that the estimate of $y^{*}$ satisfies:
    % \[
    % \hat{q}(y^*) \in \left(1\pm O\left(\nu\right)\right)p_D(y^*)2^{\ell-k}
    % \]
       % \item There are no collisions: for all $u, H_1$ chosen by $\correlatedsampler$, there is a unique $x \in \{0,1\}^n$ such that $H_1(x) = u$. Similarly, we assume that for all $v, H_2$ chosen by $\correlatedsampler$, there is a unique $r \in \{0,1\}^m$ such that $H_2(r) = v$.
    \end{itemize}
\end{definition}

\subsubsection{Analysis -- Distributional Accuracy.}\label{sssec:corrsamp-dist-acc}

In this section, we analyze the distributional accuracy of $\correlatedsampler$. Denote by $D_{\correlatedsampler}$ the distribution over outputs of $\correlatedsampler(C, \nu)$. We show that $\dtv(D_{\correlatedsampler}, D_C) \leq O(\nu)$, assuming the ideal conditions of Definition~\ref{def:ideal-cond}. 

%We break the analysis into the following steps:
%\begin{enumerate}
%    \item For a single iteration of the main loop of $\correlatedsampler$, and for any $x \in \supp(D_C)$, $x$ is returned by the call to $\elementfinder$ with probability $\frac{1}{(m+1)2^{\ell + k}}$.
%    \item From the ideal conditions of Definition~\ref{def:ideal-cond}   
%\end{enumerate}

\begin{lemma}[Distributional Accuracy of $\correlatedsampler$]
\label{lem:correlated-sampler-distributional-accuracy}
   For all circuits $C: \{0,1\}^m \rightarrow \{0,1\}^n$, the distributions $D_C$ and $D_{\correlatedsampler}$ satisfy $\dtv(D_C, D_{\correlatedsampler}) \le O(\nu)$, assuming the ideal conditions of Definition~\ref{def:ideal-cond} hold for all rounds of $\correlatedsampler$.
   Here, $D_C$ and $D_{\correlatedsampler}$ denote the distributions over $\{0,1\}^n$ induced by querying $C(r)$ and $\correlatedsampler(C, \nu; r)$ respectively with uniformly random strings $r$. 
\end{lemma}

We first prove the following useful lemma, bounding the probability that any $x \in \supp(D_C)$ is returned in a single round. 
\begin{lemma}
    \label{lem:correlated-sampler-round-return-prob}
    Fix an $x \in \supp(D_C)$. Then 
    $\Pr_{H_1, u, \ell, \beta}[\elementfinder_{C, \nu, \ell, \beta}(H_1, u) = x] \in \frac{p_D(x)(1+O(\nu))}{(m+1)2^{k}} $.
\end{lemma}

\begin{proof}
For any $\ell$, we have that $\Pr_{H_1, u}[x = H_1^{-1}(u)] = 2^{-\ell - k}$. Conditioned on $x$ being selected by $H_1$ and $u$, by construction, $\elementfinder_{C, \nu, \ell, \beta}$ returns $x$ whenever $\widehat{q}(x) \in (\frac{\beta}{2}2^{-k}, \beta 2^{-k}]$.
Rewriting $p_D(x) = \gamma2^{-j}$ for $\gamma \in [1/2,1]$, we observe that $x$ will only ever have non-zero probability of being returned by $\elementfinder_{C, \nu, \ell,\beta}(H_1, u)$ when $j-1 \leq \ell \leq j+2$, from the assumptions of Definition~\ref{def:ideal-cond}. Since we have $\widehat{q}(x) \in (1 \pm O(\nu))p_D(x)2^{\ell - k}$ for any $\ell$ in this range, it follows that for these $\ell$,
\begin{align*}
\Pr_{\beta}\left[\widehat{q}(x) \in (\tfrac{\beta}{2}2^{-k}, \beta 2^{-k}]\right]
&\in \Pr_{\beta}[\tfrac{\beta}{2} < p_D(x)2^{\ell}(1 \pm O(\nu))\leq \beta]  \\
&\in \Pr_{\beta}[\tfrac{\beta}{2} < p_D(x)2^{\ell} \leq \beta] \pm O(\nu)p(x)2^\ell. 
\end{align*}
Recalling that $\ell$ is chosen uniformly at random from $\{0, \dots, m\}$, we can then write the probability that $\elementfinder$ returns $x$ as 
\begin{align*}
    \Pr_{H_1, u, \ell, \beta}[\elementfinder_{C, \nu, \ell,\beta}(H_1, u) = x]
    &\in \sum_{\ell \in [j-1, j+2]}\frac{\Pr_{\beta}[\tfrac{\beta}{2} < p_D(x)2^{\ell} \leq \beta] \pm O(\nu)p_D(x)2^\ell}{(m+1)2^{\ell + k}} \\
    &\in    \left( \sum_{\ell \in [j-1, j+2]}\frac{\Pr_{\beta}[\tfrac{\beta}{2} < p_D(x)2^{\ell} \leq \beta] }{(m+1)2^{\ell + k}}\right) \pm \frac{O(\nu)p_D(x)}{(m+1)2^k}.
\end{align*}
Observing that $\Pr_{\beta}[\tfrac{\beta}{2} < p_D(x)2^{\ell} \leq \beta] = 0$ except for $\ell \in \{j, j+1\}$, we can simplify the series:
\begin{align*}
    \sum_{\ell \in [j-1, j+2]}\frac{\Pr_{\beta}[\tfrac{\beta}{2} < p_D(x)2^{\ell} \leq \beta] }{(m+1)2^{\ell + k}} 
    & =    \sum_{\ell \in [j, j+1]}\frac{\Pr_{\beta}[\tfrac{\beta}{2} < p_D(x)2^{\ell} \leq \beta] }{(m+1)2^{\ell + k}} \\
    &= \frac{\Pr_{\beta}[\tfrac{\beta}{2} < p_D(x)2^{j} \leq \beta] }{(m+1)2^{j + k}} + \frac{\Pr_{\beta}[\tfrac{\beta}{2} < p_D(x)2^{j+1} \leq \beta] }{(m+1)2^{j + k + 1}} \\
    &= \frac{\Pr_{\beta}[\tfrac{\beta}{2} < \gamma \leq \beta] }{(m+1)2^{j + k}} + \frac{\Pr_{\beta}[\tfrac{\beta}{2} < 2\gamma \leq \beta] } {(m+1)2^{j + k + 1}} \\
    &= \frac{\Pr_{\beta}[\beta < 2\gamma ] }{(m+1)2^{j + k}} + \frac{\Pr_{\beta}[\beta \geq 2\gamma ] } {(m+1)2^{j + k + 1}} \\
    &= \frac{2\gamma - 1}{(m+1)2^{j + k}} + \frac{2 - 2\gamma } {(m+1)2^{j + k + 1}} \\
    &= \frac{2\gamma - 1 + 1 - \gamma}{(m+1)2^{j + k}} \\
%    &= \frac{\Pr_{\beta}[\beta < p_D(x)2^{j+1} ] }{(m+1)2^{j + k}} + \frac{\Pr_{\beta}[ \beta \geq p_D(x)2^{j+1}] }{(m+1)2^{j + k + 1}} \\
%    &= \frac{p_D(x)2^{j+1} - 1}{(m+1)2^{j + k}} + \frac{2 - p_D(x)2^{j+1} } {(m+1)2^{j + k + 1}} \\
%   &= \frac{p_D(x)2^{j+1} - 1}{(m+1)2^{j + k}} + \frac{1 - p_D(x)2^{j} } {(m+1)2^{j + k }} \\
    &= \frac{p_D(x)}{(m+1)2^k}.
\end{align*}
Plugging back into the probability that $\elementfinder$ returns $x$, we have
\begin{align*}
    \Pr_{H_1, u, \ell, \beta}[\elementfinder_{C, \nu, \ell,\beta}(H_1, u) = x]
    &\in  \frac{p_D(x)(1+O(\nu))}{(m+1)2^k}. 
\end{align*}
\end{proof}

Finally, we can show that the output distribution of $\correlatedsampler$ and that of circuit $C$ are $O(\nu)$-close in variation distance. 

\begin{proof}
Proof of Lemma ~\ref{lem:correlated-sampler-distributional-accuracy}:

From Lemma~\ref{lem:correlated-sampler-round-return-prob}, we have that in each round, $x$ is returned by $\elementfinder$ with probability 
$\frac{p_D(x)(1 \pm O(\nu))}{(m+1)2^{k}}$. 

Summing over all $x$, the probability that $\correlatedsampler$ terminates in any individual round is in the range
$
\frac{1 \pm O(\nu)}{(m+1)2^k}
$.
So, conditioned on a round of $\correlatedsampler$ returning, the algorithm returns $x$ with probability in $(1 \pm O(\nu)) p_D(x)$. Finally, $\correlatedsampler$ does not return by the final $T_1$'th round with probability at most $O(\nu)$, by the choice of $T_1$. Altogether, this implies $\dtv(D_C, D_\correlatedsampler) \in O(\nu)$ as desired. 
\end{proof}

\subsubsection{Analysis --- Correlated Sampling.}\label{sssec:corrsamp-corr-analysis}

Next, we show that $\correlatedsampler$ satisfies the correlated sampling requirement of Definition~\ref{def:implicit-correlated-sampling-problem}. To simplify notation in this subsection, we will denote the probability of an element $x$ under $D_{C_1}$ and $D_{C_2}$ by $p_1(x)$ and $p_2(x)$ respectively.

\begin{lemma}[Correlated Sampling of $\correlatedsampler$]
    \label{lem:correlated-sampler-correlations}
For all pairs of circuits $C_1, C_2: \{0,1\}^m \rightarrow \{0,1\}^n$, assuming the ideal conditions hold for all rounds of $\correlatedsampler(C_1, \nu; r)$ and $\correlatedsampler(C_2, \nu; r)$, 
    $$\Pr_{r}[\correlatedsampler(C_1, \nu; r) \ne \correlatedsampler(C_2, \nu; r)] \leq O(\dtv(D_{C_1}, D_{C_2}) + \nu).$$
\end{lemma}

\begin{proof}
Let $E_0$ denote the event that $\correlatedsampler(C_1, \nu; r) \neq \correlatedsampler(C_2, \nu; r)$, and for simplicity of notation, shorten $\elementfinder$ to $\elementfindershort$, and write $\elementfindershort^{(i)}_{C_j}$ to denote the output of $\elementfinder$ in the $i$th round. We start by making some simplifying assumptions. First, observe that since the probability $\correlatedsampler$ returns $\bot$ is always at most $O(\nu)$, we can condition on the fact that both $\correlatedsampler(C_2, \nu; r)\neq \bot$ and $\correlatedsampler(C_1, \nu; r)\neq \bot$ without loss of generality. Second, we can assume by symmetry that $\correlatedsampler(C_1, \nu; r)$ does not return before $\correlatedsampler(C_2, \nu; r)$ (else relabel $C_1$ as $C_2$).

With this in mind, let $E_i$ denote the event that $\correlatedsampler(C_2, \nu; r)$ returns in round $i$. The event $E_0$ can then be bounded by
\[
\Pr[E_0] \leq \sum\limits_i\Pr[E_i]\Pr_{\beta,\ell}[\elementfindershort^{(i)}_{C_1} \neq \elementfindershort^{(i)}_{C_2} \mid E_i].
\]

% by first, w.l.o.g., conditioning on $\correlatedsampler(C_2, \nu; r)$ returning. $\correlatedsampler$ outputs $\bot$ with probability at most $\bot$, so we further condition on $\correlatedsampler(C_2, \nu; r) = x \neq \bot$. 
% Then $E_0$ occurs with probability no greater than the probability $\elementfindershort^{(i)}_{C_1, \nu, \ell, \beta}(H_1, u ; r) \neq \elementfindershort^{(i)}_{C_2, \nu, \ell, \beta}(H_1, u ; r)$ in the round that $\correlatedsampler(C_2, \nu; r)$ returns.

To bound the probability of this inner event, observe that under our assumptions, this occurs exactly when $x = \elementfindershort^{(i)}_{C_2}$, but for the empirical estimate computed by $\elementfindershort^{(i)}_{C_1}$, $\widehat{q}_1(x) \not\in [\frac{\beta}{2} 2^{-k},\beta 2^{-k}]$. We will show conditioned on any value of $x$ and $\ell$, this occurs with probability at most:
\begin{equation}\label{eq:conditional-return}
\Pr_{\beta}[\elementfindershort^{(i)}_{C_1} \neq x \mid E_i,\elementfindershort^{(i)}_{C_2} = x,\ell] \leq O(2^{\ell}|p_2(x) - p_1(x)| + \nu).
\end{equation}
In this case, we can bound $E_0$ by conditioning further on $x$ and $\ell$ as:
\begin{align*}
\Pr_{r}[E_0] &\leq \sum\limits_{i \in [T_1]}\Pr_{\beta,\ell}[E_i]\Pr_{\beta,\ell}[\elementfindershort^{(i)}_{C_1} \neq \elementfindershort^{(i)}_{C_2} \mid E_i]\\
&\leq \sum\limits_{i \in [T_1]}\Pr_{\beta,\ell}[E_i]\sum_{x \neq \bot}\Pr_{\beta,\ell}[\elementfindershort^{(i)}_{C_1}\neq \elementfindershort^{(i)}_{C_2}\mid E_i, \elementfindershort^{(i)}_{C_2}=x]\Pr_{\beta,\ell}[\elementfindershort^{(i)}_{C_2}=x | E_i]\\
&\leq O(\nu) + \sum\limits_{i \in [T_1]}\Pr_{\beta,\ell}[E_i]\sum_{x \neq \bot}\Pr_{\beta,\ell}[\elementfindershort^{(i)}_{C_1}\neq \elementfindershort^{(i)}_{C_2}\mid E_i, \elementfindershort^{(i)}_{C_2}=x]p_2(x)\\
&\leq O(\nu) + \sum\limits_{i \in [T_1]}\Pr_{\beta,\ell}[E_i]\sum_{x \neq \bot}p_2(x)\sum\limits_\ell \Pr[\ell \mid E_i, \elementfindershort^{(i)}_{C_2}=x]\Pr_{\beta}[\elementfindershort^{(i)}_{C_1}\neq \elementfindershort^{(i)}_{C_2}\mid E_i, \elementfindershort^{(i)}_{C_2}=x,\ell]\\
&\leq O(\nu) + \sum\limits_{i \in [T_1]}\Pr_{\beta,\ell}[E_i]\sum_{x \neq \bot}p_2(x)\sum\limits_\ell \Pr[\ell \mid E_i, \elementfindershort^{(i)}_{C_2}=x]O(2^\ell|p_2(x)-p_1(x)|).
\end{align*}
Under the ideal conditions, the posterior of $\ell$ is $0$ unless $p_2(x) \in [2^{-\ell-2},2^{-\ell+2}]$, so altogether:
\begin{align*}
\Pr_{r}[E_0] &\leq O(\nu) + \sum\limits_{i \in [T_1]}\Pr_{\beta,\ell}[E_i]\sum_{x \neq \bot}p_2(x)\sum\limits_\ell \Pr[\ell \mid E_i, \elementfindershort^{(i)}_{C_2}=x]O\left(\frac{1}{p_2(x)}|p_2(x)-p_1(x)|\right)\\
&\leq O(\nu)+\sum\limits_{i \in [T_1]}\Pr_{\beta,\ell}[E_i]\sum_{x \neq \bot}O(|p_2(x)-p_1(x)|)\\
&\leq O(d_{TV}(D_{C_1},D_{C_2}) + \nu)
\end{align*}
as desired.
% where the last line follows by the ideal condition that $\elementfindershort^{(i)}_{C_2}$ only returns $x$ in the range $p_2(x) \in [2^{-\ell-2},2^{\ell+2}]$

% In this case, since the probability any individual $x$ is returned by $\elementfinder$ is at most $\Pr[x \in H^{-1}(u)]=2^{\ell+k}$, the total failure probability is at most \maxh{The conditioning here is wrong. I think we should really  be conditioning over the probability we return in a given round, given choice of $\ell$, which then will scale the $2^\ell$ correctly? I'll fix this in the morning.}
% \begin{align*}
%     \Pr_{\beta}[\elementfinder_{C_1, \nu, \ell, \beta}(H_1, u;r) \neq \elementfinder_{C_2, \nu, \ell, \beta}(H_1, u;r)]
%     & \leq \frac{\sum_x  O(2^{\ell}|p_2(x) - p_1(x)| + \nu)}{2^{\ell + k}} \\
%     &\leq O(d_{TV}(D_{C_1}, D_{C_2}) + \nu)
% \end{align*}
% as desired. 

It is therefore left to prove \Cref{eq:conditional-return}, which we analyze the probability by splitting into two cases based on $p_1(x)$:
\begin{enumerate}
    \item $p_2(x)/4 \leq p_1(x) \leq 4p_2(x)$
    \item $p_1(x) < p_2(x)/4$ or $p_1(x) > 4p_2(x)$
\end{enumerate}

\paragraph{Case 1: } In this case, because we are assuming the ideal conditions of Definition~\ref{def:ideal-cond} and have conditioned on $\elementfinder_{C_2, \nu, \ell, \beta}(H_1, u;r) = x$ for some $x$ in this round, it must be the case that for this $x$ we have:
\begin{itemize}
    \item $p_2(x) \in [2^{-\ell-2}, 2^{-\ell + 2}]$
    \item $p_1(x) \in [2^{-\ell-4}, 2^{-\ell + 4}]$ (from our assumption that $p_2(x)/4 \leq p_1(x) \leq 4p_2(x)$)
    \item $\widehat{q}_2(x) \in (1 \pm O(\nu))p_2(x)2^{\ell - k}$ 
    \item $\widehat{q}_1(x) \in (1 \pm O(\nu))p_1(x)2^{\ell - k}$.
\end{itemize} 
From the uniqueness of $x$ satisfying $x = H_1^{-1}(u)$ and $p_1(x) \in [2^{\ell - 4}, 2^{-\ell + 4}]$, we can assume that $\elementfinder_{C_1, \nu, \ell, \beta}(H_1, u; r)$ outputs either $x$ or $\bot$ in this round. Therefore, we can bound the probability $\elementfinder_{C_1, \nu, \ell, \beta}(H_1, u; r)\neq\elementfinder_{C_2, \nu, \ell, \beta}(H_1, u; r)$ by the probability that $\elementfinder_{C_1, \nu, \ell, \beta}(H_1, u; r)$ outputs $\bot$ conditioned on $\elementfinder_{C_2, \nu, \ell, \beta}(H_1, u;r)$ returning $x$. This occurs whenever $\beta$ is chosen so that either
%$$ p_1(x)2^{\ell -k } - O(\nu)2^{-k} < \frac{\beta2^{-k}}{2} \leq p_2(x)2^{\ell - k}(1 + O(\nu))$$ or
%$$ p_2(x)2^{\ell - k}(1 - O(\nu)) \leq \beta 2^{-k} < p_1(x)2^{\ell -k} + O(\nu)2^{-k}.$$
$$ (1+O(\nu))p_1(x)2^{\ell -k } < \frac{\beta2^{-k}}{2} \leq p_2(x)2^{\ell - k}(1 + O(\nu))$$ or
$$ p_2(x)2^{\ell - k}(1 - O(\nu)) \leq \beta 2^{-k} < p_1(x)2^{\ell -k}(1 + O(\nu)).$$

Observing that these cases are mutually exclusive and the first interval is the largest, we consider only that worst case and rearrange to obtain the condition
\begin{align*}
\beta 
&\leq 2^{\ell + 1}|p_2(x) - p_1(x)| + O(\nu)(p_2(x)+ p_1(x))2^{\ell} \\
&\leq 2^{\ell + 1}|p_2(x) - p_1(x)| + O(\nu),
\end{align*}
where the last inequality follows from our previously stated bounds $p_1, p_2 \in [2^{-\ell - 4}, 2^{-\ell + 4}]$.

Since $\beta$ is chosen uniformly at random from the interval $[1,2]$, it follows that $\beta$ satisfies the condition above with probability no greater than $2^{\ell + 1}|p_2(x) - p_1(x)| + O(\nu)$. Conditioning on $\elementfinder_{C_2, \nu, \ell, \beta}(H_1, u;r) = x$ and $p_2(x) \in [2^{-\ell - 4}, 2^{-\ell + 4}]$, we have
\begin{align*}
\Pr_{\beta}[ \elementfinder_{C_1, \nu, \ell\, \beta}(H_1, u;r) = \bot ] 
& \leq \Pr_{\beta}[\beta \leq 2^{\ell + 1}|p_2(x_2) - p_1(x_2)| + O(\nu)] \\
& \leq 2^{\ell + 1}|p_2(x) - p_1(x)| + O(\nu).
\end{align*}

\paragraph{Case 2: }In this case, we have either $p_1(x) < p_2(x)/4$ or $p_1(x) > 4p_2(x)$, and so $p_2(x) \in O(|p_2(x) - p_1(x)|)$. Conditioning on $\elementfinder_{C_2, \nu, \ell, \beta}(H_1, u;r) = x$ in this round also gives us that $p_2(x) = \Theta(2^{-\ell})$, so
\begin{align*}
    \Pr_{\beta}[\elementfinder_{C_1, \nu, \ell, \beta}(H_1, u;r) \neq x \mid \elementfinder_{C_2, \nu, \ell, \beta}(H_1, u;r) = x ] 
    & \leq 1 \\
    & \leq O(2^{\ell}p_2(x_2)) \\
    & \leq O(2^{\ell}|p_2(x_2) - p_1(x_2)|).
\end{align*}
Then for any value of $p_1(x)$ (either in Case 1 or Case 2) we have that 
$$ \Pr_{\beta}[\elementfinder_{C_1, \nu, \ell, \beta}(H_1, u;r) \neq x \mid \elementfinder_{C_2, \nu, \ell, \beta}(H_1, u;r) = x] \in O(2^{\ell}|p_2(x) - p_1(x)| + \nu) $$
as desired.
% Summing over all $x$ and accounting for the probability of $x = H_1^{-1}(u)$ we have that, for the round in which $\correlatedsampler(C_2, \nu)$ returns, 
% \begin{align*}
%     \Pr_{\beta}[\elementfinder_{C_1, \nu, \ell, \beta}(H_1, u;r) \neq \elementfinder_{C_2, \nu, \ell, \beta}(H_1, u;r)]
%     & \leq \frac{\sum_x  O(2^{\ell}|p_2(x) - p_1(x)| + \nu)}{2^{\ell + k}} \\
%     &\leq O(d_{TV}(D_{C_1}, D_{C_2}) + \nu).
% \end{align*}
% Then, removing our assumption that it is always the run of $\correlatedsampler$ corresponding to $C_2$ that returns, we have that under the ideal conditions, 
%     $$\Pr_{r}[\correlatedsampler(C_1, \nu; r) \ne \correlatedsampler(C_2, \nu; r)] \leq O(\dtv(D_{C_1}, D_{C_2})+\nu)$$
% as desired.
\end{proof}

%-------------------------
% Runtime
%-------------------------
\subsubsection{Analysis --- Runtime}\label{sssec:corrsamp-runtime}

\begin{lemma}[Runtime of $\correlatedsampler$]
\label{lem:correlated-sampler-runtime}
Assuming inverter $\Ical$ runs in time polynomial in $m$, $n$, and $1/\nu$, algorithm $\correlatedsampler$ also runs in time polynomial in $m$, $n$, and $1/\nu$. 
\end{lemma}

\begin{proof}

Algorithm $\correlatedsampler$ runs for at most $O(m k 2^k \log (1/\nu))$ rounds. Parameter $k$ is chosen in $\Theta(\log(m) + \log(1/\nu))$, so $2^k \in O(\poly(m, 1/\nu))$. Randomly sampling a pairwise-independent hash function from $n$ bits to $\ell +k$ bits can be done in $\poly(m, n, 1/\nu)$ time. 

Each round of $\correlatedsampler$ contains a call to $\elementfinder$. In $\elementfinder$, there are $O(\poly(m, 1/\nu))$ rounds in which a hash function from $m$ bits to $m-\ell + k$ bits is randomly sampled ($O(\poly(m, 1/\nu))$ time). Furthermore, each round contains a call to $\hashchecker$, which runs in time $\poly(m,n, 1/\nu)$ by the assumption that inverter $\Ical_{\nu'}$ does as well. 

Multiplying these nesting terms together, $\correlatedsampler$ runs in time polynomial in $m$, $n$, and $\nu$. 

Note that the randomness management, which ensures that the same bits of the random string $r$ are always used across multiple executions of $\correlatedsampler$, can also be done in time polynomial in $m$, $n$, and $1/\nu$. Each algorithm and subroutine has a finite number of randomness calls, and each call can be made using $\poly(m,n,1/\nu)$ random bits. Thus, $r$ can efficiently be canonically proportioned for all uses of randomness in the algorithm. For more details, see Appendix~\ref{apps:randomness-management}.

Parameter $\beta$ is chosen uniformly randomly in $[1,2]$ in each loop of $\correlatedsampler$. Only polynomial in $m$, $n$, and $1/\nu$ bits of precision are needed to choose $\beta$ so that the errors introduced by not using uniformly random values are altogether small relative to distributional error parameter $\nu$. 
\end{proof}

\subsubsection{Analysis -- Removing Assumption of Ideal Conditions}\label{sssec:corrsamp-ideal-proof}
\begin{proposition}\label{lem:corrsamp-conditions}
The ideal conditions of Definition~\ref{def:ideal-cond} hold across all steps of $\correlatedsampler$ with probability at least $1-O(\nu)$.
\end{proposition}
We break the proof into its four constituent part.
\begin{lemma}[Inverter Never Fails]
Let $\mathcal{S}$ denote the set of strings $(x;r)$ on which $\correlatedsampler$ invokes $\mathcal{I}_{\nu'}(x;r)$. The probability the inverter fails on $\mathcal{S}$ is negligible:
\[
\Pr[\exists (x;r) \in \mathcal{S}: \mathcal{I}_{\nu'}(x;r) = \bot] \leq O(\nu),
\]
as long as $\nu' = \poly(\nu^{-1},m,n)$ is sufficiently small.
\end{lemma}
\begin{proof}
Recall our inverter has the following guarantee
\[
\Pr_{r' \sim \{0,1\}^m} [C(\Ical_{\nu'}(C, C(r')); r)=C(r') ] \ge 1 - \nu',
\]
where $r$ stands for the internal randomness of $\Ical$ and $\nu'$ can be taken to be polynomially small in $m,n$ and $\nu$. It will be enough to take the failure rate $\nu' \leq O(\frac{\nu^2}{T_1T_2}) = \poly(\nu^{-1},m,n)$. We will bound the probability such an inverter fails on a random input $(u || v)$. 

First, observe that by Markov's inequality, most choices of internal randomness $r$ for the inverter random work for almost all $r' \in \{0,1\}^m$:
\[
\Pr_r \left[ \Pr_{r'}\left[\Ical_{\nu'}(C, C(r')); r)\neq C(r')\right] \geq O\left(\frac{\nu}{T_1T_2}\right)\right] \leq O(\nu).
\]
Assume then this event does not occur. Our algorithm only fails if the random choice of $(u || v)$ is hashed to by a string for which $\Ical_{\nu'}(C, C(r')); r)\neq C(r')$. In the worst case, these bad strings each correspond to unique $(u || v) \in \{0,1\}^{m+2k}$, in which case we have a total of $\frac{\nu}{T_1T_2}2^m$ out of $2^{m+2k}$ bad inputs. Union bounding over the $T_1T_2$ applications of the inverter in $\correlatedsampler$ gives a total failure probability of $1-O(\nu)$ as desired.
\end{proof}
To prove the remaining conditions, it will be useful first to bound the number of collisions experienced by our hash functions.
\begin{claim}[Collision Avoidance]\label{claim:col}
With probability at least $1-O(\nu)$, for all choices of $H_1$, $u$, and $\ell$ in $\correlatedsampler$:
    \begin{enumerate}
        \item $H_1$ has no relevant collisions:
        \[
        \forall p_D(x),p_D(x') \in [2^{-\ell-4},2^{-\ell+4}]: H_1(x) \neq H_1(x'),
        \]
        \item For all $x \in H_1^{-1}(u)$, the total number of collisions across choices of $H_2$ is at most:
        \[
        \left|\{H_2, (r,r') \in C^{-1}(x): H_2(r)=H_2(r')\}\right| \leq \delta_xT_2|C^{-1}(x)|,
        \]
    \end{enumerate}
    where $\delta_x=O\left(2^\ell p_D(x) \cdot \frac{mk\log(1/\nu)}{\nu} \cdot 2^{-k}\right)$
\end{claim}
\begin{proof}
To prove the first condition, observe there are at most $2^{\ell+4}$ elements $x \in \{0,1\}^n$ with measure in the range $p_D(x) \in [2^{-\ell-4},2^{-\ell+4}]$. Since our hash function is pairwise-independent, the probability a collision exists in this range is therefore bounded by $\frac{2^{2\ell+8}}{2^{2\ell+2k}} = 2^{4-2k}$. Union bounding over the $T_1$ choices of $H_1$, $u$, and $\ell$, a collision still occurs with probability at most $2^{-\Omega(k)} \leq O(\nu)$ as desired.

To prove the second condition, fix $x \in H_1^{-1}(u)$ and observe that by pairwise independence, the expected number of total collisions across all choices of $H_2$ is at most
\[
T_2\frac{|C^{-1}(x)|^2}{2^{2m-2\ell+2k}}.
\]
by linearity of expectation, so by Markov's inequality the probability there are more than 
\[
T_2\frac{|C^{-1}(x)|^2}{2^{2m-2\ell+2k}}\cdot \frac{T_12^m}{\nu} \leq T_2|C^{-1}(x)| \cdot \frac{2^\ell p_D(x)mk\log(1/\nu)}{\nu2^{k}}
\]
total collisions is at most $O(\nu)/(T_12^m)$, so 
% we first need to bound the size of the (pre)-pre-image of $u$. In particular, we argue that with probability at least $1-O(\nu)$:
% \[
% |C^{-1}H_1^{-1}(u)| \leq O\left(2^{m-\ell} \cdot \frac{mk\log(1/\nu)}{\nu}\right)
% \]
% To see this, observe that the expected amount of strings that hash to $u$ (through $C$) is $2^{m-\ell-k}$. Then by Markov the probability that $O(\frac{mk2^{m-\ell}\log(1/\nu)}{\nu})$ strings hash to $u$ is at most $O(\frac{\nu}{mk2^k\log(1/\nu)})\leq O(\nu/T_1)$. 
union bounding over all choices of $H_1$, $u$, and $x$ gives the desired result.
% The expected number of collisions across all $H_2$ in the pre-image of $u$ is then at most:
% \[
% T_2\frac{|C^{-1}H_1^{-1}(u)|^2}{2^{2m-2\ell+2k}} \leq O\left(\frac{m^2k^2\log^2(1/\nu)}{\nu^2}2^{-2k}\right) \cdot T_2
% \]
% % Let $T=O(\frac{2^k}{\nu^2}\log(1/\nu))$ represent the number of $H_2$ drawn by $\elementfinder$, and observe that the above probability bound is at most $1/T$. 
% Markov's inequality then implies that the probability more than $\frac{T_1}{\nu}$ times this many collisions occur is at most $ O(\nu/(T_1T_2))$, so union bounding over all rounds of $\correlatedsampler$ the desired bound holds except with probability at most $O(\nu)$. 
\end{proof}
\begin{lemma}[Uniqueness of $y^*$]\label{lem:unique}
With probability at least $1-O(\nu)$ over all choices of $H_1$ and $u$, there is at most one element $y^* \in H_1^{-1}(u)$ satisfying
    \[
    p_D(y^*) \in [2^{-\ell-4},2^{-\ell+4}].
    \]
\end{lemma}
\begin{proof}
    This is immediate from the fact that $H_1$ has no collisions on $[2^{-\ell-4},2^{-\ell+4}]$ with high probability.
\end{proof}
\begin{lemma}[Correctness of $\hat{q}$]\label{lem:q-correct}
With probability at least $1-O(\nu)$, every run of $\elementfinder$ accurately estimates $p_D(y^*)$ in the following sense:
    \[
\widehat{q}(y^*) \in \left((1\pm O(\nu))p_D(x)2^{\ell-k}\right)
    \]
\end{lemma}
\begin{proof}
For intuition, first consider the setting where no collisions occur in any $H_2$. In this case, observe that the density of the pre-image of $x$ mapped into the range of $H_2$ is exactly $p_D(x)2^{\ell-k}$ by construction, that is:
\[
\frac{|H_2(C^{-1}(y^*))|}{2^{m-\ell+k}} = p_D(x)2^{\ell-k}.
\]
As such $\hat{q}(y^*)$ is distributed as a Binomial distribution $Bin(p_D(y^*)2^{\ell-k},T_2)$, and Chernoff promises that
\[
\Pr_{v,H_2}[|\hat{q}(x) - p_D(x)2^{\ell-k}| \geq O(\nu)p_D(y^*)2^{\ell-k}] \leq e^{-\Omega(\nu^2T_2)} \leq \frac{O(\nu)}{T_1}
\]
by our choice of $T_2$. The desired result then follows from union bounding over all choices of $H_1$ and $u$.

We now modify this analysis under the assumption that at most 
\[
C_{col} \coloneqq O\left(\frac{2^\ell p_D(y^*)mk\log(1/\nu)}{\nu2^{k}}\right)\cdot |C^{-1}(y^*)|T_2
\]
total collisions occur, which holds across all rounds except with probability $O(\nu)$ (\Cref{claim:col}). In this case, since $p_D(y^*) \leq 2^{\ell+4}$, we have 
\[
C_{col} \leq O(\nu)|C^{-1}(y^*)|T_2
\]
for $k=\Theta(\log(m/\nu))$ sufficiently large, and therefore that the expectation of our adjusted binomial trial is close enough to its ideal expectation:
\[
\sum_{H_2} \frac{|H_2(C^{-1}(y^*))|}{2^{m-\ell+k}} \in T_2 \cdot\left[(1-O(\nu))p_D(y^*)2^{\ell-k}, p_D(y^*)2^{\ell-k}\right],
\]
that the collisions have no asymptotic effect on the original Chernoff bound.
% In the case of $y^*$, we have that $p_D(y^*) \in [2^{-\ell-2}, 2^{-\ell-2}]$, and so 
% \[ 
% \widehat{q}(x) \in \left(1\pm O\left(\nu\right)\right)p_D(y^*)2^{\ell-k}.
% \]
\end{proof}
\begin{lemma}[Correctness of $\elementfinder$]
With probability at least $1-O(\nu)$, all calls to $\elementfinder$ return $y^*$ or `$\bot$'. Furthermore, for $y^*$ returned by $\elementfinder$, we have the stronger condition that ${p_D(y^*) \in [2^{-\ell - 2}, 2^{-\ell+2}]}$.
\end{lemma}
\begin{proof}
For the first claim, it is enough to argue that any $x\in H_1^{-1}(u)$ distinct from $y^*$ satisfies:
    \[
    \hat{q}(x) \notin [(\beta/2)2^{-k},\beta2^{-k}].
    \]
    
% since when $y^*$ exists, $\hat{q}(y)$ is in the desired range by \Cref{lem:q-correct} and therefore returned by $\elementfinder$ (except with probability $O(\nu)$).
% Assuming uniqueness of $p_D(y^*) \in [2^{\ell-2},2^{\ell+2}]$ (\Cref{lem:unique}), note that all such $x$ either satisfy $p_D(x) <2^{-\ell-2}$ or $p_D(x) >2^{-\ell+2}$. Assuming correctness of $\hat{q}$ (\Cref{lem:q-correct}), this means either
% \[
% \hat{q}(x) \leq \frac{2^{-k}}{4}+O(\nu)2^{-k} < (\beta/2)2^{-k},
% \]
% or that 
% \[
% \hat{q}(x) \geq 4\cdot2^{-k} - O(\nu)2^{-k} > \beta2^{-k}
% \]
% for the appropriate choice of constants. Since both assumptions hold across all calls to $\elementfinder$ except with probability $O(\nu)$, we are done.
Assuming $y^*$ is unique (\Cref{lem:unique}), we have either that $p_D(x) < 2^{-\ell-4}$ or $p_D(x) > 2^{-\ell+4}$. Consider the former. We will show $\hat{q}(x) < 2^{-k-1} \leq (\beta/2)2^{-k}$. Note that since collisions only lower $\hat{q}$, they can be ignored in this setting. By construction, the pre-image of $x$ consists of at most $2^{m-\ell-4}$ strings $r'$ map to $x$, so $|H_2(C^{-1}(x))|$ is at most a
\[
\frac{2^{m-\ell-4}}{2^{m-\ell+k}} = \frac{2^{-k}}{16}
\]
fraction of the range of $H_2$. A Chernoff and Union bound give that all such $x$ have empirical estimates less than $(\beta/2)2^{-k}$ except with probability $2^me^{-\Omega(2^{-k}T_2)} \leq O(\nu)$.

% with probability $2^{\Omega(m-\nu^2)} \ll O(\nu/T_1)$  

Finally, consider $x$ satisfying $p_D(x) \geq 2^{-\ell+4}$. In this case, we will aim show $\hat{q}(x)>\frac{2}{2^k}\geq \beta2^{-k}$, so it is sufficient to consider the worst case when $p_D(x)=2^{-\ell+4}$. By \Cref{claim:col}, we can assume there are at most
% In this case, we need to take collisions into account, and actually require a slightly stronger bound than in \Cref{claim:col}. In particular, first note that by the same argument as \Cref{lem:unique}, there are at most $m$ such elements in $H^{-1}(u)$, so we can in turn assume there are at most 
% \[
% C_{col} \leq m\frac{T_1T_2}{\nu} \cdot\frac{|C^{-1}(x)|^2}{2^{2m-2\ell+2k}} 
% \]
% collisions (by the same analysis as \Cref{claim:col}). Note that since we need to ensure $\hat{q}(x) > \frac{2}{2^k}$ in this setting, it is enough to prove the result when $p_D(x)=2^{-\ell+2}$, then:
\[
C_{col} \leq \frac{1}{4}T_2|C^{-1}(x)| 
\]
total collisions over the choices of $H_2$, thus as in \Cref{lem:q-correct}, the collision-corrected Chernoff bound still promises $\hat{q}(x)>\frac{2}{2^k}$ with probability at least $e^{-\Omega(2^{-k}T_2)} \leq \frac{O(\nu)}{2^mT_1}$. Union bounding over all choices of $H_1$, $u$, and values of $x$ completes the proof of the first claim.

To prove the second claim, we observe that Lemma~\ref{lem:q-correct} gives us $\widehat{q}(y^*) \in (1 \pm O(\nu))p_D(y^*)2^{\ell - k}$.
Bounding $(1\pm O(\nu)) \in (1/2, 2)$, we have 
$$ p_D(y^*) \in (\widehat{q}(y^*)2^{k - \ell - 1}, \widehat{q}(y^*)2^{k-\ell + 1}).$$
If $\elementfinder$ returned $y^*$, it must be the case that $$\widehat{q}(y^*) \in (\tfrac{\beta}{2}2^{-k}, \beta 2^{-k}] \in (2^{-k-1}, 2^{-k+1}],$$
and so
$$p_D(y^*) \in [2^{-\ell - 2}, 2^{-\ell + 2}],$$
conditioned on $\elementfinder$ returning $y^*$, as claimed.

% When no collisions occur, the empirical density of $x$ is at least $\frac{3}{2^k}$ with probability $1-O(\nu)$ by. Taking into account the $C_{col}$ collisions the high probability estimate becomes at worst $\frac{\frac{3}{2^k}T_2-C_{col}}{T_2}>\frac{2}{2^k}$ as desired (since for large enough $k$ we have that $C_{col} \leq 2^{-k}T_2$).
\end{proof}

\section{Separating Stability: Statistical Barriers}\label{sec:stat-sep}
\subsection{Quadratic Separation: One-way Marginals}
\label{sec:marginals}
We start by defining the one-way marginals problem over $d$ coordinates, which corresponds to outputting a good estimate of the expectation of a product of Rademacher distributions in $\ell_{\infty}$-distance.
\begin{definition}
Consider a product of $d$ Rademacher distributions with expectations $p = (p_1,\dots,p_d)$ respectively. A vector $v \in \mathbb{R}^d$ is said to be an $\alpha$-accurate solution to the one-way marginals problem if $\|v - p\|_{\infty} \leq \alpha$.
\end{definition}
\begin{definition}
Let $C$ be the class of products of $d$ Rademacher distributions. Fix any distribution $D$ in $C$. We say that an algorithm $(\alpha,\beta)$-accurately solves the one-way marginals problem over $d$ coordinates, if it observes samples from the distribution $D$, and with probability at least $1-\beta$ (over the randomness of the samples and the algorithm), produces an $\alpha$-accurate solution $v \in \mathbb{R}^d$.
\end{definition}
In this section, we show that any $0.0001$-replicable, $(0.01,0.01)$-accurate algorithm for the one-way marginals problem over $d$ coordinates requires at least $\widetilde{Omega}(d)$ samples.

On the other hand, under the constraint of $(1, \frac{1}{n^2})$-differential privacy, this problem can be solved using $\widetilde{O}(\sqrt{d})$ samples, via the Gaussian mechanism. This gives a quadratic separation between differential privacy and replicability, and proves that our reduction is asymptotically tight (up to logarithmic factors) in some settings (since our reduction would give a $\tilde{O}(d)$-sample replicable algorithm for this task). 

The main theorem we prove in this section is the following.
\begin{theorem}\label{thm:repmarginalsLB}
Fix sufficiently large $d>0$. For any  $0.0001$-replicable algorithm $\mathcal{A}$ that $(0.01,0.001)$-accurately solves the one-way marginals problem over $d$ coordinates with $m$ samples, $$m=\tilde{\Omega}(d).$$
\end{theorem}
\subsubsection{Sketch of our approach}
Our techniques for proving the lower bound for replicability draw inspiration from those used to prove lower bounds in privacy. Specifically, tight lower bounds for the one-way marginals problem over $d$ coordinates under the constraint of differential privacy are obtained using the \textit{fingerprinting method} \cite{DworkSSUV15, BunUV18, BunSU19}. The fingerprinting method captures the idea that there is a trade-off between accuracy and correlation with the input sample. It quantifies the idea that if the algorithm obtains a sample of small size, and is also very accurate, then it must be heavily correlated with one of its input examples, which is prohibited by differential privacy.
% Our intuition was that since replicability also aims to capture the idea that an algorithm is not too correlated with its input sample, the same method might be useful in proving lower bounds for replicability. 
Since replicability also prohibits such correlation (at least at a high level), one might expect the same method to be useful toward this end.

More formally, given an algorithm $\mathcal{A}$ solving the one-way marginals problem, the \textit{correlation} of coordinate $j$ of the output with the input sample $S$ can be measured by the quantity 
\[
Z=\sum\limits_{j \in [d]} \mathcal{A}^j(S)\sum\limits_{i \in [m]} (S^j_i - p_j).
\] 
Note that the quantity $(S^j_i - p_j)$ represents the drift between an input example coordinate and the expectation of the distribution it's drawn from. $\mathbb{E}[Z]$ is large when on average, for many $j$, $\mathcal{A}^j(S)$ is on the same side of $0$ as the sample drift $\sum_i(S^j_i - p_j)$, implying that the algorithm's outputs are on average correlated with its input. %For sake of comparison, for a freshly drawn dataset $X'$, we consider $Z'=\sum_j \mathcal{A}^j(X)\sum_i (X'^j_i - p_j)$. We'd expect that $\mathbb{E}[Z']$ is small since $X$ and $X'$ are independent. 

We now recall the formal statement of the fingerprinting lemma.

\begin{lemma}[Fingerprinting Lemma, Lemma 3.6 in \cite{BunSU19}]
\label{lem:fingerprinting}
Let $f$ be any function from $\{-1,1\}^m \to [-1,1]$. Suppose $r$ is sampled from the uniform distribution over $[-1,1]$ and $q \in \{-1,1\}^m$ is a vector of $m$ independent Rademacher RVs each with expectation $r$. Then, if $\mu_q$ is the empirical average of $q$, we get that
$$\mathbb{E}_{r,q}[f(q)\sum_i (q_i - r) + 2|f(q) - \mu_q|] \geq \frac{1}{3}.$$
\end{lemma}

The lower bound for differential privacy proceeds by arguing that $\mathbb{E}[Z] = \mathbb{E}[\sum_j \mathcal{A}^j(S)\sum_i (S^j_i - p_j)]$ is large (via an appropriate application of the fingerprinting lemma), and hence, by an averaging argument, there exists an example $S_{i^*} \in S$ correlated with the output, i.e.\ some $i^* \in [m]$ such that
\[
\mathbb{E}[Z_{i^*}] = \mathbb{E}[\sum_{j\in[d]} \mathcal{A}^j(S) (S^j_{i^*} - p_j)]
\] 
is large. With this in hand, consider an independently drawn example $g=(g^1_{i^*},\dots,g^d_{i^*})$, and the neighboring dataset $S'$ obtained by replacing $S_{i^*}$ in the original dataset with $g$. Since $g$ is independent of $S_{i^*}$, $\mathcal{A}^j(S')$ is uncorrelated with $(S^j_{i^*} - p_j)$, and hence the ``1-neighboring'' quantity
\[
\mathbb{E}[Z_{i^*}'] = \mathbb{E}[\sum_{j\in[d]} \mathcal{A}^j(S') (S^j_{i^*} - p_j)]
\]
% the quantity $Z_{i^*}' = \sum_j \mathcal{A}^j(S') (s^j_{i^*} - p_j)$, we'd have that $\mathcal{A}^j(S')$ is uncorrelated with $(s^j_{i^*} - p_j)$, and hence the quantity $\mathbb{E}[Z_{i^*}']$ is 
should be small. On the other hand, differential privacy promises that $Z_{i^*}$ is distributionally close to $Z'_{i^*}$, and hence $\mathbb{E}[Z'_{i^*}]$ and $\mathbb{E}[Z_{i^*}]$ must be close.
% can't be too far from each other (and $\mathbb{E}[Z_{i^*}]$ is large). 
Balancing these considerations gives a lower bound on the number of samples needed for differential privacy.

For replicability, the idea is to obtain a stronger lower bound by avoiding averaging.
% following a similar logic but not applying an averaging argument. 
Specifically, for $Z$ defined as above, we can argue that $\mathbb{E}[Z]$ is large (as we would for the differential privacy lower bound). Then, we can consider a freshly sampled dataset $S'$ (drawn from a product distribution with the same expectation $p=(p_1,\dots,p_d)$), and consider the quantity $Z'=\sum_j \mathcal{A}^j(S')\sum_i (S^j_i - p_j)$. We can argue that $Z'$ is a sum of uncorrelated random variables, and hence that $\mathbb{E}[Z']$ is small. On the other hand, by replicability, $Z$ and $Z'$ are distributionally close, since they correspond to post-processing of the algorithm applied to independent datasets. Note that this does not follow from differential privacy,
% would not give that $Z$ and $Z'$ are distributionally close,
since datasets $S$ and $S'$ may differ in many entries. Now, following a similar approach to the differential privacy lower bound, we'd get a stronger lower bound for replicability (since we have eliminated the averaging argument).

Unfortunately, this approach does not work directly for technical reasons. Specifically, $\rho$-replicability tells us that $Z$ and $Z'$ are distributionally close, but their expectations can have absolute value difference as large as $\rho dm$ (since $\mathcal{A}(S)$ and $\mathcal{A}(S')$ could differ completely with probability $\rho$). Since we are interested in constant $\rho$, this turns out to be too large for the lower bound technique to work.

We deal with this by instead applying the fingerprinting method to prove a lower bound against $(\frac{1}{m^3},1,\frac{1}{m^3})-$\textit{perfectly generalizing} algorithms. We find this lower bound interesting in its own right, as it gives the first sample complexity separation between approximate differential privacy and perfect generalization. Perfect generalization roughly asks that the algorithm's output distributions be $(1, \frac{1}{m^3})$-close on two independent datasets drawn from the product distribution. This can be used to argue that $\mathbb{E}[|Z|]$ is within $\frac{1}{m^2}$ of (a constant multiple of) $\mathbb{E}[|Z'|]$,
% (i.e. close in terms of absolute value difference to a constant multiple of $\mathbb{E}[Z']$, as opposed to $\mathbb{E}[Z']$ itself),
which turns out to be sufficient for the lower bound technique to apply.

Finally, appealing to our generic method of converting replicable algorithms to perfectly generalizing ones, this method extends to a tight lower bound on replicability (up to the loss of logarithmic factors in the number of coordinates $d$). 
% These logarithmic factors creep in due to the conversion of a reproducible algorithm to a perfectly generalizing one. 
It remains an interesting problem whether such a lower bound can be shown directly, ideally in a manner that avoids the resulting logarithmic loss.
% there is a more direct way of proving a lower bound on the number of samples needed by a replicable algorithm solving the one-way marginals problem (without going through perfect generalization), 
% that would hopefully avoid these logarithmic factors.
%Next, we define the notion of perfect generalization, a notion of distributional generalization defined in prior work on adaptive data analysis. \sstext{Add more exposition here.}

%\begin{definition}[\cite{CummingsLNRW16}]
%An algorithm $A: \mathcal{Y}^n \to \mathcal{R}$ is said to be $(\beta,\eps,\delta)$-perfectly generalizing, if for every distribution $P$ over $\mathcal{Y}$, there exists a distribution $SIM_P$ such that with probability at least $1-\beta$ over the draw of an i.i.d. sample $X \sim P^n$, $A(X) \approx_{\eps, \delta} SIM_P$.
%\end{definition}

%Cummings et al. \cite{CummingsLNRW16} show that perfect generalization implies the following.

%\begin{lemma}[\cite{CummingsLNRW16}]
%If algorithm $A: \mathcal{Y}^n \to \mathcal{R}$ is $(\beta,\eps,\delta)$-perfectly generalizing, then for every distribution $P$ over $\mathcal{Y}$, with probability $1-\beta$ over the draw of two i.i.d. samples $X_1,X_2 \sim P^n$, we have that $A(X_1) \approx_{2\eps,3\delta} A(X_2)$.
%\end{lemma}
\subsubsection{Formal argument}
We start by proving our new lower bound for perfectly generalizing algorithms. 

\begin{theorem}\label{thm:PGmarginalsLB}
Fix any $m > 0$ and sufficiently large $d>0$. Let $\mathcal{A}$ be a $(\frac{1}{m^3},1,\frac{1}{m^3})$-perfectly generalizing, $(0.01,0.01)$-accurate algorithm for the 1-way marginals problem over $d$ attributes using $m$ samples. Then, $m = \Omega(d)$. 
\end{theorem}
\begin{proof}
Assume without loss of generality that $m = \Omega(\log d)$.\footnote{We will show under this condition that $m = \Omega(d)$. Any algorithm on $O(\log d)$ samples implies one between $O(\log d)$ and $O(d)$, which would give a contradiction.} Let $p \sim [-1,1]^{d}$, and draw $S,S' \sim D_p^m$ independently where $D_p$ is a product of Rademachers with expectation $p = (p_1,\dots,p_d)$. Define the random variables
\[
Z = \sum_{j \in [d]} \mathcal{A}^j(S)\sum_{i\in[m]} (S^j_i - p_j), \quad Z' = \sum_{j \in [d]} \mathcal{A}^j(S)\sum_{i\in[m]} (S'^j_i - p_j).
\]
As discussed above, we will argue that $\mathbb{E}[Z]$ is large (by the fingerprinting lemma):
\begin{equation}\label{eq:fp1}
\mathbb{E}[|Z|] \geq \frac{d}{10},
\end{equation}
that $\mathbb{E}[|Z'|]$ is small (since $S'$ is independent of $S)$:
\begin{equation}\label{eq:fp2}
\mathbb{E}[|Z'|] \leq 2\sqrt{dm},
\end{equation}
and finally that $\mathbb{E}[|Z|]$ and $\mathbb{E}[|Z'|]$ are close (by perfect generalization):
\begin{equation}\label{eq:fp3}
\mathbb{E}[|Z|] \leq e^2\mathbb{E}[|Z'|] + \frac{8d}{m^2}.
\end{equation}
Combining the inequalities we get
$$\frac{d}{10} \leq \mathbb{E}[|Z|] \leq e^2 \mathbb{E}\Big[ |Z'| \Big] + \frac{8d}{m^2} \leq 2e^2 \sqrt{dm} + \frac{8d}{m^2}$$
which implies $m \geq \Omega(d)$ as desired.

It remains to show Inequalities \eqref{eq:fp1}, \eqref{eq:fp2}, \eqref{eq:fp3}. We start with the first. Apply the fingerprinting lemma to the function corresponding to the $j^{th}$ coordinate of the output of $\mathcal{A}$, when run on the $j^{th}$ column of the input $S$ with all other columns set to any fixed values. Then, we get that 
\begin{align*}
\mathbb{E}_{p_j \sim [-1,1],S^j \sim Rad(p_j)^m}[\mathcal{A}^j(S; r)\sum_i (S^j_i - p_j) + 2|\mathcal{A}^j(S; r) - \mu_j|] \geq \frac{1}{3},
\end{align*}
where $\mu_j$ is the empirical average of the $j^{th}$ column of the dataset. Since this is true for all fixed coins of the algorithm and fixed values of the other columns, by the law of total expectation, it is also true for random coin tosses and any distribution over the values of the other columns, and we get that 
\begin{align*}
\mathbb{E}_{p,S \sim D_p^m,\mathcal{A}}[\sum_j\mathcal{A}^j(S)\sum_i (S^j_i - p_j) + 2|\mathcal{A}^j(S) - \mu_j|] &= \mathbb{E}[Z] + 2\mathbb{E}[\sum_j|\mathcal{A}^j(S) - \mu_j|] \geq \frac{d}{3}
\end{align*}
where we have used that $S$ is drawn from a product distribution. It is therefore enough to argue that $\mathbb{E}[\sum_j|\mathcal{A}^j(S) - \mu_j|]$ is small.

% Fix any distribution $D_p$ that is a product of Rademachers with expectation $p = (p_1,\dots,p_d)$. 
% Let $S$ be drawn from $D_p^m$. 
% and that $\mathcal{A}$ outputs values between $[-1,1]$ (rounding inputs to this range only improves the accuracy and doesn't affect perfect generalization, which is robust to post-processing). 
Since $\mathcal{A}$ is $(0.01,0.01)$-accurate, we can say that with probability at least $0.99$, for all $j \in [d]$, $|\mathcal{A}^j(S) - p_j| \leq \frac{1}{100}$. Taking expectation, we get that 
$\mathbb{E}[\sum_{j \in [d]} |\mathcal{A}^j(S) - p_j|] \leq \frac{3d}{100}$. By a Chernoff bound, we can argue that since $m = \Omega(\log d)$, $\max_j |p_j - \mu_j| \leq 0.01$ with probability at least $0.99$. By the triangle inequality, this gives us that $\mathbb{E}[\sum_{j \in [d]} |\mathcal{A}^j(S) - \mu_j|] \leq \frac{6d}{100}$. Since this holds for a fixed product of Rademachers, it also holds when the expectation of the Rademacher random variables are chosen at random which proves \Cref{eq:fp1}.
%\sstext{Add in that coordinate means close to true means whp here.} 

% Next, conditioned on any set of fixed coin tosses $r$ of algorithm $\mathcal{A}$, apply the fingerprinting lemma to the function corresponding to the $j^{th}$ coordinate of the output of $\mathcal{A}$, when run on the $j^{th}$ column of the input, with the other columns set to any fixed values. Then, we get that 
% $$\mathbb{E}_{p_j \sim [-1,1],S^j \sim Rad(p_j)^m}[\mathcal{A}^j(S; r)\sum_i (s^j_i - p_j) + 2|\mathcal{A}^j(S; r) - \mu_j|] \geq \frac{1}{3}.$$ 
% Since this is true for all fixed coins of the algorithm and fixed values of the other columns, by the law of total expectation, it is also true for random coin tosses and any distribution over the values of the other columns, and we get that 
% \[
% \mathbb{E}_{p,S \sim D_p^m,\mathcal{A}}[\mathcal{A}^j(S)\sum_i (s^j_i - p_j) + 2|\mathcal{A}^j(S) - \mu_j|] \geq \frac{1}{3},
% \]
% where we have used that $S$ is drawn from a product distribution.
% Summing over all $j \in [d]$, we get that 
% \[
% \mathbb{E}_{p,S, \mathcal{A}}[\sum_{j \in [d]}\mathcal{A}^j(S)\sum_i (s^j_i - p_j) + 2|\mathcal{A}^j(S) - \mu_j|] \geq \frac{d}{3}.
% \]
% Using the implications of accuracy discussed in the first paragraph, we can simplify this to obtain that 
% $$\mathbb{E}[Z] = \mathbb{E}_{p,S,\mathcal{A}}[\sum_{j \in [d]} \mathcal{A}^j(S)\sum_i (s^j_i - p_j)] \geq \frac{d}{10}.$$
% Consider a freshly sampled dataset $S' \sim D_p^m$, and let $Z = \sum_{j \in [d]} \mathcal{A}^j(S)\sum_i (s^j_i - p_j)$ and $Z' = \sum_{j \in [d]} \mathcal{A}^j(S)\sum_i (s'^j_i - p_j)$.
Next, we show \Cref{eq:fp2}, that $\mathbb{E}[|Z'|]$ is small. Towards this end, first note that $Z'$ is a sum of mean $0$ uncorrelated random variables. To see this, consider random variables $M = \mathcal{A}^j(S) (S'^j_i - p_j) $ and $N = A^j(S) (S'^j_{i'} - p_j)$ for indices $i\neq i'$. We claim that $\mathbb{E}[M] = \mathbb{E}[N]=0$. This is by the following sequence of inequalities (we prove this for $M$, the same argument holds for $N$).
\begin{align*}
\mathbb{E}[M] & = \mathbb{E}_p \mathbb{E}_{S,S',\mathcal{A}} [M \mid p] \\
& = \mathbb{E}_p \mathbb{E}_{S,S',\mathcal{A}} [ \mathcal{A}^j(S) (S'^j_i - p_j) \mid p] \\
& = \mathbb{E}_p \Big[ \mathbb{E}_{S,S',\mathcal{A}} [ \mathcal{A}^j(S) \mid p] \text{  }\mathbb{E}_{S,S',\mathcal{A}}[ (S'^j_i - p_j) \mid p] \Big] = 0,
\end{align*}
where the last equality follows because conditioned on the vector $p$, the expectation of the Rademacher $S^j_{i}$ is exactly equal to $p_j$. Hence, by linearity of expectation, we get that the expectation of $Z'$ is also $0$.

Next, we show that $M$ and $N$ are uncorrelated. First, conditioning on $p$ we can write
\begin{align*}
    \mathbb{E}_{p,S,S',\mathcal{A}}[MN] &= \mathbb{E}_p \mathbb{E}_{S,S',\mathcal{A}} [MN \mid p]\\
    &=\mathbb{E}_p \Big[ \mathbb{E}_{S,S',\mathcal{A}} [\mathcal{A}^j(S)^2 \mid p] \text{  }\mathbb{E}_{S,S',\mathcal{A}}[  (S'^j_i - p_j)  (S'^j_{i'} - p_j) \mid p] \Big]
% &=\mathbb{E}_p \Big[ \mathbb{E}_{S,S',\mathcal{A}} [\mathcal{A}^j(S)^2 \mid p] \text{  }\mathbb{E}_{S,S',\mathcal{A}}[  (s'^j_i - p_j)|p]  \mathbb{E}_{S,S',\mathcal{A}}[ (s'^j_{i'} - p_j) \mid p] \Big]
\end{align*}
since $S$ and $S'$ are independent after conditioning. Since this is also the case for $(S'^j_i - p_j)$ and $(S'^j_{i'} - p_j)$ 
% (since $s'^j_i$ and $s'^j_{i'}$ are independent Bernoulli random variables drawn with the same bias)
we have
\[
\mathbb{E}[  (S'^j_i - p_j)  (S'^j_{i'} - p_j) \mid p] = \mathbb{E}[  (S'^j_i - p_j) \mid p] \text{  }\mathbb{E}[  (S'^j_{i'} - p_j) \mid p]  = 0.
\] 
% where the equality to $0$ is because conditioned on the vector $p$, the expectation of the Bernoulli $s'^j_{i'}$ is exactly equal to $p_j$. 
Since we have already seen that $\mathbb{E}[M]\mathbb{E}[N]=0$ (since $\mathbb{E}[M]=0$), this implies that $M$ and $N$ are uncorrelated as desired. 
% A similar argument works for other pairs of random variables in the sum representing $Z'$.

Now assuming without loss of generality that $\mathcal{A}$ outputs values between $[-1,1]$ (rounding inputs to this range only improves the accuracy and doesn't affect perfect generalization, which is robust to post-processing), we have that
\begin{align*}
\mathbb{E}^2[|Z'|] & = \mathbb{E}_{p,S,S',\mathcal{A}}^2\Big[\Big| \sum_{j \in [d]} \mathcal{A}^j(S)\sum_i (S'^j_i - p_j) \Big| \Big] & (
\text{$S$ and $S'$ are i.i.d})\\
& \leq \mathbb{E}_{p,S,S',\mathcal{A}}\Big[\Big(\sum_{j \in [d]} \mathcal{A}^j(S)\sum_i (S'^j_i - p_j)\Big)^2\Big] & (\text{Jensen's Inequality})\\
&= Var(Z'). & (\mathbb{E}[Z']=0)
\end{align*}
Since $Z'$ is a sum of uncorrelated random variables and $\mathcal{A}^j(S)^2 \leq 1$, we then get
\begin{align*}
Var(Z') &= \sum_j \sum_i Var(\mathcal{A}^j(S) (S'^j_i - p_j))\\
& \leq \sum_j \sum_i \mathbb{E}[(S'^j_i - p_j)^2]\\ 
&\leq 4dm
\end{align*}
as desired.
% where the final inequality is by the fact that $s'^j_i$ and $p_j$ are both between $-1$ and $1$.

% Hence, we get that 
% $$\mathbb{E}[|Z'|] = \mathbb{E}_{p,S,S', \mathcal{A}}[\sum_{j \in [d]} \mathcal{A}^j(S')\sum_i (s^j_i - p_j)] \leq 2\sqrt{dm}.$$

%\sstext{Condition on fixed $p$ in below argument and then take expectation?}
It is left to show \Cref{eq:fp3}. Let $Z_p$ be the random variable $Z$ conditioned on fixed $p$ (and likewise for $Z'$). Let $Z_{p,S,S'}$ be the random variable $Z$ conditioned on fixed $p$, $S$ and $S'$ (and likewise for $Z'$). If $A$ is perfectly generalizing, then by \Cref{lem:samplePG} and \Cref{prelim:postprocess}, for all fixed $p$, with probability at least $1-\frac{1}{m^3}$ over the draw of $S, S'$, we have that $Z_{p,S,S'}$ and $Z'_{p,S,S'}$ are distributionally close, as are $|Z_{p,S,S'}|$ and $|Z'_{p,S,S'}|$. %\chris{why do we need Lemma 2.5 and 3.5? Both $Z, Z'$ use $\mathcal{A}(S)$ so shouldn't they be distributionally the same? I probably am misunderstanding something here..} \sstext{The difference is the quantifiers- i.e. you draw $S$ and $S'$ and then consider the distribution of the output of the algorithm on these two datasets over only the randomness of the algorithm and not the datasets. Lemma 2.5 gives you postprocessing, and Lemma 3.5 gives you a conversion from PG to sample PG (which is the right form here since we're dealing with 2 iid datasets)}
For any fixed $p$, let $E$ be the event that $|Z_{p,S,S'}| \approx_{2,\frac{3}{m^3}} |Z'_{p,S,S'}|$, where the randomness in $E$ comes from the randomness of sampling $S$ and $S'$. Then, by the guarantee of perfect generalization, we have that for all fixed $p$, $E$ occurs with probability at least $1-\frac{1}{m^3}$ and for any fixed $p$ we can write:
% Then, we can write the following sequence of inequalities for any fixed $p$.
\begin{align*}
    \mathbb{E}[|Z_p|] = \mathbb{E}_{S,A}[\Big| \sum_{j \in [d]} \mathcal{A}^j(S)\sum_i (S^j_i - p_j) \Big| ] & =
    \int_0^{2dm} \Pr[|Z_p|>z] dz \\
    & =  \int_0^{2dm} \left[\Pr[|Z_p|>z \mid E]\Pr[E] + \Pr[|Z_p|>z \mid \overline{E}]\Pr[\overline{E})\right]  dz \\ & \leq
    \int_0^{2dm} (e^2 \Pr[|Z'_p|>z \mid E] + \frac{3}{m^3})\Pr[E] + \Pr[\overline{E}]  dz \\
    & \leq
    \int_0^{2dm} e^2 \Pr[|Z'_p|>z \mid E]\Pr[E] dz + \int_0^{2dm} \left[\frac{3}{m^3} + \Pr[\overline{E}]\right] dz \\
     & \leq
    \int_0^{2dm} e^2 \Pr[|Z'_p|>z ] dz + \int_0^{2dm} \left[\frac{3}{m^3} + \frac{1}{m^3}\right] dz \\
    & = e^2\mathbb{E}_{S,S',\mathcal{A}}[|Z'_p|] + \frac{8d}{m^2},
\end{align*}
% where the first and second equalities are by the definition of expectation (and the definition of random variable $Z$), the third equality is by the law of total probability, 
where the first inequality follows since $|Z_p|$ and $|Z'_p|$ are distributionally close conditioned on $E$, the second inequality is by the fact that $\Pr(E) \leq 1$, and the third since $\overline{E}$ corresponds to the probability of failure in the definition of perfect generalization.
% inequality is by adding a non-negative term to the first integral and using the law of total probability again and by the fact that 

Finally taking expectation with respect to $p$, we get that
$$\mathbb{E}[|Z|] \leq e^2\mathbb{E}[|Z'|] + \frac{8d}{m^2}$$
as desired.
% Hence, combining the inequalities described above, we get that
% $$\frac{d}{10} \leq \mathbb{E}[|Z|] \leq e^2 \mathbb{E}\Big[ |Z'| \Big] + \frac{8d}{m^2} \leq 2e^2 \sqrt{dm} + \frac{8d}{m^2}.$$
% Simplifying, this gives that
% $$m = \Omega(d),$$
%since $\frac{1}{m^2}$ needs to go to $0$ as $d$ goes to infinity, as $m$ has to depend at least logarithmically on $d$ even for non-privately solving the one-way marginals problem.
\end{proof}
Now, we are ready to prove the lower bound for replicable algorithms.

\begin{proof}[Proof of Theorem~\ref{thm:repmarginalsLB}]
Let $m$ be larger than an absolute constant $K$, without loss of generality.\footnote{We will show under this condition that $m = \Tilde{\Omega}(d)$. Any algorithm taking fewer than $K$ samples implies one taking between $K$ and $\tilde{O}(d)$ samples, which would give a contradiction.} By Claim~\ref{claim:2parrep}, we have that $\mathcal{A}$ is $(0.01,0.01)$-replicable and $(0.01,0.001)$-accurate when given $m$ samples. Consider any sufficiently small $\gamma > 0$, and sufficiently large constant $c>0$. Applying Theorem~\ref{thm:amprep}, we get that there is a $\frac{1}{c\log(1/\gamma)}$-replicable and $(0.01,0.008 + \frac{1}{\log(1/\gamma)}) = (0.01,0.01)$-accurate algorithm $\mathcal{A'}$ for one-way marginals over $d$ coordinates, which takes $O\left(m \log^2 (1/\gamma) \right)$ samples. 

Next, we give a way of replicably amplifying the failure probability to $\gamma$. We run the algorithm $\mathcal{A'}$ $k=20\log (1/\gamma)$ times on different samples, and take the coordinate-wise median of the outputs. Observe that for each coordinate, if more than half the values in that coordinate are within $0.01$ of the true bias, then the median is correct. Consider the probability that more than half the output values in a coordinate are not within $0.01$ of the true bias. By a Chernoff bound, we have that the number of outputs which are within $0.01$ of the true expectation in $l_{\infty}$ norm are more than $0.5k$ with probability at least $1-\gamma^2$, which guarantees that we get a $(0.01,\gamma^2)$-accurate algorithm for one-way marginals. Using composition of replicability, we have that the resulting algorithm is $(0.01,0.01)$-replicable and takes $O\left(m \log^3 (1/\gamma) \right)$ samples. 

Consider any sufficiently small $\delta > 0$. By Theorem~\ref{thm:reprodtoPG}, we have that there is a $(2\delta,1,2\delta)$-PG algorithm with failure probability at most $\delta + \gamma \log (1/\delta)$ when given $m' = O(m \log^3 (1/\gamma) \text{polylog}(1/\delta))$ samples. Setting $\gamma = \frac{0.005}{\log 1/\delta}$, we get that  that for sufficiently small $\delta > 0$, there is a $(2\delta,1,2\delta)$-PG algorithm with failure probability at most $\delta + 0.005$ (i.e. $(0.01,\delta + 0.005)$-accurate), when given $m' = O(m\cdot \text{polylog}(1/\delta) )$ samples. Setting $\delta = \frac{1}{2m'^3}$ and simplifying, we get that $m' = C m \cdot \text{polylog}(m)$ for some constant $C$. Then, since $m'>m$ is larger than $K$, we get that $\delta$ is smaller than $\frac{1}{K}$ and setting $K$ sufficiently large gives us a $(0.01,0.01)$-accurate algorithm with $m'$ samples. 

Now, using the lower bound for perfect generalization in Theorem~\ref{thm:PGmarginalsLB}, we get that $m' = \Omega(d)$, which gives us that $m = \Tilde{\Omega}(d)$, completing the proof.
\end{proof}

\subsection{Quadratic separation: Agnostic Learning}
\label{sec:agnostic-lb}
In this section, we prove a lower bound for agnostic learning (See Section~\ref{sec:PAC} for the definition of agnostic learning) under the constraint of replicability.
\begin{theorem}\label{thm:agnosticrepLB}
Fix sufficiently large $d>0$ and a hypothesis class $H$ with VC dimension $d$. Any $(0.01,0.001)$-accurate, $0.0001$-replicable agnostic learner $\mathcal{A}$ for $H$ requires at least $\tilde{\Omega}(d^2)$ examples.
\end{theorem}

The key idea is that we will reduce a variant of the one-way marginals problem over $d$ coordinates to the problem of agnostically learning any hypothesis class with VC dimension $d$ (with quadratically more samples). The variant we consider loosely corresponds to predicting the signs of the biases of the product distribution. We show that this is possible using an agnostic learner as a subroutine. We start by defining this problem more precisely.
\subsubsection{Sign-One-Way Marginals}
\begin{definition}
Consider a product of $d$ Rademacher distributions with expectations $p = (p_1,\dots,p_d)$. A vector $v \in [-1,1]^d$ is said to be an $\alpha$-accurate solution to the sign-one-way marginals problem for this distribution if $\frac{1}{d}\sum_{j=1}^d v_j p_j \geq \frac{1}{d}\sum_{j=1}^d |p_j| - \alpha$.
\end{definition}
Observe that if every $p_j$ is either $-1$ or $1$, an accurate solution requires the $v_j$'s to do a very good job of predicting the signs on average. On the other hand, if the $p_j$ values are all $0$, then every value of $v$ is a $0$-accurate solution. Thus, this definition of error scales depending on how biased the expectation is to either $1$ or $-1$, penalizing solutions more when they do a poor job of predicting heavily biased coordinates. (Indeed, we'd expect biased coordinates to be easier to predict, so it makes sense to penalize solutions more on these coordinates.) Now, we are ready to define the accuracy of an algorithm for the sign-one-way marginals problem.
\begin{definition}
Let $C$ be the class of products of $d$ Rademacher random variables. We say that an algorithm $\mathcal{A}:\{\{-1,1\}^d\}^m \to [-1,1]^d$ $(\alpha, \beta)$-accurately solves the sign-one-way marginals problem over class $C$ if for all fixed distributions $D$ in $C$, with probability at least $1-\beta$ over the randomness of the examples it obtains from $D$ and the internal randomness of the algorithm, it outputs an $\alpha$-accurate vector $v$ for $D$.
\end{definition}
%\mb{Something seems weird about this definition...why can't A have the bias vector p hardcoded into it and always predict that? Seems like accuracy of an algorithm should be defined w.r.t. the whole class of Rademacher products, i.e., ``A is accurate if for every Rademacher product, with probability at least ...}
%\sstext{Fixed.}
\subsubsection{Solving Sign-One-Way Marginals using Agnostic Learning}

Our reduction in Algorithm~\ref{alg:ag2marg} shows how to use an agnostic learner $\mathcal{A}_{ag}$ for any class $H$ of VC dimension $d$ to construct an algorithm $\mathcal{A}$ for the sign-one-way marginals problem. 

The main idea of the algorithm is as follows. Fix a distribution $D$ that is a product of Rademachers and let its expectation be $p = (p_1,\dots,p_d)$. Consider a shattered set $x_1,\dots,x_d$ for hypothesis class $H$. Consider a distribution $D'$ corresponding to sampling a uniformly random point $x_j$ from the shattered set and then sampling a label in $(-1,1)$ from a Rademacher with expectation $p_j$. Note that given a sample $S$ of $d$ independently drawn examples from $D$, we can create a dataset $S_{ag}$ of size roughly $d^2$ that looks like an i.i.d. sample from $D'$, by sampling a uniformly random $x_j$ and labeling it with a new unused entry from coordinate $j$ of $S$ (we won't run out of entries with high probability). Now, note that since the set $x_1,\dots,x_d$ is shattered by $H$, there is a hypothesis $h$ in $H$ that outputs $\sign(p_j)$ on input $x_j$ (such a hypothesis also achieves lowest possible error on $D'$ among hypotheses in $H$). If the agnostic learner is accurate when given $d^2$ samples, then the function $f$ it outputs is a good approximation to $h$ and as a result $f(x_j)$ is also likely to be an accurate prediction of the sign of $p_j$. Hence, function $f$ can be used to obtain an accurate solution to the sign-one-way marginals problem.

%\mb{Three sentence summary of the idea?}

\begin{algorithm}[H]
    \caption{Algorithm $\mathcal{A}$ for sign-one-way marginals}
    \label{alg:ag2marg}
    \hspace*{\algorithmicindent} \textbf{Input:} Sample access to a product distribution $D$ over $\{-1,1\}^d$, agnostic learner $\mathcal{A}_{ag}$ for hypothesis class $H$ with VC dimension $d$\\
    \hspace*{\algorithmicindent} \textbf{Output:} Estimated biases $(v_1,\dots,v_d).$
    \begin{algorithmic}[1] % The number tells where the line numbering should start
    	    \STATE Draw $\frac{d}{\log^c d}$ i.i.d. examples from $D$ for some $c>0$. Call the corresponding sample $S_{inp}$.
		    \STATE Let $x_1,\dots,x_d$ be a shattered set of points for $H$. Let $U_d$ be the uniform distribution over $x_1,\dots,x_d$.
		    \STATE Draw $m=\frac{d^2}{100 \log^{2c} d}$ examples $S_j$ from $U_d$. Call the sample $S_{ag}$. If any element $x_i$ occurs more than $\frac{d}{\log^c d}$ times, then output $(1,\dots,1)$, else move to the next step. \label{step:sampag}
		    \STATE For each example $S_j$, label it with a new entry from coordinate $j$ of the input sample $S_{inp}$. Call the labeled sample $S_{ag,lab}$.
		    \STATE Run agnostic learner $\mathcal{A}_{ag}$ on the labeled sample $S_{ag,lab}$. Let the output function be $f$.
		    \RETURN $(f(x_1), f(x_2), \dots, f(x_d))$.
    \end{algorithmic}
\end{algorithm}

\begin{theorem}\label{thm:agn2sowm}
Fix sufficiently large $d>0$. Let  $\mathcal{A}_{ag}$ be a $(0.01,0.001)$-accurate, 
$0.0001$-replicable agnostic learner for a hypothesis class $H$ with VC dimension $d$. Then algorithm $\mathcal{A}$ is a $(0.02,0.002)$-accurate, $0.0003$-replicable algorithm for the sign-one-way marginals problem over $d$ coordinates.
\end{theorem}

\begin{proof}
Let $D$ be a Rademacher distribution with expectation $(p_1, \dots,p_d)$. Define a distribution $D_{ideal}$ over $\{x_1,\dots,x_d\} \times \{-1,1\}$ as follows. First, uniformly draw $s \in \{x_1, \dots, x_d\}$. Then, if $s=x_j$, draw $y$ from a Rademacher with expectation $p_j$. Let $D_{ideal}$ be the distribution of the random variable $(s,y)$ obtained using this procedure. 

First, we observe that by a Chernoff bound and union bound, the probability that any element $x_i$ occurs more than $\frac{d}{\log^{c}d}$ times in $S_{ag}$ (where $S_{ag}$ is sampled as described in Step~\ref{step:sampag}) is exponentially small in $d$ (hence less than $0.01$ for sufficiently large $d$). Call this bad event $E$.% \mb{Call out step 1 as the place in the algorithm where this happens}

Next, consider the following method for sampling a dataset $S_{ideal}$ of $\frac{d}{100\log^{2c}d}$ i.i.d. samples from $D_{ideal}$: first draw $\ell=\frac{d}{100\log^{2c}d}$ i.i.d. examples $s_i \sim U_d$ and then for each of them, if the value obtained is $x_j$, sample $y_j \sim Rad(p_j)$. Consider the event $E_{ideal}$ that the number of occurrences of any example $x_i$ is larger than $\frac{d}{\log^c d}$. The probability of this event is exactly equal to $\Pr[E]$. Notice that the distribution of $S_{ag,lab}$ is identical to the distribution of $S_{ideal}$ conditioned on event $\overline{E_{ideal}}$. \chris{supposed to be $D_{ideal}$?} %\mb{Confusing since $E$ is an event over the sample space of $S_{ag, lab}$. I think you mean $S_{ag, lab} \mid \overline{E}$ 
%is distributed according to $D^{\ell}_{lab}$?}
For any distribution $D'$ and event $\overline{E}$, a simple calculation shows $\dtv(D',D'|_{\overline{E}}) \leq \Pr[E]$. %\mb{Probably want something like $D|\overline{E} = D' \implies d_{TV}(D, D') \le \Pr[E]$} 
Hence, we get that the total variation distance between the distribution of $S_{ideal}$ conditioned on event $\overline{E_{ideal}}$ and the distribution of $S_{ideal}$ is at most $0.0001$ for sufficiently large $d$. %\mb{$\Pr[E]$ was 0.01 in the previous paragraph}

Now, we know that with probability at least $0.999$ over the coins of the algorithm and the sample, the agnostic learner produces an output that is accurate with respect to its input sample. Now, by the data-processing inequality for total variation distance, we have that $\dtv(\mathcal{A}_{ag}(S_{ideal}), \mathcal{A}_{ag}(S_{ag,lab})) \leq 0.0001$. Consider the distribution $D_{ideal}$ and the subset $O$ of $0.01$-accurate functions w.r.t. the best function in the class $H$ (i.e. the function that minimizes $\Pr_{(x,y) \in D_{ideal}}[h(x) \neq y]$) . %\sstext{Put this definition in preliminaries}. 
By the definition of total variation distance, we have that the probability that learner $\mathcal{A}_{ag}$ produces outputs in this subset $O$ on seeing $S_{ag,lab}$ is within $0.0001$ of the probability that $\mathcal{A}_{ag}$ produces outputs in this subset $O$ on seeing $S_{ideal}$. Since the latter happens with probability at least $0.999$, we have that with probability at least $0.9989$, the agnostic learner is $0.01$-accurate when fed the sample $S_{ag,lab}$. This implies by the definition of the accuracy guarantee that with probability at least $0.9989$ over the randomness of the learner $\mathcal{A}_{ag}$ and sample $S_{ag,lab}$, that
\begin{align*}
    \Pr_{(x,y) \sim D_{ideal}}[f(x) \neq y] \leq \inf_{h \in H} \Pr_{(x,y) \sim D_{ideal}}[h(x) \neq y] + 0.01,
\end{align*}
where $f$ is the output function of the agnostic learner. For $x \in \{x_1,\dots,x_d\}$, let $p_x$ be $p_j$ if $x = x_j$. Observe that the function that predicts $\sign(p_x)$ achieves the infimum on the right hand side of the above equation.

Now, using the fact that expectation of an indicator is the probability of the indicated event,  and that $\mathbbm{1}[a \neq b] = \frac{1-ab}{2}$ when $a$ and $b$ are in $\{-1,1\}$, we get that with probability at least $0.9989$ over the randomness of the learner $\mathcal{A}_{ag}$ and sample $S_{ag,lab}$,
\begin{align*}
    & \mathbb{E}_{(x,y) \sim D_{ideal}}[ \mathbbm{1}[f(x) \neq y]] \leq \inf_{h \in H} \mathbb{E}_{(x,y) \sim D_{ideal}}[\mathbbm{1}[h(x) \neq y]] + 0.01 \\
    \implies & \mathbb{E}_{(x,y) \sim D_{ideal}}\left[ \frac{1-f(x) y}{2}\right] \leq \inf_{h \in H} \mathbb{E}_{(x,y) \sim D_{ideal}}\left[ \frac{1 - h(x) y}{2} \right] + 0.01 \\
    \implies  & \mathbb{E}_{(x,y) \sim D_{ideal}}\left[ \frac{1-  f(x) y}{2}\right] \leq \mathbb{E}_{(x,y) \sim D_{ideal}}\left[ \frac{1 - \sign(p_x) y}{2} \right] + 0.01 \\
    \implies  & \mathbb{E}_{(x,y) \sim D_{ideal}}\left[ f(x) y \right] \geq \mathbb{E}_{(x,y) \sim D_{ideal}}\left[ \sign(p_x) y \right] - 0.02.
\end{align*}
Now, Algorithm $\Acal$ calls the agnostic learner except with probability $0.0001$. Unraveling the expectations, accounting for the fact that $\mathcal{A}$ outputs $(1,1\dots,1)$ when it doesn't call the agnostic learner, and using the fact that the randomness of sample $S_{ag,lab}$ is from the randomness of the algorithm $\Acal$ as well as the randomness of $S_{inp}$, we get that with probability at least $0.998$ over the randomness of the algorithm $\mathcal{A}$ and input sample, $S_{inp}$, that
\begin{align*}
    \frac{1}{d} \sum_{j=1}^d f(x_j) p_j  \geq \frac{1}{d} \sum_{j=1}^d \sign(p_j) p_j - 0.02,
\end{align*}
proving that $\mathcal{A}$ is a $(0.02,0.002)$-accurate algorithm for sign-one-way marginals over $d$ coordinates.
%Now, taking expectation over the randomness of the algorithm $\mathcal{A}$ and the input sample  $S_{inp}$, we get that
%\begin{align}
 %   \mathbb{E}_{\mathcal{A}, S_{inp}} \left[\frac{1}{d} \sum_{j=1}^d f(x_j) p_j \right] \geq 0.98 \left[\frac{1}{d} \sum_{j=1}^d |p_j| - 0.02 \right] - 0.02,
%\end{align}

%which implies that 
%\begin{align}
 %   \mathbb{E}_{\mathcal{A}, S_{inp}} \left[\frac{1}{d} \sum_{j=1}^d \mathcal{A}^j(S_{inp}) p_j \right] \geq \frac{1}{d} \sum_{j=1}^d |p_j| - 0.06,
%\end{align}
%proving that $\mathcal{A}$ is a $0.06$-accurate algorithm for sign-one-way marginals over $d$ coordinates.

\mb{Can we take the rep parameter to be at least 10x bigger? It doesn't look great to have such tiny explicit constants around.} Next, we prove that $\mathcal{A}$ inherits the replicability of the agnostic learner $\mathcal{A}_{ag}$. Consider two sets of independent samples $S_{inp,1}$ and $S_{inp,2}$. Consider any set of random coins $r$ drawn for algorithm $\mathcal{A}$. 
Note that when the coins dictate that when some point in the shattered set occurs too many times, the algorithm always outputs $(1,\dots,1)$. Recall that this is event $E$, which we previously showed occurs with probability at most $0.0001$. Hence, in this case $\mathcal{A}(S_{inp,1} ; r) = \mathcal{A}(S_{inp,2} ; r) $ with probability $1$. 
Hence, it is sufficient to consider coins such that every point in the shattered set occurs fewer than $\frac{d}{\log^{2c} d}$ times in the sample $S_{ag}$. 
Now, as argued previously, the total variation distance between the distribution of $S_{ag,lab}$ (call it $D_{ag}$) and the same number of i.i.d. samples from $D_{ideal}$ is at most $0.0001$ for sufficiently large $d$. Thus, we have that the probability of any event changes by at most $0.0002$ under samples $S_{ag, lab}, S'_{ag, lab} \sim D_{ideal}^m$ versus samples $S_{ag, lab}, S'_{ag, lab} \sim D_{ag}^m$. This allows us to conclude that

\begin{align*}
    \Pr_{S_{inp,1}, S_{inp,2} \sim D^{m}, r}& [\mathcal{A}(S_{inp,1} ; r) = \mathcal{A}(S_{inp,2} ; r) ]  \\ 
    &\geq
 \Pr_{S_{inp,1}, S_{inp,2} \sim D^{m}, r} [\mathcal{A}(S_{inp,1} ; r) = \mathcal{A}(S_{inp,2} ; r) \mid \overline{E}] \\
    & \ge \Pr_{S_{ag, lab}, S'_{ag, lab} \sim D_{ag}, r_{ag}} [\mathcal{A}_{ag}(S_{ag,lab} ; r_{ag}) = \mathcal{A}_{ag}(S'_{ag,lab} ; r_{ag})] \\
    & \ge \Pr_{S_{ag, lab}, S'_{ag, lab} \sim D_{ideal}^{m}, r_{ag}} [\mathcal{A}(S_{ag,lab} ; r_{ag}) = \mathcal{A}(S'_{ag,lab} ; r_{ag})] - 0.0002 \\
    & \ge 0.9999 - 0.0002 = 0.9997.
\end{align*}
Hence, we have proved that $\mathcal{A}$ is $0.0005$-replicable.
\end{proof}

\subsubsection{Lower Bound for Sign-One-Way Marginals}

In this section, we show that accurately and replicably solving the sign-one-way marginals problem over $d$ coordinates requires a number of samples that is nearly linear in $d$. We will use a variant of the \textit{fingerprinting method} used to prove the lower bound for the one-way marginals problem for perfectly generalizing algorithms, and then extend this to a lower bound for replicable algorithms. This argument is similar to that used to prove the lower bound for the one-way marginals problem (Theorem~\ref{thm:PGmarginalsLB}), except that the notion of accuracy is different, and so we use a different version of the fingerprinting lemma, given below.
%We will require a technical lemma for this argument, that was proved as part of proving the larger fingerprinting lemma \cite{BunUV18, DworkSSUV15, BunSU19}. \mb{Repetivive text here. I think it's worth calling out what specifically is different here and if there's anything succinct we can say about why.}
\begin{lemma}[\cite{BunSU19}, Lemma A.1 and A.2]\label{lem:fingder}
Let $f$ be a function from $\{-1,1\}^m \to \mathcal{R}$. Let $p$ be a uniformly random variable between $-1$ and $1$, and $\vec{x}$ be a random vector of length $m$, consisting of i.i.d. Rademacher random variables with expectation $p$.

Then, 
\begin{align*}
    \mathbb{E}_{p, \vec{x}}[f(\vec{x}) \sum_{i=1}^m (x_i - p)] = \mathbb{E}_p [2pg(p)]
\end{align*}
where $g(p) = \mathbb{E}_{\vec{x} \sim p}[f(\vec{x})]$.
\end{lemma}

\begin{theorem}\label{thm:PGsignOWMLB}
Fix any $m > 0$, sufficiently large $d>0$. Let $\mathcal{A}$ be a $(\frac{1}{m^3},1,\frac{1}{m^3})$-perfectly generalizing, $(0.05,0.05)$-accurate algorithm for the sign-one-way marginals problem over $d$ coordinates using $m$ samples that always outputs a vector in $[-1,1]^d$. Then, $m = \Omega(d)$. 
\end{theorem}
\begin{proof}
Without loss of generality, let $m$ be larger than a constant $K$.\footnote{We will show under this condition that $m = \Omega(d)$. Any algorithm on $ \leq K$ samples implies one taking between $K$ and $O(d)$ samples, which would give a contradiction.} Let $S$ be the input dataset to the algorithm $\mathcal{A}$. Following the framework in \Cref{thm:PGmarginalsLB}, we will first argue the expected correlation of our algorithm and its input is large:
\[
\sum_{j=1}^d \mathbb{E}_{p, \mathcal{A}, S \sim D_p^m}[\mathcal{A}^j(S) \sum_{i=1}^m (S^j_i - p_j)] \geq 0.35d.
\]
To see this, observe that by the accuracy of the algorithm, for any fixed distribution $D_p$ that is a product of Rademachers with expectation $p=(p_1,\dots,p_d)$, we have that we have that
with probability at least $0.95$ over the randomness of the algorithm $\mathcal{A}$ and its input sample,
\begin{align*}
    \frac{1}{d} \sum_{j=1}^d v_j p_j  \geq \frac{1}{d} \sum_{j=1}^d \sign(p_j) p_j - 0.05,
\end{align*}
where $v$ is the vector output by the algorithm. 

Now, taking expectation over the randomness of the algorithm $\mathcal{A}$ and the input sample  $S$, we get that
\begin{align*}
   \mathbb{E}_{\mathcal{A}, S  \sim D_p^m} \left[\frac{1}{d} \sum_{j=1}^d \mathcal{A}^j(S) p_j \right] \geq 0.95 \left[\frac{1}{d} \sum_{j=1}^d |p_j| - 0.05 \right] - 0.05,
\end{align*}

which implies that 
\begin{align*}
    \mathbb{E}_{\mathcal{A}, S \sim D_p^m} \left[\frac{1}{d} \sum_{j=1}^d \mathcal{A}^j(S) p_j \right] \geq \frac{1}{d} \sum_{j=1}^d |p_j| - 0.15,
\end{align*}
%proving that $\mathcal{A}$ is a $0.06$-accurate algorithm for sign-one-way marginals over $d$ coordinates.

%Taking expectation over the randomness of the algorithm $\mathcal{A}$ and the input sample  $S_{inp}$, we get that
%\begin{align}
 %   \mathbb{E}_{\mathcal{A}, S_{inp}} \left[\frac{1}{d} \sum_{j=1}^d f(x_j) p_j \right] \geq 0.98 \left[\frac{1}{d} \sum_{j=1}^d |p_j| - 0.02 \right] - 0.02,
%\end{align}

%By the fact that $\mathcal{A}$ is $0.05$-accurate, 
 %\begin{align*}
%         & \mathbb{E}_{\mathcal{A}, S \sim D^m} \left[\frac{1}{d} \sum_{j=1}^d \mathcal{A}^j(S) p_j \right] \geq \frac{1}{d} \sum_{j=1}^d |p_j| - 0.06 \implies \\
 %        & \sum_{j=1}^d \mathbb{E}_{\mathcal{A}, S \sim D^m} \left[\mathcal{A}^j(S) p_j \right] \geq \sum_{j=1}^d |p_j| - 0.06d.
 %\end{align*} 
Now, consider each coordinate of expectation vector $p$ \mb{It'd be nice to use consistent terminology like ``bias vector'' or ``parameter vector'' for this, since ``expectation'' is an overloaded term} drawn uniformly from $[-1,1]$. Then, we have that conditioned on any fixed $p$, the above equation holds. Hence, using the law of total expectation, we get that
 \begin{align}
         & \mathbb{E}_{p,\mathcal{A}, S\sim D_p^m} \left[ \sum_{j=1}^d \mathcal{A}^j(S) p_j \right] \geq \mathbb{E}_p\left[\sum_{j=1}^d |p_j| \right]- 0.15d \\
         \implies & \sum_{j=1}^d \mathbb{E}_{p, \mathcal{A}, S \sim D_p^m} \left[\mathcal{A}^j(S) p_j \right] \geq \sum_{j=1}^d \mathbb{E}_{p_j \sim [-1,1]}\left[|p_j|\right] - 0.15d = 0.35d \label{eq:corrlb},
 \end{align} 

where we have used the fact that $\mathbb{E}_{p_j \sim [-1,1]}[|p_j|] = \frac{1}{2}$. Now, fix a coordinate $j \in [d]$. For any fixed internal randomness $r$ of algorithm $\mathcal{A}$, and for any values of columns of $S$ that are not the $j^{th}$ column, we get from Lemma~\ref{lem:fingder} applied to the function $f$ corresponding to the algorithm $\mathcal{A}$ on the complete dataset $S$ with internal randomness $r$, that

$$\mathbb{E}_{p_j, S^j \sim Rad(p_j)^m} \left[\mathcal{A}^j(S; r) p_j \right] = \mathbb{E}_{p_j, S^j \sim Rad(p_j)^m}[\mathcal{A}^j(S; r) \sum_{i=1}^m (S^j_i - p_j)].$$

Now, since this holds for any fixed values of internal randomness $r$ and for any values of columns of $S$ that are not the $j^{th}$ column, it holds for any distribution over the internal randomness $r$ and any distribution over values of other columns of $S$. Hence, we get that 

$$\mathbb{E}_{p, \mathcal{A}, S \sim D_p^m} \left[\mathcal{A}^j(S) p_j \right] = \mathbb{E}_{p, \mathcal{A}, S \sim D_p^m}[\mathcal{A}^j(S) \sum_{i=1}^m (S^j_i - p_j)],$$
where we have used that $D_p$ is a product distribution. Now, summing over all coordinates $j \in [d]$, we get that 
\begin{align}\label{eq:Zlb}
\sum_{j=1}^d \mathbb{E}_{p, \mathcal{A}, S \sim D_p^m}[\mathcal{A}^j(S) \sum_{i=1}^m (S^j_i - p_j)] = \sum_{j=1}^d \mathbb{E}_{p, \mathcal{A}, S \sim D_p^m} \left[\mathcal{A}^j(S) p_j \right] \geq 0.35d,
\end{align}
where we have used \Cref{eq:corrlb}. 

Now, we can proceed exactly as in the proof of Theorem~\ref{thm:PGmarginalsLB} (we repeat high-level details for completeness; for more details, see that proof).

Let $S'$ be another dataset drawn from the same distribution $D_p$. Let $Z = \sum_{j \in [d]} \mathcal{A}^j(S)\sum_i (S^j_i - p_j)$ and $Z' = \sum_{j \in [d]} \mathcal{A}^j(S')\sum_i (S^j_i - p_j)$.

First, note that $Z'$ is a sum of uncorrelated random variables with mean $0$ (see the proof of Theorem~\ref{thm:PGmarginalsLB} for a proof of this).

We can then prove (as in the proof of Theorem~\ref{thm:PGmarginalsLB}) that
\begin{align*}
\mathbb{E}[|Z'|] = \mathbb{E}_{p, \mathcal{A}, S,S' \sim D_p^m}\Big[\Big| \sum_{j \in [d]} \mathcal{A}^j(S')\sum_i (S^j_i - p_j) \Big| \Big] \leq 2\sqrt{dm},
\end{align*}

%If $\mathcal{A}$ is perfectly generalizing, then since perfect generalization is robust to postprocessing, $Z$ and $Z'$ are distributionally close as well, as are $|Z|$ and $|Z'|$. Let $E$ be the event that $|Z| \approx_{2,\frac{3}{m^3}} |Z'|$, where the randomness comes from the randomness of sampling $S$ and $S'$. 

Then, we can invoke the perfect generalization guarantee to prove (as in the proof of Theorem~\ref{thm:PGmarginalsLB}) that 
\begin{align*}
    \mathbb{E}[ |Z| ] \leq e^2\mathbb{E}[|Z'|] + \frac{8d}{m^2},
\end{align*}
Hence, combining the inequality above with equation~\ref{eq:Zlb}, we get that
$$0.35d \leq \mathbb{E}[|Z|] \leq e^2 \mathbb{E}[|Z'|] + \frac{8d}{m^2} \leq 2e^2 \sqrt{dm} + \frac{8d}{m^2}.$$
Simplifying, this gives that
$$m = \Omega(d).$$
\end{proof}
Now, we are ready to apply our conversion from replicability to perfect generalization to prove a similar lower bound for replicable algorithms.
\begin{theorem}\label{thm:sOWMreplb}
Fix sufficiently large $d>0$. For any  $0.0005$-replicable algorithm $\mathcal{A}$ that is $(0.02,0.002)$-accurate on the sign-one-way marginals problem over $d$ coordinates with $m$ samples, $$m=\tilde{\Omega}(d).$$
\end{theorem}
\begin{proof}
Without loss of generality, let $m$ be larger than a constant $K$.\footnote{We will show under this condition that $m = \Tilde{\Omega}(d)$. Any algorithm on $ \leq K$ samples implies one taking between $K$ and $\tilde{O}(d)$ samples, which would give a contradiction.} By Claim~\ref{claim:2parrep}, we have that $\mathcal{A}$ is $(0.01,0.05)$-replicable and $(0.02,0.002)$-accurate when given $m$ samples. Consider any sufficiently small $\gamma > 0$. Next, applying Theorem~\ref{thm:amprep}, we get that for sufficiently large constant $c>1$, there is a $\frac{1}{c\log(1/\gamma)}$-replicable and $(0.01,0.008 + \frac{4}{c\log(1/\gamma)}) = (0.01,0.01)$-accurate algorithm %\mb{Why did 0.001 become 0.008?} 
$\mathcal{A'}$ for sign-one-way marginals over $d$ coordinates, which takes $O\left(m \log^2 (1/\gamma) \right)$ samples. 

Next, we give a way of replicably amplifying the failure probability to $\gamma$. We run the algorithm $\mathcal{A}$ for $k=20\log( 1/\gamma)$ times on different samples, and take the mean of the outputs. Using composition of replicability, we have that the resulting algorithm is $0.0001$-replicable and takes $O\left(m \log^3 (1/\gamma) \right)$ samples. Now, we analyze the failure probability of this algorithm. Let the output vectors of $\mathcal{A}$ on the $k$ runs be $v^1,\dots,v^k$. We are interested in the quantity $\frac{1}{d}\sum_{j=1}^d\left[ |p_j| - \frac{p_j}{k}\sum_{i=1}^k v^i_j \right]$. First, we analyze the expectation of this quantity. 
\begin{align*}
    \mathbb{E}\left[\frac{1}{d}\sum_{j=1}^d\left[|p_j| - \frac{p_j}{k}\sum_{i=1}^k v^i_j  \right] \right] & =  \mathbb{E}\left[\frac{1}{d}\sum_{j=1}^d p_j \left[\sign(p_j) - \frac{1}{k}\sum_{i=1}^k v^i_j  \right] \right] \\
    & = \mathbb{E}\left[\frac{1}{d}\sum_{j=1}^d \frac{1}{k}\sum_{i=1}^k  p_j \left[\sign(p_j) - v^i_j  \right] \right] \\
    & = \frac{1}{k}\sum_{i=1}^k  \mathbb{E}\left[ \frac{1}{d}\sum_{j=1}^d  p_j \left[\sign(p_j) - v^i_j  \right] \right] \leq 0.03 
\end{align*}
where the last inequality is because the quantity inside the last expectation is less than $0.02$ with probability at least $0.99$ (and because all the $v^i$ are identically distributed). Next, observe that the quantity
$\frac{1}{k} \sum_{i=1}^k \left( \frac{1}{d}\sum_{j=1}^d  p_j \left[\sign(p_j) - v^i_j \right] \right)$ is a sum of $k$ independent random variables (since the $i^{th}$ term in the sum only depends on random variable $v^i$) in the interval $[-2,2]$. Hence, using Hoeffding's inequality, we have that the probability that the sum is larger than $0.05$ is less than $\gamma^2$. 

Now, by Theorem~\ref{thm:reprodtoPG}, we have that there is a $(2\delta,1,2\delta)$-PG algorithm with failure probability at most $\delta + \gamma \log (1/\delta)$ when given $m' = O(m \log^3 (1/\gamma) \poly \log(1/\delta))$ samples (where failure in this case means outputting a solution that is not $0.05$-accurate). Setting $\gamma = \frac{0.005}{\log( 1/\delta)}$, we get that  that for sufficiently small $\delta > 0$, there is a $(2\delta,1,2\delta)$-PG algorithm with failure probability at most $\delta + 0.005$ (i.e., one that is $(0.05,\delta + 0.005)$-accurate), when given $m' = O(m \cdot \poly \log(1/\delta))$ samples. Setting $\delta = \frac{1}{2m'^3}$ and simplifying, we get that $m' = C \cdot m \cdot \poly \log m$ for some constant $C$. Hence, since $m' > m$ is larger than $K$, we get that $\delta$ is smaller than $\frac{1}{K}$ and setting $K$ to be sufficiently large, we get a $(0.05,0.05)$-accurate algorithm with $m'$ samples.

Now, using the lower bound for perfect generalization in Theorem~\ref{thm:PGsignOWMLB}, we get that $m' = \Omega(d)$, which gives us that $m = \Tilde{\Omega}(d)$, completing the proof.
\end{proof}

Now, we can use the reduction from sign-one-way marginals to agnostic learning to obtain the sample complexity lower bound for agnostic learning.

\begin{proof}[Proof of Theorem~\ref{thm:agnosticrepLB}]
If there were a $(0.01,0.001)$-accurate, $0.0001$-replicable agnostic learning algorithm using fewer than $\frac{d^2}{100\log^{2c} d}$ (where $c$ is some sufficiently large constant) samples, then by Theorem~\ref{thm:agn2sowm}, there would be a $(0.02,0.002)$-accurate, $0.0005$-replicable algorithm for the sign-one-way marginals problem over $d$ coordinates taking only $\frac{d}{\log^c d}$ samples, which contradicts Theorem~\ref{thm:sOWMreplb}. %This proves the result.
\end{proof}
We note that our agnostic learning lower bound as stated only holds only for constant accuracy, and might not give the optimal dependence on the accuracy parameter $\alpha$ for general $(\alpha, \beta)$-agnostic learning. We leave it as an open problem to determine the right dependence on $\alpha$.

\subsection{Closing the Gap: Realizable Learning}
\label{sec:finiteclasses}
Now that we've seen natural settings in which our reduction is tight (and therefore exhibited a quadratic statistical separation between privacy and replicability), it is reasonable to ask whether there are any settings under which the reduction is loose, or even where privacy and replicability might have the same statistical cost. In this section, we'll show this is indeed the case for (certain regimes of) a closely related problem: realizable PAC-learning. In particular, in this section we exhibit a replicable algorithm for PAC-learning that gives a quadratically improved dependence on the accuracy and confidence parameters over applying our reduction from privacy (see \Cref{thm:finitehypred}).
\begin{theorem}[Finite Classes are Replicably Learnable]
	\label{thm:rFinite}	
	Any class $H$ is replicably Agnostic learnable with sample complexity: 
	\[
	m(\rho,\alpha,\beta) \leq O\left(\frac{ \log^2|H|+\log\frac{1}{\rho\beta}}{\alpha^2\rho^2}\log^3\frac{1}{\rho} \right).
	\]
In the realizable setting, the $\alpha$-dependence can be improved to linear:
	\[
	m(\rho,\alpha,\beta) \leq O\left(\frac{ \log^2|H|+\log\frac{1}{\rho\beta}}{\alpha\rho^2}\log^3\frac{1}{\rho} \right).
	\]
	% \[
	% m(\rho,\alpha,\beta) \leq O\left(\frac{(\log|H| + \log\frac{1}{\rho\beta})\log^3\frac{1}{\rho}}{\alpha\rho^2}+\frac{ \log^2|H|\log^3\frac{1}{\rho}}{\rho^2} \right).
	% \]
\end{theorem}
\Cref{thm:rFinite} gives a quadratic improvement over the sample complexity via reduction from private learning in both confidence and accuracy, and in particular has the same asymptotic dependence as in private PAC-learning (and hence avoids any statistical blowup in the setting where $\log|H|$ is thought of as small). In fact, it's worth noting the result is tight in these parameters, as even standard PAC-learning requires the same dependencies.

\subsubsection{Algorithm}
At its core, the algorithm achieving \Cref{thm:rFinite} relies on a simple random thresholding trick. In particular, the idea is roughly to estimate the risk of each concept in the class $H$ by standard uniform convergence bounds, choose a random error threshold $v \in [OPT,OPT+\alpha]$, and finally output a random $f\in H$ with empirical error $\emprisk{S}(f)=\frac{1}{|S|}\sum\limits_{(x,y)\in S} \mathbf{1}[f(x) \neq y]$ at most $v$. Implementing this strategy requires a bit more effort, and is achieved formally by the following algorithm.

% , using the fact that $H$ is finite to get good estimations of the true risk for each concept $f \in H$. 
% Then, $\rFinite$ picks a random threshold $v < \delta$ and a random ordering of the concepts in $H$. The algorithm outputs the first concept in this random ordering that has empirical risk at most $v$. 

\begin{algorithm}[H]

\KwResult{Replicably outputs hypothesis with error at most $OPT+\alpha$}
\nonl \textbf{Input:} Finite Class $H$, Joint Distribution $D$ over $X \times \{0,1\}$ (Sample Access)\\
\nonl \textbf{Parameters:} 
\begin{itemize}
    \item Replicability, Accuracy, Confidence $\rho, \alpha, \beta>0$
    \item Sample Complexity $m=m(\rho, \alpha,\beta) \leq O\left(\frac{ \log^2|H|\log \frac{1}{\rho}+\rho^2\log\frac{1}{\beta}}{\alpha^2 \rho^4} \right)$
    \item Replicability bucket size $\tau \leq O(\frac{\alpha \rho}{ \ln |H|})$
\end{itemize}
\nonl \textbf{Algorithm:}\\
\begin{enumerate}
    \item Draw a labeled sample $S \sim D^m$ and compute $\emprisk{S}(f)$ for every $f \in H$.
	\item Replicably output initialization  $v_{\text{init}} \in [OPT,OPT+\alpha/2]$ (see \Cref{alg:agnostic-subroutine})
	\item Select random threshold $v \gets_r \{v_{\text{init}}+\frac{3}{2}\tau,v_{\text{init}}+\frac{5}{2}\tau,\ldots,v_{\text{init}}+\alpha/4-\tau/2\}$
	\item Randomly order all $f \in H$
\end{enumerate}
\textbf{return} Output the first hypothesis $f$ in the order s.t.\ $\emprisk{S}(f) \le v$.
 \caption{(Intermediate) Replicable Learner for Finite Classes}
\label{alg:finite-concept-classes}
 
\end{algorithm}
We note that Step 2, estimating OPT, follows essentially the same argument as the basic replicable statistical query algorithm of \cite{ImpLPS22}. We give the argument in \Cref{app:OPT} for completeness.

We note that while Algorithm $\rFinite$ is a replicable agnostic PAC learner, it is not quite sufficient to prove \Cref{thm:rFinite} due to its poor dependence on $\rho$. We'll see in the next section how to obtain the stated parameters by separately amplifying $\rFinite$ starting from good constant replicability.

% Before moving on, we remark that in the (near)-realizable setting when $OPT \leq O(\eps)$, we can improve the sample complexity by using a weaker convergence bounds (with better sample complexity).
% \begin{remark}\label{rem:rfinite-real}
% In the near-realizable setting, the sample complexity of Algorithm $\rFinite$ can be improved to:
% 	\[
% 		m(\rho,\eps,\delta) \leq O\left(\frac{ \log^2(\frac{|H|}{\delta\rho})}{\eps \rho^4} \right).
% 	\]
% \end{remark}

\subsubsection{Analysis}
% We will rely on a standard bound on the uniform convergence of finite classes, which is an immediate application of Chernofff and Union bounds.
% %a classical variant of uniform convergence in the near-realizable setting which is an immediate application of Chernoff and Union bounds.

% \begin{lemma}[Uniform Convergence]
% 	\label{lem:estimation-of-empirical-errors}
% 	There exists a constant $c>0$ such that for any joint distribution $D$ over $X \times Y$, any $\tau, \delta >0$, and any $n \geq c\frac{\log \frac{|H|}{\delta}}{\tau^2}$, the empirical and true risk of all hypotheses over samples of size $n$ are close with high probability:
% 	\begin{equation}\label{eq:uniform-ag}
% 	\Pr_{S \sim D}\left[\max_{f \in H} |\emprisk(f, S) - R(f)| < \tau \right] \ge 1- \delta.
% 	\end{equation}
% \end{lemma}

% The proof of both lemmas are standard, and are immediate applications of Chernoff and Union bounds.
We'll start by proving the following weaker bound for our intermediate learner.
\begin{theorem}[Intermediate Learnability of Finite Classes]\label{thm:intermediate-learner}
	Let $H$ be any finite concept class.
	Algorithm $\rFinite$ is a (proper) agnostic replicable learning algorithm for $H$ with sample complexity: 
	\[
	m(\rho,\alpha,\beta) \leq O\left(\frac{ \log^2|H|\log(\frac{1}{\rho})+\rho^2\log\frac{1}{\beta}}{\alpha^2 \rho^4} \right).
	\] 
	In the realizable setting, the $(\alpha,\beta)$-dependence can be improved to:
	\[
	m(\rho,\alpha,\beta) \leq O\left(\frac{ \log^2|H|\log(\frac{1}{\rho})+\rho^4\log\frac{1}{\beta}}{\alpha\rho^4} \right).
	\] 
% 	Moreover, if the adversary is restricted to choosing distributions with $OPT \leq O(\eps)$, this can be improved to $O(\frac{\log (H) \log(\frac{|H|}{\delta\rho})}{\eps \rho} )$.
\end{theorem}
The main challenge in \Cref{thm:intermediate-learner} is proving replicability. (Accuracy and failure probability are essentially immediate from standard uniform convergence arguments.) To this end, note that the randomness $r$ used by $\rFinite$ is largely broken into three parts: estimating OPT, choosing a random threshold, and ordering the concepts in $H$. We'll focus first on the latter two, where the choice of $v$ restricts $H$ to two subsets $H_1$ and $H_2$ (those with empirical error at most $v$), depending on input samples $S_1$ and $S_2$. We first appeal to the classical observation of Broder \cite{Broder97} to argue that as long as the symmetric difference of $H_1$ and $H_2$ are small, outputting the first concept from these sets (according to the random ordering) is a replicable procedure.

\begin{observation}
	\label{obs:random-ordering-of-concepts}
	Let $O(H, r)$ be a random ordering of concept class $H$. Let $\emptyset \subset H_1, H_2 \subseteq H$, and let $f_1$ and $f_2$ be the first elements of $H_1$ and $H_2$ respectively according to $O(H, r)$. 
	Then $\Pr_{r} [f_1 \ne f_2] 
	= \frac{|H_1 \Delta H_2|}{|H_1 \cup H_2|}$, where $\Delta$ denotes the symmetric difference.
\end{observation}

% With high probability, the empirical risk of each concept $f \in C$ is estimated within error $\tau$. 

The key to proving replicability is then to observe that most choices of $v$ induce small symmetric difference between the corresponding $H_1$ and $H_2$. Namely, the idea is to observe that for any fixed joint distribution $D$, intervals
\[
I_0 = [OPT,OPT+\tau], \ldots , I_{\alpha/(2\tau)} = [OPT+\alpha/2-\tau, OPT+\alpha/2],
\]
and corresponding threshold positions $v_i=OPT+\frac{(2i+1)}{2}\tau$, the sets
\[
H^{(i)}_1 = \{h \in H: \emprisk{S_1}(h) \leq v_i\}, \quad H^{(i)}_2 = \{h \in H: \emprisk{S_2}(h) \leq v_i\}
\]
are close for most choices of $v_i$, $S_1$, and $S_2$. To adjust for the fact that we don't know the value of OPT, we will in fact prove something slightly more general that allows our starting point to range anywhere from $OPT$ to $OPT+\alpha/2$.
\begin{lemma}\label{lem:good-thresholds}
Let $v_{\text{init}} \in [OPT,OPT+\alpha/2]$ and $\tau \leq O\left(\frac{\alpha\rho^2}{\log|H|}\right)$ a parameter that divides $\alpha/4$. Define the intervals
\[
I_0 = [v_{\text{init}},v_{\text{init}}+\tau), \ I_1 = [v_{\text{init}}+\tau,v_{\text{init}}+2\tau), \ \ldots \ , \ I_{\frac{\alpha}{4\tau}} = \left[v_{\text{init}}+\frac{1}{4}\alpha - \tau,  v_{\text{init}}+\frac{1}{4}\alpha\right]
\]
and corresponding thresholds $v_i=v_{\text{init}}+\frac{(2i+1)}{2}\tau$, and let
\[
H^{(i)}_1 = \{h \in H: \emprisk{S_1}(h) \leq v_i\}, \quad H^{(i)}_2 = \{h \in H: \emprisk{S_2}(h) \leq v_i\}
\]
denote the hypotheses with empirical error at most $v_i$ across two independent samples $S_1$ and $S_2$ of size $O(\frac{\log \rho^{-1}}{\tau^2})$. Then with probability at least $1-\rho/4$, a uniformly random choice of $i \in [\frac{\alpha}{4\tau}]$ satisfies:
\[
\frac{|H^{(i)}_1 \Delta H^{(i)}_2|}{|H^{(i)}_1 \cup H^{(i)}_2|} \leq \rho/4.
\]
\end{lemma}
\begin{proof}
For convenience of notation, let $|I_i|$ denote the number of hypotheses whose true risk lies in interval $I_i$, and $|I_{[i]}|$ the number of hypotheses in intervals up through $I_i$. We call a threshold $v_i$ ``bad'' if any of the following conditions hold.
\begin{enumerate}
    \item The $i$th interval has too many elements:
    \[
    |I_i| > \frac{\rho}{30}|I_{[i-1]}|.
    \]
    \item The number of elements beyond $I_i$ increases too quickly: 
    \[
    \exists j \geq 1: |I_{i+j}| \geq e^{j}|I_{[i-1]}|.
    \]
%     \item The number of elements before $I_i$ increases too quickly:
% \[
%     \exists j \geq 1: |I_{i-j}| \geq 2^{j}|I_{[i-1]}|
%     \]
%     \maxh{fixing}
\end{enumerate}
and ``good'' otherwise. We will argue the following two claims. 
\begin{enumerate}
    \item If $v_i$ is a good threshold, then $H^{(i)}_1$ and $H^{(i)}_2$ are probably close
\[
\underset{S_1,S_2}{\Pr}\left[\frac{|H^{(i)}_1 \Delta H^{(i)}_2|}{|H^{(i)}_1 \cup H^{(i)}_2|} \leq \frac{\rho}{4}\right] \geq 1-\frac{\rho}{8}.
\]
\item At most a $\frac{\rho}{8}$ fraction of thresholds are bad.
\end{enumerate}
Since we pick a threshold uniformly at random, it is good with probability at least $1-\rho/8$ and a union bound gives the desired result.

It remains to prove the claims. For the first, observe that for any fixed hypothesis $h$ with true risk $\distrisk{D}(h) \in I_{i+j}$, the probability that the empirical risk of $h$ is less than $v_i$ is at most
\begin{equation}\label{eq:tail-bound}
\Pr[\emprisk{S}(h) \leq v_i] \leq e^{-\Omega(j^2\tau^2|S|)}
\end{equation}
by a Chernoff bound. Let $x_i$ denote the variable which counts the number of hypotheses with true risk beyond $I_i$ that cross the threshold $v_i$ empirically. If $v_i$ is ``good,'' we can bound $\mathbb{E}[x_i]$ by
\[
\mathbb{E}[x_i] \leq |I_{[i-1]}|\sum\limits_{j > 0} e^{-\Omega(j^2\tau^2|S| - j)} \leq \frac{\rho^2}{2000}|I_{[i-1]}|
\]
for our choice of $|S|$. Markov's inequality then promises 
\[
\Pr\left[x_i \geq \frac{\rho}{30}|I_{[i-1]}|\right] \leq \frac{\rho}{64}.
\]
On the other hand, the probability any hypothesis in $I_{[i-1]}$ crosses $v_i$ is at most $e^{-\Omega(\tau^2|S|)}$, so similarly the probability that more than a $\frac{\rho}{30}$ fraction of such hypotheses cross $v_i$ is at most $\frac{\rho}{64}$. Finally, since $v_i$ is `good,' $I_i$ itself contributes at most $\frac{\rho}{30}|I_{[i-1]}|$ hypotheses that cross the threshold in the worst case, so in total we have that with probability at least $1-\frac{\rho}{32}$, at most $\frac{\rho}{10}|I_{[i-1]}|$ hypotheses cross the threshold in either direction. Considered over two runs of the algorithm, this implies that with probability at least $1-\frac{\rho}{16}$, $|H^{(i)}_1 \Delta H^{(i)}_2|$ cannot be too big
\[
|H^{(i)}_1 \Delta H^{(i)}_2| \leq \frac{\rho}{5}|I_{[i-1]}|.
\]
Furthermore, since the probability that more than a $\frac{\rho}{30}$ fraction of hypotheses in $I_{[i-1]}$ cross $v_i$ is at most $\frac{\rho}{64}$, we also have that $|H^{(i)}_1 \cup H^{(i)}_2|$ cannot be too small:
\[
|H^{(i)}_1 \cup H^{(i)}_2| \geq \left(1-\frac{\rho}{15}\right)|I_{[i-1]}|
\]
with probability at least $1-\frac{\rho}{64}$. Thus altogether a union bound gives
\[
\underset{S_1,S_2}{\Pr}\left[\frac{|H^{(i)}_1 \Delta H^{(i)}_2|}{|H^{(i)}_1 \cup H^{(i)}_2|} \leq \frac{\rho}{4} \right] \geq 1-\frac{\rho}{8}
\]
as desired.

Finally, we need to show that almost all thresholds are good. To see this, first observe that since $v_{\text{init}} \geq OPT$, $|I_{[i]}|>0$ for all $i \geq 0$. To count the number of bad thresholds, let $i_1 \geq 1$ be the position of the first bad threshold, and $t_1$ denote the largest index such that $i_1+t_1$ fails a condition. Define $i_j$ and $t_j$ recursively as the first bad threshold beyond $i_{j-1}+t_{j-1}$ and its corresponding latest failure. Observe that by construction, any interval that does not lie in any $[i_j.i_j+t_j]$ is good, so there are at most $\sum t_j$ bad thresholds. 

Let $\ell$ denote the final index of the above greedy process. By definition of a bad interval, each $t_j$ multiplicatively increases the number of total hypotheses from $I_{[i_j]}$ by at least $\left(1+\frac{\rho}{30}\right)^{t_j}$. Since $|I_{0}|\geq 1$ and the total number of hypotheses is $|H|$ by definition, we may therefore write:
\[
|H| \geq |I_{[i_\ell+t_\ell]}| \geq \left(1+\frac{\rho}{30}\right)^{\sum\limits_{j=1}^\ell t_j}
\]
and thus that the total number of bad intervals is at most
\[
\sum_{j=1}^\ell t_j \leq O\left(\frac{\log(|H|)}{\rho}\right).
\]
Since we have chosen $\tau$ such that the total number of intervals altogether is at least $\Omega\left(\frac{\log(|H|)}{\rho^2}\right)$, the appropriate choice of constant gives that at most a $\rho/8$ fraction are bad as desired.
\end{proof}
To complete the argument, it is enough to show we can find a good starting point $v_{\text{init}}$.
\begin{lemma}
There exists a $\rho$-replicable algorithm over $O\left(\frac{\log(\frac{|H|}{\rho\beta})}{\rho^2\alpha^2}\right)$ samples that outputs a good estimate of $OPT$ with high probability:
\[
\Pr_{r,S}\big[\mathcal{A}(S) \in [OPT,OPT+\alpha/2]\big] \geq 1-\beta
\]
\end{lemma}
Proving this Lemma largely follows from prior techniques but is a bit tedious, so we leave the proof for \Cref{app:OPT}. With these tools in hand, we are finally ready to prove \Cref{thm:rFinite}.
\begin{proof}[Proof of \Cref{thm:intermediate-learner}]
We start by showing $\rFinite$ is $\rho$-replicable. $\rFinite$ starts by running a replicable subroutine (with parameters $\rho'=\rho/2$ an $\beta'=\beta/2$) to find an estimate for OPT. Using new (independent) randomness, it then selects a threshold $v_i$ and a random ordering over $H$, and outputs the first hypothesis in $H^{(i)} = \{h: R_{\text{emp}}(h,S) \leq v_i\}$. By \Cref{lem:good-thresholds} and \Cref{obs:random-ordering-of-concepts}, this latter process is $\rho/2$-replicable. By composition of replicability, the entire algorithm is therefore $\rho$-replicable as desired.

Correctness of $\rFinite$ follows from standard uniform convergence type arguments. 
% Recall that there exists a constant $c>0$ such that for any joint distribution $D$ over $X \times \{0,1\}$, any $\tau, \beta >0$, and any $n \geq c\frac{\log \frac{|H|}{\beta}}{\tau^2}$, the empirical and true risk of all hypotheses in $H$ over samples of size $n$ are close with high probability:
% 	\begin{equation*}\label{eq:uniform-ag}
% 	\Pr_{S \sim D}\left[\max_{f \in H} |\emprisk(f, S) - R(f)| < \tau \right] \ge 1- \beta.
% 	\end{equation*}
In particular, by our choice of $|S|$, any hypothesis with empirical risk at most $OPT+\alpha/2$ has true risk less than $OPT+\alpha$ with probability at least $1-\beta/2$. Furthermore, as long as our estimation of $OPT$ is successful (which occurs with probability at least $1-\beta/2$), we always output such a hypothesis. Thus altogether we output a hypothesis with true error at most $OPT+\alpha$ with probability at least $1-\beta$ as desired. 

Finally, we need to argue that the dependence on $\alpha$ can be improved to linear in the realizable setting. Note that in this case, we can simply set $v_{\text{init}}$ to 0, and ignore the estimation of OPT. The improvement then follows immediately from noting that when $OPT=0$, a standard Chernoff bound improves \Cref{eq:tail-bound} to
\[
\Pr[\emprisk{S}(h) \leq v_i] \leq e^{-\Omega(\frac{j^2\tau|S|}{(i+j)})} \leq e^{-\Omega(\frac{j^2\tau^2|S|}{\alpha})}.
\]
Similarly, only $O(\frac{\log\frac{|H|}{\beta}}{\alpha})$ examples are needed to ensure hypotheses with $O(\alpha)$ empirical risk have $O(\alpha)$ true risk with high probability, and the rest of the proof follows as in the agnostic case. 
\end{proof}
Finally, we amplify the above to prove \Cref{thm:rFinite}.
\begin{proof}[Proof of \Cref{thm:rFinite}]
Our amplification algorithm is a modification of the original technique introduced in \cite{ImpLPS22}, designed to take advantage of the fact that the dependence on $\beta$ (failure) and $\rho$ (replicability) are highly unbalanced in learning tasks. Draw $k=O(\log(1/\rho))$ random strings $r_1,\ldots,r_k$, and consider the distributions $\{D_i\}_{i=1}^k$ generated by running $\rFinite(r_i,S)$ with parameters $\rho'=.01$ and $\beta'=\beta\cdot \text{poly}(\rho)$ on a large enough sample $S$.
The idea is to argue that with good probability over the choice of random strings $r$, at least one of these distributions has an $\Omega(1)$-heavy-hitter (which is also a good hypothesis with extremely high probability). Roughly speaking, we can then use the heavy hitters algorithm of  \cite{ImpLPS22} across these distributions to $\rho/2$-replicably output a good hypothesis, and union bound over all applications to argue correctness of the final output.

Let's formalize this argument. First, observe since our setting of $\rFinite$ is $.01$-replicable, at least $90\%$ of the random strings have a `canonical element,' i.e.\ one that appears across at least $90\%$ of random samples. Call such strings good, and observe that any good string $r_i$ corresponds to a distribution $\mathcal{D}_i$ with a $.9$-heavy-hitter by construction. Over a random choice of $O(\log \rho^{-1})$ such strings, the former guarantee then promises at least one of these strings is good with probability greater than $1-\rho/4$ and therefore that at least one distribution in $\{D_i\}_{i=1}^k$ has a $.9$-heavy-hitter. With this in mind, we now appeal to the heavy-hitter algorithm of \cite{ImpLPS22}, which draws $O(\frac{\log \rho^{-1}}{\rho^2})$ samples from a distribution to replicably output a list of all $\Omega(1)$-heavy-hitters.\footnote{We note the technique actually outputs a list of some weight close to $c$, but this is largely irrelevant in our setting where $c$ (and the shift in $c$) are constant.} To make our entire process $\rho$-replicable, we will run the above process for $\rho'=O(\frac{\rho}{\log\rho^{-1}})$. To this end, we draw samples $S_1,\ldots, S_{t}$ for $t =O(\frac{\log^3 \rho^{-1}}{\rho^2})$ and generate $t$ corresponding samples from each $\{D_i\}_{i=1}^k$ (re-using $S_j$ between distributions), which we use to run \cite{ImpLPS22}'s heavy-hitters algorithm. Union bounding over all applications, this process is $\rho/4$-replicable, and with probability at least $1-\rho/4$ outputs a non-empty list of hypotheses. Finally, we break in to one of two cases. If the result list is indeed non-empty, simply output a random element of the list (a fully replicable procedure). Otherwise, run $\rFinite$ on a fresh random string and sample, and output the result.

Finally, we argue replicability and correctness of the above process. First, note the output list is non-empty with probability at least $1-\rho/4$, and two independent samples produce the same list (with fixed randomness) with probability at least $1-\rho/2$ by replicability of the repeated heavy hitter process discussed above. Therefore the entire process is $\rho$-replicable as desired. For correctness, note that we have used at most $\poly(\rho^{-1})$ instances of $\rFinite$. Recall that we set the failure probability of $\rFinite$ to be very small, with $\beta'=\beta\cdot \text{poly}(\rho^{-1})$. Since each individual application of $\rFinite$ fails with probability less than $\beta'$, union bounding over all applications of the algorithm implies every output hypothesis is `good' (within $\alpha$ of OPT) with probability at least $1-\beta$. Since we only output hypotheses generated by this process, the probability of outputting a bad hypothesis is then less than $\beta$ as desired.

Altogether, the sample complexity of the above algorithm is given by the size of samples 
\[
t|S_i| = O\left(\frac{\log^3 \rho^{-1}}{\rho^2}\cdot \frac{ \log^2|H|+\log\frac{1}{\rho\beta}}{\alpha^2}\right),
\]
or 
\[
O\left(\frac{\log^3 \rho^{-1}}{\rho^2}\cdot\left(\frac{\log|H| + \log\frac{1}{\rho\beta}}{\alpha}+\log^2|H|\right)\right)
\]
in the realizable case.
\end{proof}

\section{Applications} \label{sec:applications}
In this section we take advantage of our reductions between notions of stability to resolve (or otherwise make progress on) several open problems in the algorithmic stability literature.

\subsection{Item-level to User-level Privacy Transformation} \label{sec:item-user}

The original motivation of introducing the definition of pseudo-global stability in \cite{GhaziKM21} was to come up with PAC learning algorithms in the example-rich, user-level privacy setting. In this setting, there are a number of users with many samples (the regime we will be interested in is when there are a few users who have enough samples to solve the problem for themselves). The motivation behind this setting is to leverage the data of example-rich users to obtain statistical insights without compromising their privacy. Such statistical analyses could then be released and used widely, even by users who didn't have as much data. 

The technique of \cite{GhaziKM21} involves coming up with pseudo-globally stable algorithms for PAC learning, and then having each user run a pseudo-globally stable algorithm with the same coins. Then, you can privately identify a heavy hitter among the outputs of the users (such a heavy hitter exists with high probability because of the property of pseudo-global stability), and since changing an entire user's sample will affect only one output, this procedure will be user-level differentially private. One open question raised in their paper was whether their techniques could be extended beyond the PAC setting.

Our argument that pseudo-global stability and differential privacy are two sides of the same coin answers this question in the affirmative, and has the additional benefit of eliminating the need to cleverly design pseudo-globally stable (replicable) algorithms. We showed previously that item-level differentially private algorithms can be compiled into replicable algorithms with only a quadratic overhead in sample complexity (See Sections~\ref{sec:dp2PG} and~\ref{sec:PG2rep}). Hence, our results allow for a general transformation from item-level to user-level privacy for statistical tasks in the example-rich setting; each user applies correlated sampling to the same item-level differentially private algorithm applied to their specific sample, and then a heavy hitter is identified via differentially private selection.
\begin{theorem}
There are universal constants $c,K >0$ such that the following holds. Let $\mathcal{T}$ be a statistical task with a finite output space. Given a $(0.1, \frac{c}{n^3})$-item level differentially private algorithm $\mathcal{A}$ that solves $\mathcal{T}$ using $n$ samples and with failure probability $\beta$, for every $1\geq \eps, \delta > 0$, there exists an $(\eps, \delta)$-user level differentially private algorithm $\mathcal{A}_u$ that solves $\mathcal{T}$ with failure probability $O(\beta \log \frac{1}{\beta})$, when given access to the data of $O(\frac{\log 1/\beta}{\eps} \log \frac{\log 1/\beta}{\delta})$ users, each of whom have at least $K n^2 \log(1/\beta)$ examples.
\end{theorem}
\begin{proof}
Firstly, we can amplify the privacy parameters of $\mathcal{A}$ by subsampling. By Lemma~\ref{lem:samp-wo-replace}, the algorithm $\mathcal{A}'$ that, given $m = Kn^2$ samples, subsamples $n$ items without replacement and runs $\mathcal{A}'$ on the result is $(1/\sqrt{Km}, c K/m^2)$-differentially private. Moreover, when run on inputs consisting of i.i.d. samples from a distribution $D$, the output of $\mathcal{A}'$ is identically distributed to that of $\mathcal{A}$, so $\mathcal{A}'$ also solves $\mathcal{T}$ with $m$ examples and failure probability $\beta$.

For a sufficiently large constant $K$ and sufficiently small constant $c$, Corollary~\ref{cor:DPtoRep} then implies that $\mathcal{A}'$ gives rise to a $c$-replicable algorithm solving $\mathcal{T}$ with $m$ examples and failure probability $\beta$.

Now, consider Algorithm~\ref{alg:Rep-to-DP} adapted to the user-level setting as follows (with number of users as specified in the theorem): instead of partitioning a centralized sample (as done in that algorithm), each of the users applies algorithm $\mathcal{A}'$ to their own set of $O(\log(1/\beta))$ i.i.d. samples (with the same set of $O(\log(1/\beta))$ different coins---this can be achieved through a common random string that they share). Then, the outputs are sent to a central server, and $(\eps' = \frac{\eps}{C\log 1/\beta}, \delta' = \frac{\delta}{2\log 1/\beta})$-DP selection is then applied to choose a heavy hitter (as discussed in Algorithm~\ref{alg:Rep-to-DP}). Let's call this algorithm $\mathcal{A}_u$.

The accuracy guarantees proved for Algorithm~\ref{alg:Rep-to-DP} give us that the failure probability of this Algorithm $\mathcal{A}_u$ is at most $O(\beta \log 1/\beta)$. Hence, we are left to argue privacy. Note that changing a single user's sample can change at most $C \log 1/\beta$ outputs to which the $(\eps', \delta')$-DP selection algorithm is applied to. Hence, by group privacy (Lemma~\ref{lem:group_privacy}), we have that $\mathcal{A}_u$ is $(\eps' \log 1/\beta, \delta' \frac{e^{\eps' \log 1/ \beta} -1}{e^{\eps'}-1})$-user level differentially private. Substituting the value of $\eps'$ and $\delta'$ then gives an $(\eps, \delta)$-user level private algorithm. This completes the proof.
\end{proof}

\subsection{Parameter Amplification for Differential Privacy and Perfect Generalization} \label{sec:delta-amp}

The equivalence between replicability and differential privacy gives us the first generic amplification theorem for the $\delta$ parameter of approximate differential privacy for general statistical tasks. Prior to our work, it was known that the $\eps$ parameter could be amplified algorithmically and efficiently. That is, using random sampling (Lemma~\ref{lem:samp-wo-replace}), one can improve an $(\eps_0, \delta_0)$-differentially private algorithm to a $(p\eps_0, p\delta_0)$-differentially private one with an $O(1/p)$ blowup in the sample complexity. However, this technique is unable to improve the $\delta$ parameter of such an algorithm asymptotically as a function of the number of samples $n$, e.g., from $\delta(n) = 1/n^{10}$ to $\delta'(n) = \exp(-n^{0.99})$.

The recent characterization of private PAC learnability in terms of Littlestone dimension~\cite{alon2019private, bun2020equivalence, ghazi2021sample} implies that such an amplification of $\delta$ is (at least, in principle) possible for private PAC and agnostic learning and for private query release. Given a target class $\mathcal{C}$ in one of these settings, the existence of a $(\eps = 0.1, \delta = O(1/n^2
\log n))$-differentially private algorithm using a finite number of samples $n$ implies that $\mathcal{C}$ has some finite Littlestone dimension $d$. This in turn implies that, for every $\eps, \delta > 0$, there is an $(\eps, \delta)$-differentially private agnostic PAC learning algorithm for $\mathcal{C}$ using $\poly(d, 1/\eps, \log(1/\delta))$ samples and a private query release algorithm for $\mathcal{C}$ using $\poly(2^{2^d}, 1/\eps, \log(1/\delta))$ samples. Unfortunately, the first part of this argument is non-constructive, and in the worst-case, leads to a final algorithm using a number of samples that is an exponential tower in $\Omega(n)$! \cite{bun2020equivalence} posed the open question of whether such amplification could be done algorithmically, even for the special case of private PAC learning. %\sstext{This open question was posed again at the Fields Workshop on Differential Privacy and Statistical Data Analysis in July 2022.}

Our approach is to first convert a differentially private algorithm with weak parameters to a replicable one. We may then use the fact that replicable algorithms can be converted back to differentially private ones with excellent privacy parameters. Altogether we obtain a constructive amplification theorem that achieves only a modest  blowup in sample complexity, and which applies to general statistical tasks with finite output spaces.

\begin{theorem}
    There is a universal constant $c > 0$ such that the following holds. Let $\mathcal{T}$ be a statistical task with a finite output space. Suppose there is an $(\eps = 0.1, \delta = c/n^3)$-differentially private algorithm that solves $\mathcal{T}$ using $n$ samples and with failure probability $\beta$. Then for every $\eps, \delta > 0$, there exists an $(\eps, \delta)$-differentially private algorithm solving $\mathcal{T}$ using
    \[O\left(\frac{\log(1/ \delta)\log(1/\beta)}{\eps}+\log^2(1/\beta)\right) \cdot n^2\]
    samples and with failure probability $O(\beta \log 1/\beta)$.
\end{theorem}

\begin{proof}
Let $A$ be a $(0.1, c/n^3)$-differentially private algorithm solving $\mathcal{T}$ with $n$ samples and failure probability $O(1/\log(1/\beta))$. By Lemma~\ref{lem:samp-wo-replace}, the algorithm $A'$ that, given $m = Kn^2$ samples, subsamples $n$ items without replacement and runs $A$ on the result is $(1/\sqrt{Km}, cK/m^2)$-differentially private. Moreover, when run on inputs consisting of i.i.d. samples from a distribution $P$, the output of $A'$ is identically distributed to that of $A$, so $A'$ also solves $\mathcal{T}$ with $m$ samples and failure probability $\beta$.

For a sufficiently large constant $K$ and sufficiently small constant $c$, \Cref{cor:DPtoRep} implies that $A'$ gives rise to a $c$-replicable algorithm solving $\mathcal{T}$ with $m$ samples and failure probability $\beta$. Applying \Cref{thm:Rep-to-DP} thus results in an $(\eps, \delta)$-differentially private algorithm solving $\mathcal{T}$ with 
\[
m \cdot O\left(\frac{\log(1/ \delta)\log(1/\beta)}{\eps}+\log^2(1/\beta)\right)
\]
samples and failure probability $O(\beta \log 1/\beta)$.
\end{proof}

We can also show a similar amplification of the parameters for perfect generalization, indeed, even from one-way perfect generalization to perfect generalization itself.

\begin{theorem} \label{thm:opg-to-pg}
There is a universal constant $c > 0$ such that the following holds. Let $\mathcal{T}$ be a statistical task with a finite output space. Suppose there is an $(\beta' = 10^{-5}, \eps = 10^{-5}, \delta = 10^{-5})$-one-way perfectly generalizing algorithm that solves $\mathcal{T}$ using $n$ samples and with failure probability $\beta$. Then for every $\eps, \delta > 0$, there exists an $(\delta, \eps, \delta)$- perfectly generalizing algorithm solving $\mathcal{T}$ using
    \[O\left(n \cdot \frac{\log(1/\eps)\poly\log(1/\delta)}{\eps^2}\right)\]
    samples and with failure probability $O(\delta) + \sqrt{\beta}\log(1/\delta)$.
\end{theorem}

\begin{proof}
Let $A$ be a $(10^{-5}, 10^{-5}, 10^{-5})$-perfectly generalizing algorithm solving $\mathcal{T}$ with $n$ samples and failure probability $\beta$. By Lemma~\ref{PGtoRep}, there is a $(0.0001)$-replicable algorithm solving $\mathcal{T}$ with the same failure probability and the same number of samples. By Claim~\ref{claim:2parrep} and Theorem~\ref{thm:reprodtoPG}, there exists a $(\delta^2 / 4, \eps, \delta^2 / 4)$-sample perfectly generalizing algorithm solving $\mathcal{T}$ using $O(n \cdot \frac{\log(1/\eps)\poly\log(1/\delta)}{\eps^2})$ samples with failure probability $O(\delta) + \sqrt{\beta}\log(1/\delta)$. Lemma~\ref{lem:samplePGtoregPG} converting sample perfect generalization to perfect generalization gives the result.
\end{proof}

\subsection{A Realizable-to-Agnostic Reduction for Structured Distributions} \label{sec:realizable-agnostic}
One of the main running examples throughout this work (and indeed a focal point in \cite{ImpLPS22,GhaziKM21} as well) is the PAC-learning paradigm. Traditionally, PAC-learning has two main settings, \textit{realizable} learning (where the adversary must choose a hypothesis in the class), and the \textit{agnostic} setting (which allows an arbitrary adversary). It is a well known fact in the study of traditional statistical learning that realizable and agnostic learning are equivalent up to polynomial blowup in sample complexity \cite{vapnik1974theory,Blumer,haussler1992decision}. Furthermore, an analog of this fact holds for most supervised paradigms (see e.g.\ \cite{benedek1991learnability,bartlett1996fat,long2001agnostic,ben2009agnostic,david2016supervised,montasser2019vc,alon2021theory}), including privacy \cite{beimel2013private,alon2020closure}. This was originally shown by Beimel, Nissim, and Stemmer \cite{beimel2013private}, who gave a sample-efficient agnostic-to-realizable reduction for approximately differentially private PAC-learning. Their result has since been used extensively (see e.g. \cite{bun2016simultaneous,beimel2019private,alon2019limits,alon2020closure,bun2020equivalence}), and is often used to justify focus on the realizable setting.

While certainly impactful, Beimel, Nissim, and Stemmer's reduction (and later improvements on the same \cite{alon2020closure,bun2020equivalence}) are complicated and limited in application. Like the results that came before them (in the traditional setting), their techniques rely heavily on \textit{uniform convergence}, and therefore always incur a cost in VC dimension of the class (or analogously in $\log|H|$ in many settings we consider). Such bounds are typically only useful in the \textit{distribution-free} setting, where the adversary is free to choose arbitrary (often strange, combinatorial) distributions over the data that don't appear in practice. Outside of such cases, it is typically possible to learn in many fewer samples than VC dimension would predict (see e.g.\ \cite{nagarajan2019}), so it is reasonable to ask whether this efficiency can be generically maintained in the agnostic setting. Towards this end, Hopkins, Kane, Lovett, and Mahajan \cite{hopkins2021realizable} recently gave a more generic reduction independent of VC dimension, but their techniques do not adapt directly to the private setting, which was left as an open problem in their work.

We resolve this problem (at least in finite domains) via reduction to and from replicability: agnostic private learning requires only a small polynomial blowup over the realizable case that is independent of class-size, even under arbitrary distributional assumptions. With this in mind, we briefly introduce the \textit{distribution-family} variant of the PAC-model, which first appeared (implicitly) in seminal work of Benedek and Itai \cite{benedek1991learnability} on learning under fixed distributions. We highlight the differences from the standard model below.
\begin{definition}[Distribution-Family Model \cite{benedek1991learnability}]
A learning problem is defined by a hypothesis class $H$ \textbf{and family of distributions $\pmb{\mathscr{D}}$} over the instance space $\mathcal{X}$. We say an algorithm $\mathcal{A}$ is an $(\alpha, \beta)$-accurate Agnostic learner for the hypothesis class $(H,\mathscr{D})$ if for all distributions $D$ over input, output pairs \textbf{whose marginal $\pmb{D_{\mathcal{X}} \in \mathscr{D}}$}, $\mathcal{A}$ on being given a sample of size $m$ drawn i.i.d. from $D$ outputs a hypothesis $h$ such that with probability greater than or equal to $1 - \beta$ over the randomness of the sample and the algorithm,
\begin{equation*}
    \err_D(h) \leq \inf_{f \in H} \err_D(f)+\alpha,
\end{equation*}
where $\err_D(h) = \Pr_{(x,y) \in D} [h(x) \neq y]$. When the adversary is additionally restricted to choosing $D$ s.t.\ $\inf_{f \in H} \err_D(f)=0$, we call the problem \textbf{realizable}.
\end{definition}
We give the first private agnostic-to-realizable reduction from the distribution-family model, and in general the first reduction with no reliance on uniform convergence or VC dimension.
\begin{theorem}\label{thm:agn-to-real}
Let $(H,\mathscr{D})$ be a hypothesis class that is $(1,\frac{1}{poly(n)})$-privately $(\alpha, 0.01)$-PAC learnable in $n=n(\alpha)$ samples in the realizable setting. Then $(H,\mathscr{D})$ is $(1,\frac{1}{poly(m)})$-privately $(\alpha, \beta)$-agnostically learnable in
\[
m(\alpha,\beta) \leq O\left( \frac{n^2\alpha^2+\log^3(\Pi_H(cn^2))}{\alpha^2}\log\beta^{-1}\log\left(\frac{n\log\beta^{-1}}{\alpha}\right)\right)
\]
samples for some universal constant $c>0$. 
\end{theorem}
In the statement of Theorem~\ref{thm:agn-to-real},  $\Pi_H(n)$ denotes the growth function of $(X,H)$. The growth function captures the maximum number of  labelings functions from $H$ can induce on samples of size $n$, and is at most $n^{O(d)}$ for classes with VC dimension $d$.  Moreover, since $\Pi_H(c n^2) \leq 2^{c n^2}$, this means agnostic learning experiences at most a polynomial blowup over the realizable setting:
\begin{corollary}
Let $(H,\mathscr{D})$ be a hypothesis class that is $(1,\frac{1}{poly(n)})$-privately $(\alpha, 0.01)$-PAC learnable in $n=n(\alpha)$ samples in the realizable setting. Then $(H,\mathscr{D})$ is $(1,\frac{1}{poly(m)})$-privately $(\alpha, \beta)$-agnostically learnable in
\[
m(\alpha,\beta) \leq \tilde{O}\left( \frac{n^6\log\beta^{-1}}{\alpha^2}\right)
\]
samples.
\end{corollary}
At a high level, the proof of this result is (comparatively) simple. Given a realizable private learner $\mathcal{A}$, we will transform $\mathcal{A}$ into a replicable learner, apply a variant of \cite{hopkins2021realizable}'s agnostic-to-realizable reduction, and finally lift the resulting agnostic learner back to differential privacy. The main challenge lies in adapting \cite{hopkins2021realizable} to the replicable setting, which is roughly done via the following procedure (see \Cref{alg:agnostic-subroutine}):
\begin{enumerate}
    \item \textbf{Sample:} Draw an unlabeled sample $S \sim D^n$, and random string $r$
    \item \textbf{Generate Candidates:} Run $\mathcal{A}$ on all labelings of $S$ with internal randomness $r$
    \item \textbf{Prune:} Using fresh samples, remove any high error candidates
\end{enumerate}
The result then follows from replicably outputting a heavy hitter of this procedure.

\subsubsection{List Heavy-Hitters} \label{sec:list-heavy-hitters}
It is this final step, it turns out, that contains most of the subtlety in this reduction. While estimating heavy hitters of a given distribution is a core subroutine in many replicable algorithms (used, e.g.,\ in \Cref{thm:rFinite}), and was studied in \cite{ImpLPS22}, our setting is more challenging since our goal is to output a heavy hitter from a distribution over \textit{lists}, where the list size may be exponential in the desired parameters. In this section, we show how similar arguments used for our efficient finite learner can also be used to output a heavy hitter with cost only \textit{polylogarithmic} in the list size.\footnote{We note that a similar result also appears implicitly in \cite[Theorem 20]{GhaziKM21}, albeit with worse sample complexity.}

More formally, let $\Omega$ be a finite set, and $\mathcal{D}$ a distribution over subsets of $H$. We call $h \in H$ an $\eta$-heavy-hitter of $\mathcal{D}$ if
\[
\Pr_{S \sim \mathcal{D}}[h \in S] \geq \eta.
\]
We prove it is possible to replicably output a heavy hitter with complexity scaling that is only polylogarithmic in the largest set supported by $\mathcal{D}$.
\begin{theorem}\label{thm:list-heavy-hitters}
For any finite set $\Omega$, $\rho,\eta,\beta>0$, and distribution $\mathcal{D}$ over subsets of $\Omega$ with an $\eta$-heavy-hitter, there exists a $\rho$-replicable algorithm $\mathcal{A}$ with the following guarantees:
\begin{enumerate}
    \item $\mathcal{A}$ outputs a $\frac{\eta}{2}$-heavy-hitter with probability at least $1-\beta$
    \item $\mathcal{A}$ uses at most $O\left(\frac{\log^2\frac{|\mathcal{D}|\log\frac{1}{\rho\beta}}{\eta}\log\frac{1}{\rho}+\log \frac{1}{\rho\beta}}{\eta^2\rho^2}\log^3\frac{1}{\rho}\right)$ samples from $\mathcal{D}$,
\end{enumerate}
where $|\mathcal{D}|$ is the maximum size subset supported by $\mathcal{D}$.
\end{theorem}
We note that it is easy to modify this result to remove the assumption that $\mathcal{D}$ has a heavy hitter (the algorithm instead outputs `$\bot$' in this case, or can test for the heaviest element), but the simpler version above is sufficient for our applications. We now give the algorithm itself, which combines \cite{ImpLPS22}'s heavy hitters with our thresholding technique for finite learning.

\begin{algorithm}[H]

\KwResult{Replicably outputs a heavy hitter}
\nonl \textbf{Input:} Distribution $\mathcal{D}$ over subsets of universe $\Omega$ (Sample Access)\\
\nonl \textbf{Parameters:} 
\begin{itemize}
    \item Replicability, confidence, and heaviness $\rho,\beta, \eta > 0$
    \item Sample sizes $t_1 = O\left(\frac{\log(\frac{|\mathcal{D}|}{\rho\beta})}{\eta}\right)$, $t_2 = O\left(\frac{\log^2\frac{|\mathcal{D}|\log\frac{1}{\rho\beta}}{\eta}\log\frac{1}{\rho}+\log \frac{1}{\beta}}{\eta^2\rho^4}\right)$
    \item Threshold accuracy $\tau \leq O(\frac{\rho\eta}{\log|\mathcal{D}|})$
\end{itemize}
\nonl \textbf{Algorithm:}\\
\begin{enumerate}
    \item Sample $t_1$ subsets $C \sim \mathcal{A}$, and call their union $T$.
    \item Sample an additional $t_2$ subsets $C \sim \mathcal{A}$, and call their collection $S = \{C_i\}$.
    \item For each $t \in T$, let $\hat{p}_t$ denote its empirical measure over $S$:
    \[
    \hat{p}_t = \frac{1}{|S|}|\{ C \in S: t \in C\}|.
    \]
    \item Choose a random threshold $v \in \{\eta/4 + 2\tau, \eta/4+6\tau, \ldots 3\eta/4-2\tau\}$
    \item Randomly order $\Omega$
\end{enumerate}
\textbf{return} first $t \in T$ with respect to the order satisfying $\hat{p}_t \geq v$
 \caption{List Heavy Hitters}
 \label{alg:list-amplification}
 
\end{algorithm}
It is not hard to see this algorithm succeeds via the same analysis as for \Cref{thm:rFinite}.
\begin{proof}
By a Chernoff and union bound, we first note that with probability at least $1-\beta\rho/4$, $T$ contains every $\eta$-heavy-hitter of $\mathcal{D}$. 

Similar to the proof of \Cref{thm:rFinite}, we consider intervals of the form
\[
I_0 = [\eta/4,\eta/4+4\tau], \ldots , I_{1/(2\tau)} = \left[\frac{3\eta}{4}-4\tau, \frac{3\eta}{4}\right],
\]
with corresponding threshold positions $v_i=\eta/4+\frac{(2i+1)}{2}\tau$, and the sets
\[
H^{(i)}_1 = \{t \in T: \hat{p}(t,S_1) \leq v_i\}, \quad H^{(i)}_2 = \{t \in T: \hat{p}(t,S_2) \leq v_i\}.
\]
In the proof of \Cref{thm:rFinite}, we argued that replicability followed from bounding the quantity 
\[
\frac{|H^{(i)}_1 \Delta H^{(i)}_2|}{|H^{(i)}_1 \cup H^{(i)}_2|}
\]
with high probability, as this promised that choosing the first element from a joint random ordering of $\Omega$ usually gives the same answer over $H^{(i)}_1$ and $H^{(i)}_2$. Here we need to be slightly more careful, in that we need to ensure not only that the same element is chosen, but also that it is truly an $\eta/4$-heavy-hitter. This ensures replicability despite the fact that our set $T$ depends on samples, because we are promised that all $\eta/4$-heavy-hitters lie in $T$ except with probability $1-\frac{\beta\rho}{4}$ (and therefore have no dependence $T$ itself).

Thankfully, this is already implicit in the proof of \Cref{lem:good-thresholds}, since it is actually proved that, with high probability, the number of elements of $T$ that cross threshold $v_i$ is at most $O(\rho|I_{[i-1}]|)$, where we recall $|I_{[i-1]}|$ denotes the number of elements with true weight in buckets $I_1,\ldots,I_{i-1}$. This followed from the fact that any element $t \in T$ whose true weight lay in the $j$th interval for $j>i$ satisfied:
\[
\Pr[\hat{p}(t,S_j) \leq v_i] \leq e^{-\Omega(j^2\tau^2|S_j|)},
\]
which remains true in this setting for our choice of $|S|$ by Chernoff. As such, our full process remains $\rho$-replicable for the correct choice of constants as desired. Furthermore, correctness holds with probability at least $1-\beta$, since all weight estimates are correct up to $\eta/4$. Finally, to get the correct dependence on $\rho$ we simply apply the amplification technique used in the proof of \Cref{thm:rFinite}.
\end{proof}
We note that this result is similar to the pseudo-globally stable learner of \cite{GhaziKM21}, which also uses a method of replicably finding a heavy hitter from a distribution on lists. Their algorithm uses a variant of the exponential mechanism instead of random thresholding, and loses polynomial factors over our bound as a result. Both \cite{GhaziKM21} and our algorithm have the downside of only working over finite universes (or more generally in settings where correlated sampling is possible). On the other hand, the problem can be solved \textit{privately} without this assumption. This raises a natural question: does list heavy hitters give an exponential separation between privacy and replicability over infinite domains?\footnote{Recall the problem can be solved replicably with $\poly(|\mathcal{D}|)$ dependence by \cite{ImpLPS22} even in the infinite setting.}

% We note that in the finite setting of the last section, list amplification can be applied as a blackbox but loses a log factor in the number of samples due to the implementation cost of $\mathcal{A}$.
Before moving on, we note the following immediate implication of list heavy hitters for learning classes with low Littlestone dimension, giving a moderate improvement over the analogous result of \cite{GhaziKM21}.
\begin{corollary}
Let $(X,H)$ be a class with Littlestone dimension $d$. Then the sample complexity of realizably replicably learning $(X,H)$ is at most:
\[
n(\rho,\alpha,\beta) \leq \tilde{O}\left(\frac{d^{12}\log^3(1/\beta)}{\alpha^2\rho^2}\right)
\]
\end{corollary}
We give the proof in \Cref{app:little}.

\subsubsection{Replicable Realizable-to-Agnostic Reduction}
We now show how to combine List Heavy-Hitters with \cite{hopkins2021realizable}'s agnostic-to-realizable technique to generalize their reduction to replicable learning.
\begin{theorem}\label{thm:agnostic}
Let $(X,H)$ be a replicably learnable class with sample complexity $n(\rho,\alpha,\beta)$, and $n'=n(1/4,\alpha/4,\beta/4)$. Then $(X,H)$ is agnostically learnable in
\[
m(\rho,\alpha,\beta) \leq O\left( \frac{\log^2(\Pi_H(n')\log\frac{1}{\rho\beta})\log\frac{1}{\rho}+\log \frac{1}{\rho\beta}}{\alpha^2\rho^2}\left(\alpha^2n' + \log\frac{\Pi_H(n')}{\beta}\right)\log^3\frac{1}{\rho}\right)
% \tilde{O}\left(\frac{n^2\Pi_H(n)^2}{\rho^2(1-\eta)^4} + \frac{\log(1/\beta)}{\rho^2\alpha^2}\right)
\]
samples, where $\Pi_H(n)$ is the growth function of $(X,H)$.
\end{theorem}
Since $\Pi_H(n)\leq 2^n$, this means agnostic learning experiences only a small polynomial blow-up in sample complexity compared to the easier realizable setting. Furthermore, this bound holds even with distributional assumptions, since it does not rely on external quantities such as VC-dimension.

% In the distribution-free setting, learnable classes have finite VC dimension so we immediately get the following corollary by Sauer-Shelah-Perles.
% \begin{corollary}
% Let $(X,H)$ be a reproducibly learnable class with sample complexity $n(\rho,\alpha,\beta)$, and $n'=n(1/4,\alpha/4,\beta/4)$. Then $(X,H)$ is agnostically learnable in
% \[
% m(\rho,\alpha,\beta) \leq O\left( \frac{VC(H)^2\log(n')(VC(H)\log(n')+\log(\frac{1}{\rho\beta}))(\alpha^2n' + VC(H)\log(n')+\log(1/\beta))}{\rho^2\alpha^2}\right).
% \]
% \end{corollary}
% Note that in the distribution-free setting this is at most
% \[
% m(\eta,\nu,\alpha,\beta) \leq ???
% % \tilde{O}\left(\frac{n^{O(VC(X,H))}}{\rho^2(1-\eta)^4} + \frac{\log(1/\beta)}{\rho^2\alpha^2}\right).
% \]

The proof of \Cref{thm:agnostic} is based on the following variant of a sub-routine from \cite{hopkins2021realizable}'s agnostic-to-realizable reduction that generates a small list of good hypotheses.

\begin{algorithm}[H]

\KwResult{Outputs a list of good hypotheses}
\nonl \textbf{Input:} Replicable learner $\mathcal{L}$ on $n=n(\rho,\alpha,\beta)$ samples, family of $O(\log(1/\beta))$ random strings $\{r\}$\\
\nonl \textbf{Parameters:} 
\begin{itemize}
    \item Accuracy and confidence parameters $\alpha,\beta > 0$
    \item Labeled sample size $t = O\left(\frac{\log(\Pi_H(n(1/4,\alpha/4,\beta/4))+\log(1/\beta)}{\alpha^2}\right)$
\end{itemize}
\nonl \textbf{Algorithm:}\\
\begin{enumerate}
    \item Sample $n$ (unlabeled) samples $S_U \sim D_X^n$, and $t$ labeled samples $S_L \sim D^t$
    \item Run $\mathcal{L}$ across all strings in $r$ on all possible labelings of $S_U$ to receive:
    \[
    C_r(S_U) \coloneqq \{\mathcal{L}((S_U,h(S_U));r_i): h \in H, r_i \in \{r\} \}
    \]
    \item Prune sub-optimal hypotheses from $C_r(S_U)$:
        \[
    C^\alpha_r(S_U,S_L) \coloneqq \left\{h \in C_r(S_U): R_{S_L}(h) \leq \min_{h' \in C_r(S_U)}(R_{S_L}(h')) +\alpha/2 \right\}
    \]
\end{enumerate}
\textbf{return}  $C^\alpha_r(S_U,S_L)$
 \caption{List Distribution Generator}
 \label{alg:agnostic-subroutine}
 
\end{algorithm}
\Cref{alg:agnostic-subroutine} generates a sample from a distribution over families of hypotheses with near-optimal error. The accuracy of $\mathcal{A}$ promises that $C^\alpha_r(S_U,S_L)$ will be non-empty, and its replicability guarantees the distribution will have heavy-hitters. This means we can apply list heavy hitters to find a good hypothesis replicably.

\begin{proof}[Proof of \Cref{thm:agnostic}]
It is enough to prove that for any distribution $D$ over $X \times Y$, the distribution over lists defined by $C_r^\alpha(S_U,S_L)$ satisfies
\begin{enumerate}
    \item Correctness:
    \[
    \Pr[\forall h \in C_r^\alpha(S_U,S_L): R(h) \leq OPT+\alpha] \geq 1-\beta/2
    \]
    \item Heaviness:
    \[
    \exists h \in H: \Pr[h \in C_r^\alpha(S_U,S_L)] \geq 1/2
    \]
\end{enumerate}
For $\beta$ a small enough constant, any $\Omega(1)$-heavy-hitter has error at most $OPT+\alpha$, so the result then follows immediately from applying list heavy hitters.

% for most families $\{r\} = \{r_i\}_{i=1}^{O(\log(1/\beta))}$ of $O(\log(1/\beta))$ random seeds \Cref{alg:agnostic-subroutine} simulates a list oracle where the score function $s: H \to \R_{\geq 0}$ is given by the risk:
% \[
% s(h)=R(h)= \Pr_{(x,y) \sim D}[h(x) \neq y].
% \] 
% In this framework, the goal of agnostic learning is exactly to output a hypothesis $h \in H$ satisfying $s(h) \leq s^* + \alpha=OPT + \alpha$ with probability at least $1-\beta$, which is promised by the conclusion of \Cref{thm:list-amplify}. Given this fact, it is enough to show that \Cref{alg:agnostic-subroutine} is a list oracle across a $1-\beta/2$ fraction of families $\{r\}$, as the result then follows from picking a random family $\{r\}$ and running \textsc{ListAmplify} (setting the confidence parameter to $\beta/2$) with the resulting list oracle.

The first of these facts, correctness, is essentially trivial and just follows from observing that the size of $C_r(S_U)$ (the pre-pruned set) is at most $O(\log(1/\beta)\Pi_H(n))$ by construction. Since we use empirical estimates over $t = O(\frac{\log(\Pi_H(n)/\beta)}{\alpha^2})$ samples, standard Chernoff and union bounds imply that the empirical error of every element is estimated within $\alpha/4$ of its true value which implies the desired correctness guarantee.

% follows essentially from the analysis of \cite{hopkins2021realizable}, who observed that
% To show these facts, we first argue that over a $1-\beta/4$ fraction of families $\{r\}$, all elements in $C^\alpha_r(S_U,S_L)$ have error at most $OPT+\alpha$ with probability at least $1-\beta/4$. This follows from two observations. First, since the size of $C_r(S_U,S_L)$ (the pre-pruned set) is at most $\Pi_H(n)$ by construction and we use empirical estimates over $t = O(\frac{\log(\Pi_H(n)/\beta)}{\alpha^2})$ samples, standard Chernoff and Union bounds imply that the error of every element is estimated within $\alpha/4$ of its true value. Second, recall that 
% by the guarantee that $\mathcal{L}$ is a realizable learner, $\mathcal{L}(S_U,h_{OPT}(S_U); r)$ is within $\alpha/4$ of $h_{OPT}$ with probability at least $1-\beta/4$ over the randomness of $r$ and $S_U$. Furthermore, since 

% Considering the randomness over $r$ and $S_U$ separately, it must also then be the case that over a $1-\beta/2$ fraction of strings, $\mathcal{L}(S_U,h_{OPT}(S_U))$ is within $\alpha/4$ of $h_{OPT}$ with probability at least $1-\beta/2$ (for small enough $\beta$). Finally, since $C^r(S_U,S_L)$ contains $\mathcal{L}(S_U,h_{OPT}(S_U))$ by construction, this means that over a $1-\beta/2$ fraction of strings we have that $\min_{h' \in C_r(S)}(R_{S_L}(h')) \leq OPT + \alpha/4$ with probability at least $1-\beta/2$, and therefore that all elements with empirical error at most $\min_{h' \in C_r(S)}(R_{S_L}(h')) +\alpha/2$ have true risk at most $OPT+\alpha$ as desired.

It is left to show that $C^\alpha_r(S_U,S_L)$ has a heavy hitter. Fix some $h_{OPT} \in H$ achieving error OPT. Any $\frac{1}{4}$-replicable learner has the property that over at least half its random strings $r$, there exists some $h_r \in H$ such that:
\[
\Pr_{S_U}[\mathcal{L}((S_U,h_{OPT}(S_U)),r) = h_r] \geq 1/2.
\]
Furthermore, since $\mathcal{L}$ is additionally a PAC-learner, it must also be the case that $h_r$ is within $\alpha/4$ of $h_{OPT}$ over a $1-\beta/4$ fraction of these ``good'' strings. Since our family consists of $O(\log(1/\beta))$ random strings and $\mathcal{L}((S_U,h_{OPT}(S_U)),r) \in C_r(S_U)$ by construction, this means the pre-pruned set $C_r(S_U)$ has a $1/2$-heavy-hitter with error at most $OPT+\alpha/4$ over a $1-\beta/2$ fraction of families $\{r\}$. Finally, since our empirical estimates are good with high probability, $h_r$ also appears in the pruned set $C_r^\alpha(S_U,S_L)$ with at least constant probability as desired. Finally, the sample complexity bound then follows from combining \Cref{thm:list-heavy-hitters} with the observation that generating a sample from $C^\alpha_r(S_U,S_L)$ requires $O\left(\frac{\alpha^2n' + \log(\Pi_H(n'))+\log(1/\beta)}{\rho^2\alpha^2}\right)$ samples and $|C_r^\alpha(S_U,S_L)| \leq O(\log(1/\beta)\Pi_H(n'))$ by construction.
\end{proof}

\subsubsection{Private Realizable-to-Agnostic Reduction}
We are finally ready to prove our agnostic-to-realizable reduction for private learning. We restate the Theorem for ease of reading.
\begin{theorem}[\Cref{thm:agn-to-real} Restated]
Let $(H,\mathscr{D})$ be a hypothesis class that is $(1,\frac{1}{\poly(n)})$-privately $(\alpha, 0.01)$-PAC learnable in $n=n(\alpha)$ samples in the realizable setting. Then $(H,\mathscr{D})$ is $(1,\frac{1}{\poly(m)})$-privately $(\alpha, \beta)$-\textbf{Agnostically} learnable in
\[
m(\alpha,\beta) \leq O\left( \frac{n^2\alpha^2+\log^3(\Pi_H(cn^2))}{\alpha^2}\log\beta^{-1}\log\left(\frac{n\log\beta^{-1}}{\alpha}\right)\right)
\]
samples for some universal constant $c>0$.
\end{theorem}
\begin{proof}
Recall we are given a realizable $(1,\poly(n^{-1}))$-DP, $(\alpha,0.01)$-accurate PAC learner $\mathcal{A}$ on $n$ samples. We will convert $\mathcal{A}$ into an agnostic learner via the following 5 step process:
\begin{enumerate}
    \item Amplify privacy to $(m^{-1/2},\poly(m^{-1}))$ by ``secrecy of the sample'' for $m \approx n^2$
    \item Convert $\mathcal{A}$ to a $0.01$-replicable realizable learner $\mathcal{R}$
    \item Convert $\mathcal{R}$ into an agnostic learner $\mathcal{R}_{agn}$
    \item Convert $\mathcal{R}_{agn}$ back into an agnostic private learner $\mathcal{A}_{agn}$
    \item Privately amplify correctness of $\mathcal{A}_{agn}$
\end{enumerate}
Let's formalize this procedure. In the first step, we simply wish to convert $\mathcal{A}$ into a $(cm^{-1/2},\poly(m^{-1}))$-DP, $(\alpha,.01)$-accurate PAC learner on $m$ samples for some small enough constant $c>0$. This can be done by the so-called ``secrecy of the sample'' method (\Cref{lem:samp-wo-replace}): draw $m=O(n^2)$ examples, construct a subset $S$ by selecting $m$ elements uniformly at random, and return $\mathcal{A}(S)$. For an appropriate choice of constants this is $(cm^{-1/2},\poly(m^{-1}))$-DP. Accuracy is maintained since $S$ is equidistributed with a standard size $m$ sample. 

% only fails if $|S| \leq n$ (which occurs with at most say $.01$ probability). \maxh{Double check this.}

Now that we have our $(cm^{-1/2},\poly(m^{-1}))$-DP, $(\alpha,.01)$-accurate PAC learner, we invoke \Cref{cor:DPtoRep} (applying correlated sampling) to build the $.01$-replicable learner $\mathcal{R}$ on $O(n^2)$ samples that maintains $(\alpha,.01)$-correctness. Applying our agnostic-to-replicable reduction for replicable learning, this gives an $.01$-replicable $(\alpha,.01)$-correct agnostic learner on
\[
m' \leq O\left(n^2+ \frac{\log^3(\Pi_H(c'n^2))}{\alpha^2}\right)
\]
samples for some constant $c'>0$. Finally, we move back to the private regime via \Cref{thm:Rep-to-DP}, which gives an $(\varepsilon,\delta)$-DP, $(\alpha,.1)$-correct agnostic learner on $O(\frac{m'\log \delta^{-1}}{\varepsilon})$ samples.

It is left to amplify the correctness probability $\beta$. This can be done by running the above algorithm independently $\log(1/\beta)$ times, and privately outputting the best hypothesis on the output set via the exponential mechanism, which one can check results in a ($2\varepsilon$,$\delta$)-DP $(2\alpha,\beta)$-accurate learner (see e.g.\ \cite[Theorem A.1]{sivakumar2021multiclass}).

We have now seen how to build
a ($\varepsilon$,$\delta$)-DP ($\alpha$,$\beta$)-accurate learner on 
\[
m'' \leq O\left(\frac{m'\log\delta^{-1}\log\beta^{-1}}{\varepsilon}\right)
\]
samples. To give the form of the result in the theorem statement, it is enough to choose sample size $t$ satisfying the recurrence $t \geq \Omega(m'\log t \log \beta^{-1})$. Selecting $t=c_2m'\log \beta^{-1}\log (m'\beta^{-1})$ for large enough $c_2>0$ then completes the proof.
\end{proof}

\subsection{Replicable Algorithms from Reduction}\label{sec:alg-by-reduction}
In this section, we show how we can use our reduction from replicability to differential privacy to obtain new replicable algorithms. Our reduction preserves accuracy, because on any fixed dataset, the output distribution of the replicable algorithm is identical to that of the differentially private algorithm (since our reduction simply applies correlated sampling to the output distribution of the differentially private algorithm on the input dataset---see Sections~\ref{sec:dp2PG} and~\ref{sec:PG2rep}). 
\maxh{This section has many key references to Theorems in other works without stating them, which makes it a bit un-self-contained. I'd probably state the privacy versions we're converting formally in each part, but happy to stick with current version if others prefer.}
\subsubsection{PAC Learning}

We note that what we term ``replicable'' PAC learning corresponds to settings where the algorithm $A$ is a PAC learner, and additionally is replicable for all input distributions $D$. 

We show that our reduction gives the best known sample complexity bounds for replicable realizable and agnostic PAC learning for many hypothesis classes. Prior work also had to prove sample complexity bounds separately for all of these frameworks, whereas we are able to translate bounds proved for differential privacy directly through our reduction.

\paragraph{Thresholds/Approximate Median:} Fix any integer $d \geq 0$. We apply our framework to the hypothesis class $Thresh_d$ consisting of thresholds over the domain $\{0,1,\dots,d\}$. A threshold function $f_z$ parameterized by integer $0 \leq z \leq d$, is defined as follows.
\begin{equation}
f_z(x) = 
\begin{cases}
1 & \text{ if } x > z \\
0 & \text{ if } x \leq z
\end{cases}
\end{equation}

Impagliazzo et al. \cite{ImpLPS22} asked whether a PAC learner could be obtained for this class with sample complexity polynomial in $\log^* d$. Our reduction answers this question by using a result of \cite{KaplanLMNS20} on learning thresholds privately. \footnote{We note that a similar approach (taking a differentially private algorithm for learning distributions under Kolomogorov distance (guaranteed by a reduction in \cite{BunNSV15} to the interior point problem), and applying our conversion from approx DP to replicability) also gives a replicable algorithm for releasing approximate median of a distribution with accuracy $\alpha$ and sample complexity $O_{\alpha}(\poly \log^*(d))$. This closes an exponential gap in \cite{ImpLPS22}.}  %\sstext{Improve the dependence on $\alpha$ by starting with constant $\alpha$ and boosting/starting with interior point problem and then converting (not sure known method works reproducibly- may need a little thought).}

We first introduce the interior point problem.
\begin{definition}
An algorithm solves the interior point problem over a totally ordered domain $\mathcal{X}$ with error probability $\beta$, if for all datasets $S \in \mathcal{X}^m$, if
$$ \Pr[\min_i S_i \leq \mathcal{A}(S) \leq \max_i S_i ] \geq 1-\beta. $$
\end{definition}
Now we are ready to apply our reduction to obtain the improved sample complexity.
\begin{theorem}
For all sufficiently small $\rho, \alpha, \beta \in (0,1)$, there exists a $\rho$-replicable, $(\alpha, \beta)$-accurate realizable PAC learner for the hypothesis class $Thresh_d$ with sample complexity 
\begin{align*}
    m = \tilde{O}\left( \frac{(\log^* d)^3 \log^2(1/\beta) \log^4(1/\rho)}{\alpha^2 \rho^2} \right)
\end{align*}
\end{theorem}
\begin{proof}
Let $\eps$ and $\delta$ be set as specified in Corollary~\ref{cor:DPtoRep}. The work of \cite{KaplanLMNS20} (Theorem 4.1 in their paper) gives an $(\eps/2, \delta/2)$ algorithm for solving the interior point problem with sample complexity $O(\frac{1}{\eps} (\log^* d \log(1/\delta))^{1.5})$ and error probability at most $1/10$. By a result of Bun et al. \cite[Theorem 5.6, Part 1]{BunNSV15}, this gives an $(\eps, \delta)$-DP, $(\alpha, 2/10)$-proper PAC learner for $Thresh_d$ with sample complexity $O(\frac{1}{\eps \alpha} \left(\log^* d \log(1/\delta)\right)^{1.5})$. Now, by work of \cite{BunCS20} (See \cite[Theorem A.1]{sivakumar2021multiclass} for a formal statement we use directly) this can be boosted to give an $(\eps, \delta)$-DP, $(\alpha, \beta)$-accurate proper PAC learner for $Thresh_d$ with sample complexity $O(\frac{1}{\eps \alpha} \left(\log^* d \log(1/\delta)\right)^{1.5} \log(1/\beta))$. 

First, we note that correlated sampling does not affect the accuracy guarantees since it maintains the distribution of the differentially private algorithm. Now, applying Corollary~\ref{cor:DPtoRep}, and substituting in the values of $\eps$ and $\delta$ we get that there is a $\rho$-replicable $(\alpha, \beta)$-accurate PAC learner for $Thresh_d$, whose sample complexity is the solution to the equation
$$m = C\frac{\sqrt{m \log(1/\rho)}}{\rho \alpha} \left(\log^* d \log(m/\rho)\right)^{1.5} \log(1/\beta)$$
This gives us that 
$$m = \tilde{O}\left(\frac{\log^4(1/\rho)}{\rho^2 \alpha^2} \left(\log^* d \right)^{3} \log^2(1/\beta)\right).$$
\end{proof}

\paragraph{Finite Binary Hypothesis Classes:} By applying a result of \cite{kasiviswanathan2011can} on privately (agnostically) learning finite classes $H$, we get a learner with sample complexity that's polynomial in $\log |H|$. 
\begin{theorem}\label{thm:finitehypred}
For all sufficiently small $\rho, \alpha, \beta \in (0,1)$, and for all finite hypothesis classes $H$, there exists a $\rho$-replicable, $(\alpha, \beta)$-accurate agnostic PAC learner for $H$ with sample complexity 
\begin{align*}
    m= O\left( \frac{(\log |H| + \log(1/\beta))^2  \log (1/\rho)}{\alpha^2 \rho^2} \right)
\end{align*}
\end{theorem}
\begin{proof}
Let $\eps$ be set as specified in Corollary~\ref{cor:DPtoRep}. The work of \cite{kasiviswanathan2011can} gives an $(\eps, 0)$-DP agnostic learner for finite classes with sample complexity $m=O\left((\log |H| + \log(1/\beta))\left(\frac{1}{\alpha \eps} + \frac{1}{\alpha^2}\right)\right)$.

We note that correlated sampling does not affect the accuracy guarantees since it maintains the distribution of the differentially private algorithm. Hence,substituting the value of $\eps$ and applying Corollary~\ref{cor:DPtoRep}, we get a $\rho$-replicable $(\alpha, \beta)$-accurate agnostic PAC learner for finite class $H$, whose sample complexity is the solution to the equation
$$m=O\left((\log |H| + \log(1/\beta))\left(\frac{\sqrt{m \log(1/\rho)}}{\alpha \rho} + \frac{1}{\alpha^2}\right)\right),$$
which gives us a quadratic in $\sqrt{m}$. Solving, we get that
\begin{align*}
    m = O\left( \frac{(\log |H| + \log(1/\beta))^2  \log (1/\rho)}{\alpha^2 \rho^2} \right)
\end{align*}
\end{proof}

The dependence on the accuracy parameter obtained via directly using our transformation here is suboptimal for realizable learners, and we cannot hope to improve it using our reduction alone, since the private finite class learner described above is optimal for realizable learners as well. See the finite class learner presented in Section~\ref{sec:finiteclasses} (that achieves the right inverse linear dependence on $\alpha$) for more discussion. There also isn't a known replicable boosting algorithm with an inverse linear dependence on the accuracy parameter $\alpha$. It is an interesting open question to investigate whether such a boosting algorithm exists.

We also observe that via our reduction, we obtain $\rho$-replicable $(\alpha, \beta)$-PAC learners with sample complexity scaling as $O_{\rho, \alpha,\beta}(\mathtt{PRDim}(H)^2)$ for finite hypothesis classes with finite probabilistic representation dimension (denoted by $\mathtt{PRDim}$), and $\rho$-replicable $(\alpha, \beta)$-PAC learners with sample complexity scaling as $\tilde{O}_{\rho, \alpha,\beta}(\mathtt{\mathtt{LDim}}(H)^{12})$ for finite hypothesis classes with finite Littlestone dimension (denoted by $\mathtt{LDim}$) by instantiating our general transformation with private learners due to \cite{beimel2013characterizing} and \cite{ghazi2021sample} respectively. This improves on the sample complexities obtained in the work by Ghazi, Kumar and Manurangsi \cite{GhaziKM21}. We also give a direct version of the argument in \Cref{app:littlestone}.
\subsubsection{Distribution Estimation Problems} As another illustration of the generality of our reduction, we instantiate it to give the first replicable algorithms for some distribution estimation problems. 
\paragraph{Discrete distribution estimation}
Consider the set $P_k$ of all distributions over the domain $[k] = \{1,2,3,\dots,k\}$ (where $k$ is a natural number). The problem of discrete distribution estimation involves getting samples from any unknown fixed distribution $D$ from $P_k$, and having to output a distribution $D'$ that is close in some measure of distance to $D$ (we call the closeness the ``accuracy'' of the algorithm). 

We now describe the problem more formally. An algorithm is said to solve the discrete distribution estimation problem with accuracy $\alpha$ and $m(k)$ samples, if for all $k>1$, and for all fixed distributions $D$ over $[k]$, there exists an algorithm taking $m(k)$ independently drawn examples from $D$ and outputting a distribution $D'$ such that in expectation over the coins of the algorithm and the randomness of the sample, $d_{TV}(D,D') \leq \alpha$. 

It is known that the sample complexity of solving this problem with accuracy $\alpha$ (with no stability constraints) is $\Theta(k/\alpha^2)$. %On the other hand, there is an algorithm solving this problem with accuracy $\alpha$ and $O(k/\alpha^2)$ samples. That is, the sample complexity of $\alpha$-accurate discrete distribution estimation without stability  is $\Theta(k/\alpha^2)$. 

If we add privacy constraints to the picture, it is known that there is an $\eps$-DP algorithm for discrete distribution estimation that requires $O(k/\alpha^2 +k/\alpha \eps)$ examples (see e.g., \cite{AcharyaSZ21}). They show that this is tight even for $(\eps, \delta)$-DP algorithms, (when $\delta \leq \eps$, which is most often the regime of interest). 

We can instantiate our reduction with this algorithm to get the first replicable algorithm for discrete distribution estimation.

\begin{theorem}
Fix any $k>2$. For all sufficiently small $\rho, \alpha \in (0,1)$, there exists an $\alpha$-accurate, $\rho$-replicable algorithm that solves the discrete distribution estimation problem with $\alpha$-accuracy, and $m$ examples, where
$$m = O\left(\frac{k^2 \log1/\rho}{\alpha^2 \rho^2} \right).$$
\end{theorem}
\begin{proof}
 Let $\eps$ be set as specified in Corollary~\ref{cor:DPtoRep} (which gives the parameters for our conversion from differential privacy to replicability). The work of \cite{AcharyaSZ21} and \cite{DiakonikolasHS15} give an $(\eps, 0)$-DP algorithm $\mathcal{A}$ for discrete distribution estimation that takes in $m=O(k/\alpha^2 +k/\alpha \eps)$ samples. 
 
 Now, we post-process this algorithm to get a finite output space to apply our reduction to. For every $i \in [k]$, round every coordinate of the output distribution $d$ to the closest multiple of $\frac{\alpha}{k}$, to get a vector $v$. Now, apply the procedure given in Corollary~\ref{cor:DPtoRep} to convert this to a replicable algorithm. Call the output of the replicable algorithm $\hat{v}$. Finally do an $\ell_1$ projection from $\hat{v}$ back to the $k$-simplex to get a new distribution $\hat{d}$.
 
 We now argue that the above transformation preserves accuracy. Let the original distribution be $p$. Then, by the triangle inequality, we can write that $$\mathbb{E}[\| \hat{d} - p \|_1] \leq \mathbb{E}[\| \hat{d} - \hat{v} \|_1] + \mathbb{E}[\| \hat{v} - p \|_1].$$
 Now, since $\hat{v}$ is identically distributed to $v$ (since the transformation to replicability simply involves correlated sampling), we have that for every such vector, there is a distribution in the $k$-simplex that is within $\alpha$ of it in $\ell_1$ distance (because every output vector $v$ prior to applying correlated sampling was obtained by rounding each coordinate of a distribution in the $k$-simplex to the closest multiple of $\frac{\alpha}{k}$). Hence, since $\hat{d}$ is the $\ell_1$ projection of $\hat{v}$ onto the $k$-simplex, we have that $\mathbb{E}[\| \hat{d} - \hat{v} \|_1] \leq \alpha$. Hence, we can write that 
 $$\mathbb{E}[\| \hat{d} - p \|_1] \leq \alpha + \mathbb{E}[\| v - p \|_1],$$
 where in the second term on the right hand side, we have used again that $\hat{v}$ and $v$ are identically distributed.
 
 Finally, since $\mathbb{E}[\|v-p\|_1] \leq \mathbb{E}[\|v-d\|_1] + \mathbb{E}[\|d-p\|_1]$ (by triangle inequality), and each of these terms is smaller than $\alpha$ (since $v$ is obtained by discretizing $d$ to a grid of length $\frac{\alpha}{k}$, and the second term can be bounded by the accuracy of algorithm $\mathcal{A}$), we get that
  $$\mathbb{E}[\| \hat{d} - p \|_1] \leq \alpha,$$
 as required.

  Replicability of this transformation follows from the fact that the transformation prior to projection onto the simplex is replicable (by Corollary~\ref{cor:DPtoRep}), and the fact that replicability is preserced under post-processing.
  
  Hence, substituting the value of $\eps$ into the number of samples needed for the algorithm $\mathcal{A}$ to be $\alpha$-accurate, we get a $\rho$-replicable $\alpha$-accurate algorithm for discrete distribution estimation, whose sample complexity is the solution to the equation
$$m=O\left( \frac{k}{\alpha^2} + \frac{k \sqrt{m \log 1/\rho}}{\alpha \rho} \right),$$
which gives us a quadratic in $\sqrt{m}$. Solving, we get that
\begin{align*}
  m = O\left(\frac{k^2 \log1/\rho}{\alpha^2 \rho^2} \right),
\end{align*}
completing the proof.
\end{proof}
It is unclear from our results whether there is an algorithm for replicable discrete distribution estimation over $[k]$ that can achieve sample complexity linear in $k$; we leave this as an open problem.

\paragraph{Gaussian mean estimation}
In this section, we give a replicable algorithm for high-dimensional Gaussian mean estimation (in the unknown covariance case).

In Gaussian mean estimation in $d$ dimensions, algorithms are given examples drawn independently from a Gaussian distribution $N(\mu,\Sigma^2)$, where $\mu \in \mathbb{R}^d$ is the unknown mean, and $\Sigma$ is an unknown $d \times d$ positive definite matrix. The goal is to estimate $\mu$. The metric we will use to evaluate the quality of an estimate is the ``Mahalanobis distance'', which measures the error scaled according to the covariance matrix of the Gaussian distribution.

That is, with probability at least $1-\beta$ over the examples and the internal randomness of the algorithm, we want the algorithm given sample access to $N(\mu,\Sigma^2)$ to output a value $\hat{\mu}$ such that 
$$\| \hat{\mu} - \mu \|_{\Sigma} = \| \Sigma^{-1/2} (\hat{\mu} - \mu) \|_2 \leq \alpha.$$ 

Without stability constraints, it is known that this problem can be solved using $m = \Theta(d / \alpha^2)$ examples. 

Under the constraints of approximate differential privacy, the picture is more complicated. For a long time, the best dependence on $d$ that was known, was $d^{3/2}$, with the bottleneck being private covariance estimation. However, in recent work, Brown, Gaboardi, Smith, Ullman, and Zakynthinou \cite{BrownGSUZ21} gave a sophisticated differentially private algorithm that achieved a linear dependence on $d$, by avoiding covariance estimation entirely. It is not clear how to make similar techniques work to obtain replicable algorithms via a direct analysis. Our reduction allows us to lift the analysis from \cite{BrownGSUZ21} to give a replicable algorithm for this task.

Since correlated sampling is known to only work on finite output spaces, we need to assume that the mean falls in a bounded $\ell_{\infty}$ ball (though our accuracy will not depend on the bounds of this ball). Additionally, we will need to discretize the output of the differentially private algorithm. However, discretization in this case is non-trivial, as the measure of accuracy is with respect to the unknown covariance matrix, and hence, we will first have to replicably estimate the minimum eigenvalue of the covariance matrix in order to decide the right level of discretization. For this purpose, we once again use our reduction and apply it to a differentially private algorithm for this task, also \cite{BrownGSUZ21}. We will assume that the covariance matrix's minimum eigenvalues are between non-negative numbers $k$ and $\ell$ (known to the algorithm),\footnote{Note that this assumption can be relaxed by directly estimating the minimum eigenvalue using replicable heavy hitters instead of reducing to the DP algorithm} to guarantee finiteness of the output space for this algorithm. Again, our sample complexity is independent of these parameters. 

\begin{theorem}
Fix $R >0$, $0 < k < \ell$, and sufficiently small $\rho>0$. Fix a distribution $D = N(\mu, \Sigma)$, where $\|\mu \|_{\infty} \leq R$, and the minimum eigenvalue of $\Sigma$ is between $k$ and $\ell$. Then, there is $\rho$-replicable algorithm that outputs an $\alpha$-accurate estimate of the mean $\mu$ (in Mahalanobis distance) with probability at least $1-\beta$, when given $m$ independently drawn samples from $D$, where 
  $$m =\tilde{O}\left(\frac{(d \log (1/\alpha) + \log (1/\beta))^2 \log^3 (1/\rho)}{\alpha^2 \rho^2}\right).$$
\end{theorem}
\begin{proof}
 First, consider Lemma C.2 in \cite{BrownGSUZ21}. This gives a general way to privately estimate the minimum eigenvalue of $\Sigma$. We first discretize and truncate the output space of this algorithm as follows (and round outputs to their closest point in the corresponding grid). The discretization length will be $k/8$, and the upper bound will be $4\ell$. Since the algorithm in their paper guarantees a $4$-approximation of the minimum eigenvalue with probability at least $1-\beta$, and takes $O(\frac{d \log (1/\beta \delta)}{\eps})$ examples, by the assumed bounds on the covariance matrix, the discretized version guarantees an $8$-approximation. 
 
 Now, since the output space of the eigenvalue estimation algorithm has been made finite, we apply the transformation in Corollary~\ref{cor:DPtoRep} with $\eps, \delta$ set accordingly (with $\rho$ being $\rho/2$). This gives a $\rho/2$-replicable algorithm that gives an $8$-approximation to the minimum eigenvalue of the covariance matrix with sample complexity that is the solution to the equation $m_1 = O\left(\frac{d \sqrt{m_1 \log (1/\rho)} \log (m_1/\beta \rho) }{\rho}\right)$, which gives us that $$m_1 = \tilde{O}\left(\frac{d^2 \log^3 (1/\rho) \log^2 (1/\beta)}{\rho^2}\right).$$ Let the output of this algorithm $\mathcal{A}_1$ be $\hat{\lambda}$.
 
 Now, we are ready to use the mean estimator described in Theorem 2 from \cite{BrownGSUZ21} (on a fresh set of samples). We will assume that hardcoded into this algorithm is an eigenvalue $\hat{\lambda}$, which is an $8$-approximation to the minimum eigenvalue $\lambda_d$ of the covariance matrix. Consider a postprocessing of the output $\hat{\mu}$ of the algorithm described in that theorem such that the value of each coordinate is truncated to have $\ell_{\infty}$ norm at most $R$, and has been projected to an $\alpha'$-grid, where $\alpha' = \min(\hat{\lambda}^{1/2},1) \frac{\alpha}{\sqrt{d}}$. Let the post-processed mean be $\mu_{disc}$. Then, by the guarantee of Theorem 2 in their paper, we have that with probability at least $1-3\beta$, 
 \begin{align*}
 \|\mu_{disc} - \mu \|_{\Sigma} & \leq 
 \|\mu_{disc} - \hat{\mu} \|_{\Sigma} + \|\hat{\mu} - \mu \|_{\Sigma} \\
 & \leq \|\Sigma^{-1/2}[\mu_{disc} - \hat{\mu}] \|_{2} + \alpha \\
  & \leq \frac{1}{\sqrt{\lambda_d}}\|\mu_{disc} - \hat{\mu}] \|_{2} + \alpha \\
 & \leq \frac{\alpha '}{\sqrt{\lambda_d}} \sqrt{d} + \alpha \\
 & \leq \frac{\alpha \sqrt{\lambda}}{ \sqrt{d} \sqrt{\lambda_d}} \sqrt{d} + \alpha \\
 & \leq \frac{\alpha \sqrt{8 \lambda_d}}{ \sqrt{d} \sqrt{\lambda_d}} \sqrt{d} + \alpha = O(\alpha)
 \end{align*}
 Note that the sample complexity of this $(\eps, \delta)$-DP algorithm (where privacy is wrt the fresh sample) is $$m_2=O\left(\frac{d+\log(1/\beta)}{\alpha^2} + \frac{\log(1/\delta)}{\eps} + \frac{d \log(1/\alpha) + \log(1/\beta)}{\alpha \eps}\right).$$
 
 Now, we are ready to apply our transformation (recall that it preserves accuracy, since the distribution is unchanged by correlated sampling). Setting $\eps, \delta$ as in Corollary~\ref{cor:DPtoRep} (with $\rho$ set to $\rho/2$), and applying our transformation we can convert this to a $\rho/2$-replicable algorithm $\mathcal{A}_2$, with sample complexity that is the solution to the equation
 $$m_2= O\left(\frac{d+\log(1/\beta)}{\alpha^2} + \log\frac{m_2}{\rho} \frac{\sqrt{m_2 \log (1/\rho)}}{\rho} + \frac{(d \log(1/\alpha) + \log (1/\beta))\sqrt{m_2 \log (1/\rho)}}{\alpha \rho}\right).$$
  Solving this equation gives that
   $$m_2=\tilde{O}\left(\frac{(d \log(1/\alpha) + \log (1/\beta))^2 \log^3 (1/\rho)}{\alpha^2 \rho^2}\right).$$
Now, note first that by adaptive composition, running $\mathcal{A}_1$ and then using its estimate in the algorithm $\mathcal{A}_2$ together gives a $\rho$-replicable algorithm (since each is individually $\rho/2$-replicable). Additionally, the composed algorithm is $\alpha$-accurate with probability $1-4\beta$ (taking a union bound of the failure probabilities of $\mathcal{A}_1$ and $\mathcal{A}_2$). Note that $m_2$ asymptotically dominates $m_1$, so the sample complexity of this entire procedure is
   $$m =\tilde{O}\left(\frac{(d \log(1/\alpha) + \log (1/\beta))^2 \log^3 (1/\rho)}{\alpha^2 \rho^2}\right).$$
\end{proof}

\subsubsection{Gaussian Identity Testing}
As a final example of the generality of our reduction, we consider the problem of identity testing of multivariate Gaussian distributions. In this problem, we are given samples from either a fixed Gaussian distribution $D$ with known mean and covariance, or from a Gaussian distribution that is $\alpha$-far in Mahalanobis distance from $D$. The goal is to correctly guess which case we're in with probability at least $2/3$ (we will say the algorithm successfully distinguishes the two cases if this is satisfied). Without stability constraints, this problem can be solved with $O(\frac{\sqrt{d}}{\alpha^2})$ samples.

This problem was studied subject to privacy constraints in \cite{Canonne0MUZ20}, and their results were then improved in \cite{Narayanan22}. We apply results of the latter to get a replicable algorithm for Gaussian identity testing.

\begin{theorem}
Fix $d \in \mathbb{N}^+$ and sufficiently small $\rho, \alpha > 0$. Fix known $\mu \in \mathbb{R}^d$ and known covariance matrix $\Sigma \in \mathbb{R}^{d \times d}$. Then, there is a $\rho$-replicable algorithm $\mathcal{A}$ that can succesfully distinguish between Case $H_0$, where $\mathcal{A}$ receives $m$ samples from $N(\mu,\Sigma)$ and Case $H_1$ where $\mathcal{A}$ receives $m$ samples from any distribution $N(\mu',\Sigma)$, such that $\|\mu' - \mu\|_{\Sigma} \geq \alpha$, as long as 
$$ m = \tilde{O}\left(\frac{d^{1/2}}{\alpha^2 \rho^2}\right).$$
\end{theorem}
\begin{proof}
Note that Theorem 1.7 of \cite{Narayanan22} gives an $(\eps, 0)$-algorithm for this task that achieves sample complexity $m = \tilde{O}\left(\frac{d^{1/2}}{\alpha^2} + \frac{d^{1/4}}{\alpha \eps}\right)$. We directly instantiate this algorithm with our reduction in order to get the result (since the output space is finite (just a single bit), we can do this without modifying the algorithm).

Set $\eps$ as in Corollary~\ref{cor:DPtoRep}. Note that since our reduction maintains the same distribution as the differentially private algorithm, the accuracy guarantees are the same. Hence, there is a $\rho$-replicable algorithm for this task with sample complexity that is the solution to the equation
$$m = \tilde{O}\left(\frac{d^{1/2}}{\alpha^2} +  \frac{d^{1/4} \sqrt{m \log 1/\rho}}{\alpha \rho}  \right).$$
Then, solving this equation, we get that
$$ m = \tilde{O}\left(\frac{d^{1/2}}{\alpha^2 \rho^2}\right).$$
\end{proof}
\section*{Acknowledgements}
We thank Adam Smith for helpful discussions on max information, Zhiwei Wu for pointing us to perfect generalization, Shay Moran for helpful discussions regarding the relation of our work to \cite{KKMV23}, and Christopher Ye for helpful comments on a prior version of this manuscript. The views expressed in this paper are those of the authors and not those of the U.S. Census Bureau or any other sponsor.

\newpage

\bibliographystyle{amsalpha}  
\bibliography{references} 

\newpage
\appendix
\section{Estimating OPT}\label{app:OPT}
In this section we give an algorithm for replicably estimating the minimum error hypothesis in a class $(X,H)$ over an arbitrary joint distribution $D$ over $X \times \{0,1\}$.

\begin{algorithm}[H]

\KwResult{Outputs $v \in [OPT, OPT+\alpha/2]$}
\nonl \textbf{Input:} Finite Class $H$, Joint Distribution $D$ over $X \times \{0,1\}$ (Sample Access)\\
\nonl \textbf{Parameters:} 
\begin{itemize}
    \item Replicability, Accuracy, Confidence $\rho, \alpha, \beta>0$
    \item Sample Complexity $m=m(\rho, \alpha,\beta) \leq O\left(\frac{\log(|H| /\beta\rho)}{\alpha^2 \rho^2} \right)$
    % \item Reproducibility threshold $\tau \leq O(\frac{\alpha \rho}{ \ln |H|})$
\end{itemize}
\nonl \textbf{Algorithm:}\\
\begin{enumerate}
    \item Draw a labeled sample $S \sim D^m$ and compute $\emprisk{S}(f)$ for every $f \in H$.
    \item $a \leftarrow_r [0,\alpha/16]$ 
    \item $B_i = [i\alpha+a, (i+1)\alpha+a)$
\end{enumerate}
\textbf{return} $\frac{j}{8}\alpha+a$, where $OPT_S+\alpha/4 \in B_j$
 \caption{Replicably estimate OPT}
\label{alg:OPT-estimate}
\end{algorithm}
\begin{lemma}
Let $\mathscr{D}$ be a joint distribution over $X \times \{0,1\}$ and $H$ a concept class over $X$. Then for any $\alpha,\beta,\rho>0$, \Cref{alg:OPT-estimate} is a $\rho$-replicable algorithm over $O\left(\frac{\log(\frac{|H|}{\rho\beta})}{\rho^2\alpha^2}\right)$ samples that outputs a good estimate of $OPT$ with high probability:
\[
\Pr_{r,S}\big[\mathcal{A}(S) \in [OPT,OPT+\alpha/2]\big] \geq 1-\beta.
\]
\end{lemma}
\begin{proof}
The proof uses an argument similar to the randomized rounding trick introduced in \cite{ImpLPS22} for replicable statistical queries. Assume for simplicity that $\frac{1}{8\alpha}$ is integer (the argument is essentially no different otherwise), and break the interval $[0,1]$ into $\frac{\alpha}{8}$-sized buckets:
\[
B_1 = \left[0,\frac{\alpha}{8}\right), \ \ldots \ , B_{\frac{1}{\alpha}} = \left[1-\frac{\alpha}{8},1\right].
\]
Consider the rounding scheme \textsc{Round} that maps $OPT_S$ to the upper limit of its corresponding bucket. Notice that as long as $OPT$ is not within $\frac{\rho\alpha}{64}$ of the threshold value between two buckets, uniform convergence promises that $OPT_{S_1}$ and $OPT_{S_2}$ will lie in  the same bucket with probability at least $1-\frac{\rho}{2}$. As such the problem only occurs at the boundaries, which can be fixed by randomly shifting the thresholds between each bucket by $a \in [0,\frac{\alpha}{16}]$. Then for any fixed value of $OPT$, the probability it lies within $\frac{\rho\alpha}{64}$ of a shifted boundary is at most $\frac{\rho}{2}$, which combined with the previous observation proves the algorithm $\rho$-replicable.

Towards correctness, observe that uniform convergence of finite classes promises that the empirical optimum $OPT_S$ is within $\frac{\alpha}{16}$ of the true optimum with probability at least $1-\beta$. Furthermore, rounding shifts any value by at most $\frac{3\alpha}{16}$. Thus $\textsc{Round}(OPT_S+\alpha/4) \in [OPT,OPT+\frac{\alpha}{2}]$ with high probability as desired.
\end{proof}
\section{Learning Finite Littlestone Classes}\label{app:little}
One immediate application of list heavy-hitters is a sample-efficient replicable algorithm for classes with finite Littlestone dimension, as in \cite{ghazi2021sample}, leading to a modest improvement in sample complexity over the best known bound of $\tilde{O}(d^{14})$.
\begin{theorem}\label{app:littlestone}
Let $(X,H)$ be a class with Littlestone dimension $d$. Then the sample complexity of realizably replicably learning $(X,H)$ is at most:
\[
n(\rho,\alpha,\beta) \leq \tilde{O}\left(\frac{d^{12}\log^3(1/\beta)}{\alpha^2\rho^2}\right)
\]
\end{theorem}
\begin{proof}
In their work on user level privacy, Ghazi, Kumar, and Manurangsi \cite{GhaziKM21} build on the work of \cite{ghazi2021sample} to show the existence of an algorithm outputting lists of hypotheses satisfying the following guarantees:
\begin{enumerate}
    \item \textit{Optimality}: With probability at least $1-\beta/2$, all hypotheses output by $\mathcal{L}$ have risk at most $\alpha/2$
    \item \textit{Heavy Hitter}: There exists $h \in H$ output with probability $\Omega(1/d)$
    \item \textit{Size}: $\mathcal{L}$ outputs at most $\exp(d^2+d\log \frac{d}{\alpha\beta})$ hypotheses
    \item \textit{Sample Complexity}: $\mathcal{L}$ uses at most $\tilde{O}(\frac{d^6\log^2 \frac{1}{\beta}}{\alpha^2})$ samples.
\end{enumerate}
Note that for $\beta \leq O(1/d)$, any heavy hitter of this distribution is a good hypothesis, so it is enough to replicably output such a heavy hitter. Applying \Cref{thm:list-heavy-hitters}, this can be done $\rho$-replicably and with probability at least $1-\beta$ using
\[
O\left(\frac{\log^2\frac{|\mathcal{D}|\log\frac{1}{\rho\beta}}{\eta}\log\frac{1}{\rho}+\log \frac{1}{\rho\beta}}{\eta^2\rho^2}\log^3\frac{1}{\rho}\right)=\tilde{O}\left(\frac{d^6\log\frac{1}{\rho}+d^2\log \frac{1}{\rho\beta}}{\rho^2}\log^3\frac{1}{\rho}\right)
\]
i.i.d outputs of the list algorithm. Each output itself costs $\tilde{O}(\frac{d^6\log^2 \frac{1}{\beta}}{\alpha^2})$ samples to generate, leading to the stated sample complexity.
\end{proof}

\section{Additional Properties of Replicability}
\label{app:additional-properties-of-replicability}

\subsection{Randomness Management}
\label{apps:randomness-management}

Often, we design replicable algorithms which use randomness for multiple purposes. How do we ensure that they use the same sections of random string $r$ for the same subroutines? What if the number of bits used for each purpose varies between runs of the algorithm? The following arguments show that we can typically guarantee that the same sections of $r$ are used for the same purposes across both runs.

\begin{lemma}
Say an algorithm $\mathcal{A}$ makes at most $k$ calls to its randomness oracle, using at most $b_1, \dots, b_k$ bits of randomness for each call respectively. Then there is an algorithm $\mathcal{A}'$ that replicably uses at most $k \cdot \max_{i \in [k]}\{b_i\}$ bits of randomness. 
\end{lemma}

Here, by ``replicably uses" we mean that algorithm $\mathcal{A}'$ uses the same positions in the random string for each subroutine in every run of the algorithm. 
\begin{proof}
Have $\mathcal{A}'$ interpret its random string as follows: rather than using randomness sequentially for each of $k$ purposes (non-replicable if the required number of bits changes), portion the random string into $k$ pieces in a modular way. In other words, use the bits of $r$ in positions $i \mod k$ solely for the $i$'th call to the randomness oracle by algorithm $\mathcal{A}$. 
%This canonical ordering of the randomness ensures that all runs of the algorithm uses the same subsections of the given random string. 
At most $k \cdot \max_{i \in [k]}\{b_i\}$ bits of randomness are used. 
\end{proof}

Note that the algorithm itself does not need to know how much randomness it will use a priori to use this method. 

What if the algorithm does not have a fixed number of calls to the randomness oracle? As long as the randomness calls occur sequentially, one can assign consistent subsections of the random string to each possible call. To do so, we use the same snake-path trick (i.e., the Cantor pairing function) often used to equate the cardinality of the natural numbers and rational numbers. 

\begin{lemma}
Say an algorithm $\mathcal{A}$ makes at most $k$ calls to its randomness oracle, using at most $b_1, \dots, b_k$ bits of randomness for each call respectively. Then there is an algorithm $\mathcal{A}'$ that replicably uses at most $(k + \max_{i \in [k]}\{b_i\})^2/2 + (k + \max_{i \in [k]}\{b_i\})/2$ bits of randomness. 
\end{lemma}

\begin{proof}
We allocate bits from our randomness oracle to different (unknown bit-length) calls using the Cantor pairing function. The maximum overhead in bit complexity of the randomness occurs when the $k$'th randomness call uses the most bits. In this case, roughly half of $(k + \max_{i \in [k]}\{b_i\})^2$ random bits must be drawn.
\end{proof}
Again, the algorithm itself does not need to know how much randomness it will use a priori to use this method. It also does not need to know how many different calls $k$ to the randomness oracle will be performed, so long as these calls can be ordered sequentially in some canonical way. 
For example, if the algorithm operates in rounds with a finite number of (conditional) random calls in each round, then the algorithm can reserve specific sections of its random tape for each of the possible calls without requiring infinite randomness. 

%Question: for any deterministic algorithm with access to a randomness oracle, is there always a way to canonically order all the (possible) randomness queries?

%If the total possible number of different randomness calls is not finite but countably infinite, then \rexnote{maybe we can say something about this case?}
%, similar to how two applications of the Cantor pairing function can be used to imply that the sets $(\Z^+)^3$ and $\Z^+$ have the same cardinality.

More generally, if a replicable algorithm uses the Cantor pairing function to allocate portions of its randomness, then the Cantor pairing function can be replaced by any deterministic pairing function $f : \Z^+ \times \Z^+ \rightarrow \Z^+$ (and the ensuing algorithm will still be replicable).
However, different or more situational pairing functions may give a specific replicable algorithm $\Acal$ improvements in complexity parameters such as amount of random bits used, time complexity, and space complexity. When designing replicable algorithms with very small parameters, one may have to be careful to ensure that the randomness management can also be done within these constraints. 

%cite Ken Regan's 1990s paper?

%{\url{https://en.wikipedia.org/wiki/Pairing_function\#Other_pairing_functions}}
%\rexnote{\url{http://szudzik.com/ElegantPairing.pdf}}

\subsection{Replicability across Two Close Distributions}

Next, we prove a simple Lemma bounding the effect of distributional shift on replicability.

\begin{lemma}[Replicability under Distributional Shift]
\label{lem:repl-loss-by-distributional-shift}
Let $D_1$ and $D_2$ be two distributions over $\X$ with total variational distance $d_{TV}(D_1, D_2) = \delta$. Let $\rho \ge 0$, and let $\Acal$ be a $\rho$-replicable algorithm that uses a sample of size exactly $m$. 
Then
$$
\Pr_{S_1 \sim D_1^m, S_2 \sim D_2^m, r} [\Acal(S_1; r) = \Acal(S_2; r)] 
\ge (1-\delta)^{2m} \rho.$$
\end{lemma}

\begin{proof}
Since $d_{TV}(D_1, D_2) = \delta$, there exist distributions $D, D',$ and $D''$ such that 
$D_1 = (1-\delta)D + \delta D'$ and $D_2 = (1-\delta) D + \delta D''$. 
Thus,
\begin{equation*}
    \begin{split}
        \Pr_{S_1 \sim D_1^m, S_2 \sim D_2^m, r} [\Acal(S_1; r) = \Acal(S_2; r)] 
        &=
        (1-\delta)^{2m} \Pr_{S_1 \sim D^m, S_2 \sim D^m, r} [\Acal(S_1; r) = \Acal(S_2; r)] 
        \\&\qquad + \dots 
        \\&\qquad + \delta^{2m} \Pr_{S_1 \sim D'^m, S_2 \sim D''^m
        , r} [\Acal(S_1; r) = \Acal(S_2; r)] 
        \\&\ge   (1-\delta)^{2m} \Pr_{S_1 \sim D^m, S_2 \sim D^m, r} [\Acal(S_1; r) = \Acal(S_2; r)] 
        \\&= (1-\delta)^{2m} \rho.
    \end{split}
\end{equation*}

\end{proof}

However, Lemma~\ref{lem:repl-loss-by-distributional-shift} may not be tight for specific class of algorithms, functions, or distributions.

\newpage
\section{Glossary}
\begin{itemize}
\item [$\rho$] replicability parameter
\item [$(\eta, \nu)$] 2-parameter definition of replicability
\item [$(\eps, \delta)$]  differential privacy parameters
\item [$H$] hypothesis classes
\item [$f,g,h$] target functions and functions from hypothesis classes
\item [$\mathcal{X}$] input spaces
\item [$\mathcal{Y}$] output spaces
\item [$(\alpha, \beta)$] accuracy parameters (failure probability is $\beta$).
\item [$m$] sample complexity, size of datasets
\item [$D$] distributions (subscripts for multiple distributions)
\item [$P$] family of distributions
\item [$\err$] learning error (subscripts for sample and distribution)
\item [$S$] input datasets 
\item [$x,y$] single data point and single output respectively
%\item $\beta, \eps, \delta$- for Perfect Generalization (in context of learning, change $\beta$).
\item [$r$] internal coins
%\item ; - separate sample and internal randomness in algorithms
%\item for probabilities $\Pr$ (use capital P in macros), and expectations $\E$
\item [$d$] (notions of) dimension
\item [$\Delta$] symmetric difference%, or sensitivity, depending on context.
\item [$\dtv$] total variation distance
\item [$\approx_{\eps, \delta}$] approximate max-KL indistinguishable
\item [$p$] Bernoulli biases %(with subscripts if necessary).
\item [$b$] bit used for message
\item [$c$] ciphertext
\item [$i,j$] subscripts for (dual) indexing
\item [$O$] set of outputs
%\item in general use (specific) subscripts to disambiguate.
\item [$\Acal, \Bcal$]  algorithms
%\item Use ``sample" if implied distribution, and ``dataset" if no distributional assumptions. An ``example" is a single thing in a ``sample".
\item [$I$]  interval
\item [$v$] threshold values
\item [$\mathbin\Vert$] concatenation of strings
\end{itemize}

\end{document}